	\documentclass{article}
	\usepackage{microtype}
	\usepackage{graphicx}
	\usepackage{subfigure}
	\usepackage{booktabs} 
	
	\usepackage[colorlinks=true,citecolor=brown,
	linkcolor=blue,anchorcolor=red,unicode=true,
	urlcolor=blue]{hyperref}

	\usepackage[accepted]{icml2021}
	\usepackage[utf8]{inputenc}

\usepackage{amsmath,amsfonts,bm}

\def\eqref#1{Eq.~(\ref{#1})}

\def\1{\bm{1}}

\DeclareMathAlphabet{\mathsfit}{\encodingdefault}{\sfdefault}{m}{sl}
\SetMathAlphabet{\mathsfit}{bold}{\encodingdefault}{\sfdefault}{bx}{n}

\DeclareMathOperator*{\argmax}{arg\,max}

	\usepackage{enumitem}
	\usepackage{pifont}
	\usepackage{hhline}
	\usepackage{url}

	\usepackage{booktabs}      
	\usepackage{nicefrac}   
	\usepackage{microtype}
	\usepackage{graphicx}
	\usepackage{csquotes}
	\usepackage{subfigure}
	\usepackage{wrapfig}
	\usepackage{enumitem}
	\usepackage{multirow}
	\usepackage{diagbox}
	\usepackage{verbatim}
	\usepackage{amsmath}
	\usepackage{bbm}
	\usepackage{comment}
	\usepackage{amsthm}
	\usepackage{amssymb}
	\usepackage{amsmath}
	\usepackage{nicefrac}
	\usepackage{comment}
	\usepackage{epigraph}
	\usepackage{natbib}
	\usepackage{changes}
	\usepackage{ulem}
	\usepackage{makecell}
	\usepackage{rotating}
	\setcitestyle{authoryear,open={(},close={)}}
	\usepackage{cuted}

	\definecolor{mygreen}{rgb}{0,0.5,0.2}

	\newtheorem{theorem}{Theorem}
	
	\newtheorem{lemma}[theorem]{Lemma}
	\newtheorem{definition}[theorem]{Definition}
	
	\newtheorem{remark}{Remark}
	\newtheorem{proposition}[theorem]{Proposition}
	
	\newtheorem{example}{Example}[section]
	\numberwithin{theorem}{section}
	
	\icmltitlerunning{Latent Causal Invariant Model}
	
\begin{document}
		
		\twocolumn[
		\icmltitle{Latent Causal Invariant Model}
		
		
		
		\icmlsetsymbol{equal}{*}
		
		\begin{icmlauthorlist}
			\icmlauthor{Xinwei Sun}{1}
			\icmlauthor{Botong Wu}{2}
			\icmlauthor{Xiangyu Zheng}{2}
			\icmlauthor{Chang Liu}{1}
			\icmlauthor{Wei Chen}{1}
			\icmlauthor{Tao Qin}{1}
			\icmlauthor{Tie-yan Liu}{1}
		\end{icmlauthorlist}
		
		\icmlaffiliation{1}{Microsoft Research Asia, Beijing, 100080}
		\icmlaffiliation{2}{Peking University, Beijing, 100871}
		\icmlcorrespondingauthor{Xinwei Sun}{xinsun@microsoft.com}


		
		\icmlkeywords{Machine Learning, ICML}
		
		\vskip 0.3in
		]
		
		\printAffiliationsAndNotice{}
		
		
		\begin{abstract}
			
			Current supervised learning can learn spurious correlation during the data-fitting process, imposing issue regarding out-of-distribution (OOD) generalization. To address this problem, we propose \textbf{La}tent \textbf{C}ausal \textbf{I}nvariance \textbf{M}odels (LaCIM), as a set of causal models by taking causal structure into consideration, in order to identify causal relations for prediction. Specifically, we introduce latent variables that are separated into (a) output-causative factors and (b) others that are spuriously correlated to the output via confounders. Such a spurious correlation can lead to the domain shift. We show that the observational distribution conditioning on latent factors are invariant to the above domain shift, and is thus called \emph{causal invariance}. Further, we give the identifiable claim of such invariance, particularly the disentanglement of output-causative factors from others, as a theoretical guarantee for precise inference and avoiding spurious correlation. We then propose a Variational-Bayesian-based method to learn such an invariance. The utility of our approach is verified by improved generalization ability on various OOD scenarios. 

		\end{abstract}
		\section{Introduction}
		Current data-driven deep learning models, revolutionary in various tasks though, heavily rely on \emph{i.i.d} data to exploit all types of correlations to fit data well. Among such correlations, there can be spurious ones corresponding to biases (\textit{e.g.}, selection or confounding bias due to coincidence of the presence of the third factor) inherited from the data provided. Such data-dependent spurious correlations can erode the prediction power for out-of-distribution (OOD) samples (\emph{i.e.}, the ones that are differently distributed with training data), which is crucial especially in safety-critical tasks.

		Recently, there is a Renaissance of causality in machine learning, expected to pursue causal prediction \citep{scholkopf2019causality}. The so-called ``causality" is pioneered by Judea Pearl \citep{pearl2009causality}, as a mathematical formulation of this metaphysical concept grasped in the human mind. The incorporation of a \textit{priori} about cause and effect endows the model with the ability to identify the causal structure \cite{pearl2009causality} which entails not only the data but also the underlying process of how they are generated. For causal prediction, the old-school methods \citep{peters2016causal, buhlmann2018invariance} causally related the output label $Y$ to the \textit{observed input $X$}, which however is NOT conceptually reasonable in scenarios with sensory-level observed data ({\textit{e.g. modeling pixels as causal factors of $Y$ does not make much sense}}). 
		
		
		For such applications, we rather adopt the manner of human visual perception \cite{bengio2013representation, biederman1987recognition} to relate the causal factors of human label $Y$ to unobserved abstractions denoted by $S$, \textit{i.e.}, $Y \gets f_y(S,\varepsilon_y)$ via mechanism $f_y$. We further assume existence of additional latent components denoted as $Z$, that together with $S$ generates the input $X$ via mechanism $f_x$ as $X \gets f_x(S, Z,\varepsilon_x)$. Such an assumption is similarly adopted in the literature of nonlinear Independent Components Analysis (ICA) \citep{hyvarinen2016unsupervised, hyvarinen2019nonlinear, khemakhem2020variational, teshima2020few} and latent generative models \citep{suter2019robustly}. To model the effect of domain shifts, we allow the $Z$ to be spuriously correlated with $S$, hence the output $Y$, as marked by the bidirected arrow in Fig.~\ref{fig:lacim} (a). Taking image classification as an example, the $S$ and $Z$ respectively refer to object-related abstractions (\textit{e.g.}, contour, texture) and contextual information (\textit{e.g.}, background, view). During data-fitting process, the model can learn contextual information into prediction, as it can be correlated with the label in data provided. 
		
		
		We encapsulate these assumptions into a set of causal models as illustrated in Fig.~\ref{fig:lacim} (a), in which we argue that the generating mechanisms $f_x,f_y$ are invariant across domains (as marked by the blue arrow); while the spurious correlation between $S$ and $Z$ is allowed to be varied (as marked by the red (bi-directed) arrow). Such a domain-dependent spurious correlation, as governed by an auxiliary domain variable $D$ in Fig.~\ref{fig:lacim} (c) when takes a closer inspection, can lead to domain shifts. We call such a set of causal models augmented with the domain variable $D$ as \textbf{La}tent \textbf{C}ausal \textbf{I}nvariance \textbf{M}odels (LaCIM). Under the assumptions embedded in the causal structure of LaCIM, we can derive that the $P(Y|do(s))$ and $P(X|do(s),do(z))$ are stable to the shift across domains and we thus call them \emph{\textbf{C}ausal \textbf{I}nvariance} (CI). Further, we can show that if the multiple environments are diverse enough, such CI are identifiable, which can benefit the OOD prediction. Besides, our identifiability results can implicate that the learned $Y$-causative factor (\textit{a.k.a}, $S$) can be disentangled from others (\textit{a.k.a}, $Z$), \emph{i.e.}, does \emph{not} mixture the information of $Z$.

		Guaranteed by the identifiability claims, we propose to learn the CI for prediction. Given the causal structure of LaCIM, we resort to latent generative model by reformulating the Variational Auto-encoder (VAE) \citep{kingma2014auto} to our supervised scenario. For OOD prediction, we propose to optimize over latent space under the identified CI (specifically $P(X|do(s),do(z)$). To verify the correctness of our identifiability claim, we conduct a simulation experiment. We further demonstrate the utility of our LaCIM via improved prediction power on various OOD scenarios (including tasks with confounding and selection bias, healthcare) and high explainable learned semantic features.

		We summarize our contribution as follows: \textbf{(i) Methodologically}, we propose in section~\ref{sec:lacim} a set of causal models in which the causal assumptions are incorporated in order to reason the \emph{causal invariance} for OOD generalization; \textbf{(ii) Theoretically}, we prove the identifiability (in theorem~\ref{thm:iden}) of CI $P(X|do(s),do(z)),P(Y|do(s))$ and also the $Y$-causative factor up to permutation and point-wise transformation; \textbf{(iii) Algorithmically}, guided by the identifiability, we in section~\ref{sec:vae} reformulate Variational Bayesian method to estimate CI during training and optimize over latent space during the test; \textbf{(iv) Experimentally}, LaCIM outperforms others in terms of prediction power on OOD tasks and interpretability in section~\ref{sec:expm-ood}.

		
		\vspace{-0.2cm}
		\section{Related Work}
		
		The invariance/causal learning proposes to learn the assumed invariance for transferring to OOD samples. For the invariance learning methods in \citet{krueger2020out,subbaswamy2020spec} and \cite{scholkopf2019causality}, they are still data-driven without incorporating causal assumptions \cite{pearl2009causality} beyond data, which results in that the learned ``invariance" is still stable correlation rather than causation and hence impedes its generalization to a broader set of domains. For causal learning, \citet{peters2016causal, buhlmann2018invariance, kuang2018stable, heinze2017conditional} assume causal factors as observed input, which is inappropriate for sensory-level observational data. In contrast, our LaCIM takes into account the causal structure; specifically, we introduce \textbf{i)} latent factors and separate them into $Y$-causative factor and others; \textbf{ii)} an augmented domain variable, which plays as a selection variable that generates the varied $S$-$Z$ correlation across domains. The incorporation of such a causal structure makes it possible to isolate the causal invariance and also only the $Y$-causative factor for OOD prediction. In independent and concurrent works, \citet{teshima2020few} and \citet{ilse2020designing} also explore latent variables in causal relation. As comparisons, \citet{teshima2020few} did not differentiate $S$ from $Z$. The \citet{ilse2020designing} is limited to only considering the spurious correlation between domain and label; while our LaCIM can allow the spurious correlation existed in a single domain.


		
		Other works which are conceptually related to us, as a non-exhaustive review, include (i) transfer learning which also leverages invariance in the context of domain adaptation \citep{scholkopf2011robust, zhang2013domain, gong2016domain} or domain generalization \citep{li2018domain, shankar2018generalizing}; and (ii) causal inference \citep{pearl2009causality,peters2017elements} which proposes a structural causal model to incorporate intervention via ``do-calculus" for cause-effect reasoning and counterfactual learning; (iii) latent generative model which also assumes generation from latent space to observed data \citep{kingma2014auto,suter2019robustly} \textit{but} aims at learning generator in the unsupervised scenario.

		\vspace{-0.2cm}
		\section{Preliminaries}
		\vspace{-0.1cm}
		
		\textbf{Problem Setting.} Let $X,Y$ respectively denote the input and output variables. The training data $\{\mathcal{D}^e\}_{e \in \mathcal{E}_{\mathrm{train}}}$ are collected from multiple environments $\mathcal{E}_{\mathrm{train}}$, where each domain $e$ is associated with a distribution $\mbox{P}^e(X,Y)$ over $\mathcal{X} \times \mathcal{Y}$ and $\mathcal{D}^e = \{x^e_i,y^e_i\}_{i \in [n_e]} \overset{i.i.d}{\sim} \mbox{P}^e$ with $[k]:=\{1,...,k\}$ for any $k \in \mathbb{Z}^{+}$. Our goal is to learn $f: \mathcal{X} \to \mathcal{Y}$ that learns $Y$-causative (or output-causative) factor for prediction and performs well on the set of all environments $\mathcal{E} \supset \mathcal{E}_{\mathrm{train}}$, which is aligned with existing OOD generalization works \citep{arjovsky2019invariant, krueger2020out}. We use respectively upper, lower case letter and Cursive letter to denote the random variable, the instance and the space, \textit{e.g.}, $a$ is an instance in the space $\mathcal{A}$ of random variable $A$. The $[f]_{\mathcal{A}}$ denotes the $f$ restricted on dimensions of $\mathcal{A}$. The Sobolev space $W^{k,p}(\mathcal{A})$ contains all $f$ such that $\int_{\mathcal{A}} \big|\partial_A f^{\alpha}\bigr|_{A=a} \big|^p d\mu(a) < \infty, \forall \alpha \leq k$.

		\textbf{Structural Causal Model.} The structural causal model (SCM) is defined as a triplet $M:=\langle G,\mathcal{F},P(\varepsilon)\rangle$, in which \textbf{i)} the causal structure $G:=(V,E)$ described by a directed acyclic graph (DAG); \textbf{ii)} the structural equations $\mathcal{F}:=\{f_k\}_{V_k \in V}$ are autonomous, \emph{i.e.}, intervening on $V_k$ does not affect others, based on which we can calculate causal effect; \textbf{iii)} the $P(\varepsilon)$ are probability measure for exogenous variables $\{\varepsilon_k\}_k$. By assuming independence among $\{\varepsilon_k\}_k$, it can be obtained according to Causal Markov Condition that each $P$ that is compatible with $G$ has $\mbox{P}(\{V_k=v_k\}_{V_k \in V})=\Pi_{k} \mbox{P}(V_k=v_k|Pa(k)=pa(k))$. A back-door path from $V_a$ to $V_b$ is defined as a path that ends with an arrow pointing to $V_a$ \citep{pearl2009causality}.

		\section{Methodology}
		
		We build our causal models which incorporate the causal assumptions in section~\ref{sec:lacim}, with which we can define the causal invariance that is robust to domain shift. In section~\ref{sec:iden}, we will present the identifiability of such causal invariance; and the $Y$-causative factor up to permutation transformation and point-wise addition, which guarantees the disentanglement of learned $Y$-causative features from others. Finally, we will introduce our learning method in section~\ref{sec:vae} to identify the causal invariance for prediction.


		
		
		\begin{figure*}
			\centering
			\begin{tabular}{cccccc}
				\includegraphics[width=0.17\textwidth]{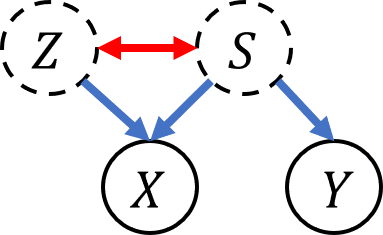}
				\hspace{-0.5cm}
				& \includegraphics[width=0.14\textwidth]{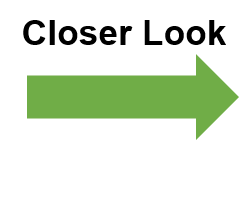}
				& \includegraphics[width=0.17\textwidth]{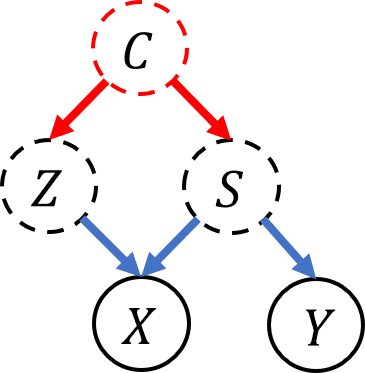}
				\hspace{-0.5cm}
				& \includegraphics[width=0.14\textwidth]{Figure/p66.png}
				& \includegraphics[width=0.17\textwidth]{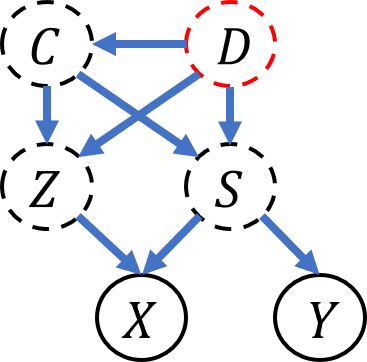} \\
				(a) &  & (b) & & (c) 
			\end{tabular}
			\caption{The directed acyclic graph of our Latent Causal Invariant Models. The observed and observed variables are respectively marked by solid and dot circle. The directed arrow represents direct causal relation; while the bidirected arrows represent spurious correlation. The arrows marked by blue and red represent the invariant and variant mechanisms, respectively. The $X,Y$ denote input and output variable; the $S,Z$ denote the $Y$-causative factor and $Y$-non-causative factor. The $C$ denotes the confounder of $S,Z$. The $D$ denotes the domain index with fixed value for each domain. The (a) $\to$ (b) $\to$ (c) is step-by-step closer inspection of spurious correlation. From (a) to (b): the spurious correlation between $S$ and $Z$ can be explained by another unobserved confounder $C$, which is differently distributed across environments and generates the $S,Z$ via domain-dependent mechanisms, as further explained by (c).}
			\label{fig:lacim}
		\end{figure*}
		
		\subsection{\textbf{La}tent \textbf{C}ausal \textbf{I}nvariance \textbf{M}odels}
		\label{sec:lacim}
		
		In this section, we introduce our model with \emph{causal invariance} from latent variables to observed variables, namely \textbf{La}tent \textbf{C}ausal \textbf{I}nvariance \textbf{M}odel (LaCIM). The corresponding DAG of LaCIM is illustrated in Fig.~\ref{fig:lacim} (c), as a step-by-step inspections of the skeleton shown in Fig.~\ref{fig:lacim} (a). 
		
		Specifically, we in Fig.~\ref{fig:lacim} (a) introduce latent factors $V:=\{S,Z\}$ to model the abstractions/concepts that generate the observed variables $(X,Y)$, which has been similarly assumed in existing latent generative models \cite{kingma2014auto} for sensory-level data. Further, we explicitly separate the $V$ into $S$ and $Z$ that respectively denote the $Y$-causative and $Y$-non-causative factors, with only $S$ having the direct causal effect on the label $Y$. In other words, the $Y$ is generated by $S$, which \emph{for example} refers to the shape, the contour of the object of interest in object classification; while the image $X$ is additionally affected by contextual factors such as \emph{light,view}. Since follows from physical law in generating the observed variables, the process $(S,Z) \to X$ and $S \to Y$ are hence assumed to be invariant across all environments/domains. We derive that the associated interventional distributions (and $P(X|do(s),do(z)),P(Y|do(s))$) as \emph{\textbf{C}ausal \textbf{I}nvariance}, as formally defined in Prop.~\ref{prop:CI}. 
		
		From another perspective, note that we consider the scenario when the $X$ and $Y$ are generated concurrently, \emph{i.e.}, there is neither directed paths from $X$ to $Y$ \cite{arjovsky2019invariant}, nor directed paths from $Y$ to $X$ \cite{ilse2020designing}, which can commonly exist in real scenarios however has been ignored in the literature. For example, the clinicians are recording the disease status while implementing the ultrasound test at the same time, during medical diagnosis.

		In addition, we assume that there exists a \emph{spurious correlation} between the $S$ and $Z$, as marked by bidirected red arrows in Fig.~\ref{fig:lacim} (a). Such a spurious correlation corresponds to the bias inherited from data, \emph{e.g.} the contextual information in object classification. Therefore, unlike invariant causation, this correlation is data-dependent and the magnitude of it can vary across domains. In statistics, such a spurious implicates the presence of a third unobserved (we use dot circle to represent unobserveness) confounder, which is denoted as $C$ in Fig.~\ref{fig:lacim} (b). The unblocked path from $Z$ to $Y$ can make the model learn unexpected feature during data-fitting process. Taking a closer inspection in Fig.~\ref{fig:lacim} (b), the varying degree of correlation can be either due to the changing mechanism from $C \to (S,Z)$ (as marked by the red arrows) or mutable distribution of the confounder $C$ itself (as marked by the red circle) across domains. Further, we ascribe both changing causes to the domain variable $D$, which takes a fixed value for each domain, as shown in the red circle of $D$ in Fig.~\ref{fig:lacim} (c). In other words, the variation of $D$ across domains governs the change of $P(C)$ and the auxiliary dependency $S,Z$ on $D$ can further explain the varying mechanisms of $P(S,Z|C)$ across domains. Note that this domain variable is not required to be observed in our scenario, in which we can only access the domain index $I^e$ (\emph{e.g.} one-hot encoded vector with length $m:=|\mathcal{E}_{\mathrm{train}}|$). 
		
		As an illustration, we first give a realization of the LaCIM (Fig.~\ref{fig:lacim} (c)), followed by the formal definition~\ref{def:lacim}. 
		\begin{example}[Sampling Bias]
			Consider the cat/dog classification task in which the animal in each image is either associated with the snow or grass. The $D$ denotes the sampler, which generates the $C$ that denotes the (time,weather) to go outside and collect sample. The $S$ refers to the features of cat/dog while $Z$ denotes the scene concepts of grass/snow. Since each sampler may have a fixed pattern (\emph{e.g.} gets used to going out in the sunny morning (or in the snowy evening)), the data he/she collects may have sampling bias, (\emph{e.g.} with dogs (cats) more associated with grass (snow) in the sunny morning (or snowy evening) ). In this regard, the scene concepts $Z$ can be correlated with the label $Y$.
		\end{example}
		
		\begin{definition}[LaCIM]
			\label{def:lacim}
			The LaCIM denotes a set of SCMs augmented with the domain variable $D$, \emph{i.e.}, $\{\langle M^e,d^e \rangle\}_{e \in \mathcal{E}}$, in which $M^e := \langle G,\mathcal{F}^e,P(\varepsilon) \rangle$. The $G$ denotes the DAG in Fig.~\ref{fig:lacim} (b). For each environment/domain $e$, the $\mathcal{F}^e:=\{f_x,f_y,f^e_s,f^e_z,f^e_c\}$ corresponding generating mechanism of $X,Y,S,Z,C$, with $f^e_c(\varepsilon_c):=g_c(\varepsilon_c,d^e)$, $f^e_s(c,\varepsilon_s):=g_s(c,\varepsilon_s,d^e)$ and $f^e_z(c,\varepsilon_z):=g_z(c,\varepsilon_z,d^e)$ from some $g_c,g_s,g_z$. 
		\end{definition}
		
		\begin{remark}
			The Def.~\ref{def:lacim} is compatible with the \emph{selection diagram} in \cite{subbaswamy2020spec}, in which an auxiliary variable is augmented with a set of acyclic graphs and generates mutable distributions across environments. In Def.~\ref{def:lacim}, such an auxiliary variable refers to $D$, which makes the $P(S,Z|C)$ and $P(C)$ vary across domains. Besides, note that the $f^e_c(\varepsilon_c):=g_c(\varepsilon_c,d^e)$ (or $f^e_s(c,\varepsilon_s):=g_s(c,\varepsilon_s,d^e)$) is allowed to be intervention, \emph{i.e.} set the $S$ or ($Z$) as a fixed value. 
		\end{remark}
		
		The Def.~\ref{def:lacim} specifies the generating mechanisms across environments and how can they differ. Equipped with such a specification, we can define the \emph{invariant} prediction mechanism which is stable to domain shifts: 
		
		\begin{proposition}[Causal Invariance]
			\label{prop:CI}
			For LaCIM in Def.~\ref{def:lacim}, the $P(X|do(s^{\star}),do(z^{\star}))$ and $P(Y|do(s^{\star}))$ are invariant to shifts across $\mathcal{E}$, and are denoted as \textbf{C}ausal \textbf{I}nvariance (CI). According to Fig.~\ref{fig:lacim}, we have $p(x|do(s^{\star}),do(z^{\star})) = p(x|s^\star,z^\star)$ and $p(y|do(s^{\star}))=p(y|s^\star)$.
		\end{proposition}

		Note that here ``$do(s^{\star})$" (or ``$do(z^{\star})$") denotes the intervention operation during data generating process, rather than the process during inference in which the $S,Z$ are unobserved. During prediction, for sample $(x,y)$ generated by $x \gets f_x(s^\star,z^\star,\varepsilon_x)$, $y \gets f_y(s^\star,\varepsilon_y)$, the goal is to inference $s^{\star}$ from $p_{f_x}(x|s,z)$ first and also $p_{f_y}(y|s^{\star})$ for prediction. Two natural \emph{identifiability} questions can be asked: 
		\begin{enumerate}[topsep=0pt,itemsep=-1ex,partopsep=1ex,parsep=1ex]
			\item \emph{Can the \emph{causal invariance} be identified and efficiently learned, in order for OOD generalization?}
			\item \emph{Can the learned $Y$-causative factor (\emph{a.k.a} $S$) be disentangled from the others?}
		\end{enumerate}
		We will give our answer to these questions in the subsequent section, followed by our learning methods to identify the $Y$-causative factor and the CI for prediction.

		
		\subsection{Identifiability Analysis}
		\label{sec:iden}
		
		In this section, we present the identifiability claims regarding the two questions imposed in section~\ref{sec:lacim}. Statistically, the identifiability implies the precisely inference of underlying parameters giving rise to the observational distribution: $p_\theta(x,y)\!=\!p_{\tilde{\theta}}(x,y)\!\implies\!\theta \!=\! \tilde{\theta}$. In our scenario, the parameters specifically refers to the \textbf{(i)} CI ($f_x,f_y$ and corresponding $p(x|s,z),p(y|s)$); and \textbf{(ii)} the $Y$-causative features (\emph{a.k.a} $S$) estimated from $f_x$ up to transformation that does \emph{not} mix with the information of $Y$-non-causative features (\emph{a.k.a} $Z$), which correspondingly echo the questions in section~\ref{sec:lacim}.

		
		Our main results are presented in theorem~\ref{thm:iden}. Speaking in a high-level way, our results require that the degree of diversity regarding $S$-$Z$ correlation across environments is large enough, which has been similarly assumed in \cite{peters2016causal,arjovsky2019invariant} and implies complementary information from multiple datasets for the invariance to be picked up. Besides, for the disentanglement of $S$ and $Z$, note that the $S$ and $Z$ play ``asymmetric roles" in terms of generating process, as reflected in additional arrow from $S$ to $Y$. This ``information intersection" property of $S$ for $X,Y$, \textit{i.e.}, $f_y^{-1}(\bar{y}) = [f_x^{-1}]_S(\bar{x})$ for any $(\bar{x},\bar{y}) \in f_x(\mathcal{S}, \mathcal{Z}) \times f_y(\mathcal{S})$ {if $y = f_y(s) + \varepsilon_y$}, is exploited to disentangle $S$ from $Z$. Such a disentanglement analysis, is crucial to causal prediction but is ignored in existing literature about identifiability, such as those identifying the discrete latent confounders \citep{janzing2012detecting,sgouritsa2013identifying}, or those relying on ANM assumption \citep{janzing2012identifying}, or linear ICA \citep{eriksson2003identifiability,khemakhem2020variational,khemakhem2020ice,teshima2020few} (Please refer to supplement~\ref{sec:related-iden} for more exhaustive reviews). Besides, our analysis extends the scope of \cite{khemakhem2020variational} to categorical $Y$ and general forms of $\mbox{P}(S,Z|C)$ that belongs to Sobolev space, in theorem~\ref{thm:p-iden}.

		

		We assume the \emph{Additive Noise Model} (ANM) for $f_x(s, z, \varepsilon_x)$ $\!=\! \hat{f}_x(s,z) + \varepsilon_x$ (we replace $\hat{f}_x$ with $f_x$ for simplicity), which has been widely adopted to identify the causal factors \citep{janzing2009identifying, peters2014causal,khemakhem2020variational}. We first narrow our interest to a subset of types of parameterization for LaCIM denoted as $\mathcal{P}_{\mathrm{exp}}$ in which any parameterization in $\mathcal{P}_{\mathrm{exp}}$ satisfies that \textbf{(i)} the $S,Z$ belong to the exponential family; and \textbf{(ii)} the $Y$ is generated from the ANM. We show later that $\mathcal{P}_{\mathrm{exp}}$ can approximate any $\mbox{P}(S,Z|c) \in W^{r,2}(\mathcal{S} \times \mathcal{Z})$ for some $r \geq 2$:
		\vspace{-0.2cm}
		\begin{equation}
		\begin{aligned}
		\label{eq:p-exp}
		\mathcal{P}_{\mathrm{exp}} & = \Big\{ \text{LaCIM} | \ y = f_y(s) + \varepsilon_y, \\ 
		& p^e(s,z|c) := p_{\mathbf{T}^z, \bm{\Gamma}^z_{c,d^e}}(z|c) p_{\mathbf{T}^s, \bm{\Gamma}^s_{c,d^e}}(s|c), \forall e\Big\}, \text{with} \nonumber 
		\end{aligned}
		\end{equation}
		\begin{equation*}
		p_{\mathbf{T}^{t}, \bm{\Gamma}^t_{c,d^e}}(t) = {\small\prod_{i=1}^{q_t}} \exp\Big( {\small \sum_{j=1}^{k_t}} T^t_{i,j}(t_i) \Gamma^t_{c,d^e,i,j} + B_i(t_i) -  A^t_{c,d^e,i} \Big),
		\end{equation*}
		for $t = s,z$ and $e\in\mathcal{E}$.
		The $\{T^t_{i,j}(t_i)\}$, $\{\Gamma^t_{c,d^e,i,j}\}$ denote the sufficient statistics and natural parameters, $\{B_i\}$ and $\{A_{c,d^e,i}^t\}$ denote the base measures and normalizing constants to ensure the integral of distribution equals to 1. Let {\small $\mathbf{T}^t(t)\!:=\![\mathbf{T}^{t}_{1}(t_{1}),...,\mathbf{T}^{t}_{q_t}(t_{q_{t}}) ]$ $\!\in\! \mathbb{R}^{k_t \times q_t}$ $\big(\mathbf{T}^t_{i}(t_{i}) \!:=\! [T^t_{i,1}(t_i),...,T^t_{i,k_t}(t_i)], \forall i \in [q_t]\big)$}, {\small$\bm{\Gamma}^t_{c,d^e} \!:=\! \left[\bm{\Gamma}^{t}_{c,d^e,1},...,\bm{\Gamma}^{t}_{c,d^e,q_t} \right]$ $\!\in\! \mathbb{R}^{k_t \times q_t}$ $\big(\bm{\Gamma}^t_{c,d^e,i} \!:=\! [\Gamma^t_{c,d^e,i,1},...,\Gamma^t_{c,d^e,i,k_t}], \forall i \in [q_t]\big)$}. We further assume that the $P^e(C)$ serves to discrete distributions on the set $\{c_1,...,c_R\}$, with which the $p^e(s,z) := \int p(s|c)p(z|c)dP^e(c)$ can be regarded as the mixture of distributions that belong to exponential family. 
		
		For our supervised scenario  with additional goals of disentangle $S$ and $Z$ and identifying the $p(y|s)$, we extend the $\sim_p$-identfiability \cite{khemakhem2020variational} of $\theta:=\Big\{f_x,f_y,\mathbf{T}^s,\mathbf{T}^z,\{P^e(C)\}_e \Big\}$:
		\begin{definition}[$\sim_p$-identifiability]
			\label{p-identifiable}
			We define a binary relation $\theta \sim_p \tilde{\theta}$ on the parameter space of $\mathcal{X} \times \mathcal{Y}$: there exist two sets of permutation matrices and vectors, $(M_s, a_s)$ and $(M_z, a_z)$ for $s$ and $z$ respectively, such that for any $(x,y) \in \mathcal{X} \times \mathcal{Y}$, the following hold: 
			\vspace{-0.1cm}
			\begin{align}
			& \tilde{\mathbf{T}}^s([\tilde{f}_{x}^{-1}]_{\mathcal{S}}(x)) = M_s \mathbf{T}^s([f_{x}^{-1}]_{\mathcal{S}}(x)) + a_s; \label{eq:s} \\
			& \tilde{\mathbf{T}}^z([\tilde{f}_{x}^{-1}]_{\mathcal{Z}}(x)) = M_z \mathbf{T}^z([f_{x}^{-1}]_{\mathcal{Z}}(x))  + a_z; \label{eq:z} \\
			& p_{\tilde{f}_y}(y|[\tilde{f}_{x}^{-1}]_{\mathcal{S}}(x)) = p_{f_y}(y|[f_{x}^{-1}]_{\mathcal{S}}(x)). \label{eq:y_1} 
			\end{align}
			We then say that $\theta$ is $\sim_p$-identifiable, if for any $\tilde{\theta}$, $p_\theta^e(x,y) = p_{\tilde{\theta}}^e(x,y) ~\forall e \in \mathcal{E}_{\mathrm{train}}$, implies $\theta \sim_p \tilde{\theta}$.
		\end{definition}
		This definition is inspired by but beyond the scope of unsupervised scenario considered in nonlinear ICA \citep{hyvarinen2019nonlinear, khemakhem2020variational} to further disentangle $S$ from $Z$ and identify the CI for prediction. To connect these results with practical inference, recall that as hidden factors are unobserved, the first step during test stage is to estimate the $Y$-causative factor from $f^{-1}_x$ (or from $\argmax_{s,z} \log p(x|s,z)$). The \eqref{eq:s} claims that such an estimation is up to permutation and point-wise addition, which implicates that the identified $s$ (characterized by sufficient statistics) does \emph{not} mix with the information of $Z$. With such learned $s$ (\emph{i.e.}, $[\tilde{f}^{-1}]_{\mathcal{S}}(x)$), the next step is to implement $p_{f_y}(y|s)$ for prediction. The \eqref{eq:y_1} guarantees that the learned $p_{\tilde{f}_y}(y|[\tilde{f}^{-1}]_{\mathcal{S}}(x))$ can recover the ground-truth predicting mechanism, \emph{i.e.},  $p_{f_y}(y|[f_{x}^{-1}]_{\mathcal{S}}(x))$. The formal result is presented in theorem~\ref{thm:iden}.

		\begin{theorem}[$\sim_p$-identifiability]
			\label{thm:iden}
			For $\theta$ of $\mathcal{P}_{\exp}$ in Def.~\ref{def:lacim} with $m:=|\mathcal{E}_{\mathrm{train}}|$, under following assumptions:
			\begin{enumerate}[topsep=0pt,itemsep=-1ex,partopsep=0ex,parsep=1ex]
				\item The characteristic functions of $\varepsilon_x,\varepsilon_y$ are almost everywhere nonzero.
				\item $f_x$, $f'_x, f''_x$ are continuous and $f_x, f_y$ are bijective;
				\item The $\{T^t_{i,j}\}_{1\leq j \leq k_t}$ are linearly independent in $\mathcal{S}$ or $\mathcal{Z}$ for each $i \in [q_t]$ for any $t=s,z$; and $T^t_{i,j}$ are twice differentiable for any $t=s,z, i \in [q_t], j \in [k_t]$;
				\item The $\{\left(\mathbf{T}^s([f^{-1}]_{\mathcal{S}}(x)),\mathbf{T}^z([f^{-1}]_{\mathcal{Z}}(x))\right); \mathcal{B}(x)>0\}$ contains a non-empty open set in $\mathbb{R}^{q_s\times k_s + q_z\times k_z}$, with
				\begin{align}
				\mathcal{B}(x):=\prod_{i_s \in [q_s]} B_{i_s}([f^{-1}]_{i_s}(x))\prod_{i_z \in [q_z]} B_{i_z}([f^{-1}]_{i_z}(x)), \nonumber
				\end{align}
				\item The $L:=[P^{e_1}(C)^\mathsf{T},...,P^{e_m}(C)^\mathsf{T}]^\mathsf{T} \in \mathbb{R}^{m \times R}$ and $\big[[\bm{\Gamma}^{t=s,z}_{c_2,d^{e_1}} - \bm{\Gamma}^{t=s,z}_{c_1,d^{e_1}}]^\mathsf{T},...,[\bm{\Gamma}^{t=s,z}_{c_{R},d^{e_m}} - \bm{\Gamma}^{t=s,z}_{c_1,d^{e_1}}]^\mathsf{T}\big]^\mathsf{T} \in \mathbb{R}^{(R\times m) \times (q_t \times k_t)}$ have full column rank,
			\end{enumerate}
			we have that the $\theta$ is $\sim_p$ identifiable.
		\end{theorem}
		

		The bijectivity of $f_x$ and $f_y$ have been widely assumed in \citet{janzing2009identifying, peters2014causal, peters2017elements, khemakhem2020variational, teshima2020few} as a basic condition for identifiability. It naturally holds for $f_x$ to be bijective since the latent components $S,Z$, as high-level abstractions which can be viewed as embeddings in auto-encoder \citep{kramer1991nonlinear}, lies in lower-dimensional space compared with input $X$ which is supposed to have more variations, \textit{i.e.}, ($q_s + q_z < q_x$). For categorical $Y$, the $f_y$ which generates the classification result, \textit{i.e.}, $p(y=k|s) = [f_y]_k(s)/\left(\sum_k [f_y]_k(s) \right)$, will be shown later to be identifiable. 
		
		The containment of an open set in assumption (4) for $\{\left(\mathbf{T}^s([f^{-1}]_{\mathcal{S}}(x)),\mathbf{T}^z([f^{-1}]_{\mathcal{Z}}(x))\right); \mathcal{B}(x)>0\}$ implies that space expanded by sufficient statistics are dense in some open set, as a sufficient condition for the mixture distribution $P^e(C)$ and also $P^e(X,Y|c)$ to be identified.

		The diversity assumption (5) implies that \textbf{i)} $m \geq R$ and $m*R \geq \max(k_z*q_z,k_s*q_s)+1$; and that \textbf{ii)} different environments are diverse enough in terms of $S$-$Z$ correlation, as an almost a necessary for the invariant one to be identified (a different version is assumed in \cite{arjovsky2019invariant}). In supplement~\ref{appx:proofs-iden}, we will show that the \textbf{ii)} can hold unless the space of $\bm{\Gamma}$ belong to a zero-(Lebesgue) measure set. As noted in the formulation, a larger $m$ would be easier to satisfy the condition, which agrees with the intuition that more environments can provide more complementary information for the identification of the invariant mechanisms.

		\begin{table*}[!h]
			\centering
			\setlength{\tabcolsep}{5.5pt}
			\small
			\begin{tabular}{l|ll|ll|ll|ll|ll|ll}
				\toprule
				& \multicolumn{2}{c|}{Data \#1} & \multicolumn{2}{c|}{Data \#2} & \multicolumn{2}{c|}{Data \#3} & \multicolumn{2}{c|}{Data \#4} & \multicolumn{2}{c|}{Data \#5} & \multicolumn{2}{c}{Average} \\
				\cmidrule{2-13}
				&  $Z$ & $S$ & $Z$ & $S$ & $Z$ & $S$ & $Z$ & $S$ & $Z$ & $S$ & $Z$ & $S$  \\ 
				\midrule
				pool-LaCIM & 0.26 & 0.61 & 0.26 & 0.67 & 0.44 & 0.70 & 0.51 &  0.78 & 0.58 &  0.77 & 0.41 & 0.71   \\
				\midrule
				{LaCIM (\textbf{Ours}, $m=3$)} & 0.52 & 0.92  & 0.61 & 0.86 & 0.70 & 0.83 & 0.70 & 0.86 & 0.62 & 0.77 &  0.63 & \textbf{0.84}  \\
				LaCIM (\textbf{Ours}, $m=5$) & 0.61 & 0.85  & 0.77 & 0.85 & 0.72 & 0.80 & 0.72 & 0.79 & 0.69 & 0.85 &  \textbf{0.71}  & \textbf{0.84} \\
				\bottomrule
			\end{tabular}
			\label{table:sim-con-2} 
			\caption{MCC of identified latent variables. Average over 20 times for each data.} 
		\end{table*}
		\textbf{Extension to the general parameterization of LaCIM.} We extend the theorem~\ref{thm:iden} to general parameterization of LaCIM as long as its $\mbox{P}(S,Z|C=c) \in W^{r,2}(\mathcal{S} \times \mathcal{Z})$ (for some $r \geq 2$) and categorical $Y$, in the following theorem. This is accomplished by proving that any model in LaCIM can be approximated by a sequence of distributions with parameterization in $\mathcal{P}_{\exp}$, motivated by \citet{barron1991approximation} that the exponential family is dense in the set of distributions with bounded support, and in \citet{maddison2016concrete} that the continuous variable with multinomial logit model can be approximated by a series of distributions with \emph{i.i.d} Gumbel noise as the temperature converges to infinity.
		\begin{theorem}[Asymptotic $\sim_p$-identifiability]
			\label{thm:p-iden}
			Suppose the LaCIM satisfy that $p(x|s,z)$ and $p(y|s)$ are smooth w.r.t $s,z$ and $s$ respectively. For each $e$ and $c \in \mathcal{C}$, suppose $\mbox{P}^e(S,Z|c) \in W^{r,2}(\mathcal{S} \times \mathcal{Z})$ for some $r \geq 2$, we have that the LaCIM is asymptotically $\sim_p$-identifiable: $\forall \epsilon > 0$, $\exists \sim_p$-identifiable $\tilde{\mbox{P}}_{\theta} \in \mathcal{P}_{\exp}$, s.t. $d_{\mathrm{Pok}}(p^{e}(X,Y), \tilde{p}^e_\theta(X,Y)) < \epsilon, \forall e \in \mathcal{E}_{\mathrm{train}}$ \footnote{The $d_{\mathrm{Pok}}(p^1,p^e)$ denotes the Pokorov distance between $p^1$ and $p^2$, with $\lim_{n \to \infty} d_{\mathrm{Pok}}(\mu_n,\mu)$ $\to 0$ $\Longleftrightarrow \mu_n \overset{d}{\to} \mu$.}.
		\end{theorem}
		
		
		\subsection{Learning and Inference} 
		\label{sec:vae}
		
		Guided by the identifiability result, we in this section introduce our learning method to identify the CI, \emph{i.e.}, $p(x|s,z)$ and $p(y|s)$ for prediction. Roughly speaking, we first introduce our learning method as a generative model guided by Fig.~\ref{fig:lacim} (c) to learn the CI during training phase, followed by inference method for prediction. For inference method, we first leverage the learned $p(x|s,z)$ for estimating the value of $Y$-causative factor (\emph{a.k.a} $S$) that is ensured to be able to not mix the information from $Y$-causative factor (\emph{a.k.a} $Z$), followed by $p(y|s)$ for prediction. 
		
		\subsubsection{Learning Method}

		To learn the CI $p(x|s,z),p(y|s)$ for invariant prediction, we implement the generative model to fit $\{p^e(x,y)\}_{e \in \mathcal{E}_{\mathrm{train}}}$. Specifically, we reformulate the objective of Variational Auto-Encoder (VAE), as a generative model proposed in \citep{kingma2014auto}, in supervised scenario. As a latent generative model on the unsupervised Bayesian network $Z \to X$, the VAE was proposed in \cite{kingma2014auto} for unsupervised generation of high-dimensional data (such as image) that can makes the traditional methods like Markov chain Monte Carlo (MCMC) intractable. Specifically, to make it tractable, the VAE introduces the variational distribution $q_\psi$ parameterized by $\psi$ to approximate the intractable posterior by maximizing the following \textbf{E}vidence \textbf{L}ower \textbf{B}ound (ELBO): 
		\begin{align}
		-\mathcal{L}_{\theta,\psi} = \mathbb{E}_{p(x)}\Big[ \mathbb{E}_{q_\psi(v|x)}\log{\frac{p_\theta(x,v)}{q_\psi(v|x)}}\Big], \nonumber
		\end{align}
		as a tractable surrogate of $\mathbb{E}_{p(x)}\log{p_{\theta}(x)}$. In details, the ELBO is less than and equal to $\mathbb{E}_{p(x)}\big[\log{p_\theta(x)}\big]$ and the equality can only be achieved when $q_\psi(v|x) \!=\! p_\theta(v|x)$. Therefore, maximizing the ELBO over $p_\theta$ and $q_{\psi}$ will drive \textbf{(i)} $q_\psi(v|x)$ to learn $p_\theta(v|x)$; \textbf{(ii)} $p_{\theta}$ to learn the ground-truth model $p$ (including $p_{\theta}(x|v)$ to learn $p(x|v)$).


		{In our scenario, we introduce the variational distribution $q^{e}_{\psi}(s,z|x,y)$ for each environment $e$. The corresponding ELBO for $e$ is 
			\begin{align*}
			-\mathcal{L}^e_{\theta,\psi} \!\!=\!\! \mathbb{E}_{p^e(x,y)} \big[ \mathbb{E}_{q^e_\psi(s,z|x,y)} \log{ \frac{p^{e}_{\theta}(x,y,s,z)}{q^e_\psi(s,z|x,y)}} \big].
			\end{align*}
			Similarly, minimizing $\mathcal{L}^e_{\theta,\psi}$ can drive $p_{\theta}(x|s,z),p_{\theta}(y|s)$ to learn the CI (\textit{i.e.} $p(x|s,z),p(y|s)$), and also $q_{\psi}^e(s,z|x,y)$ to learn $p^e_{\theta}(s,z|x,y)$. Therefore, the $q_\psi$ can inherit the properties of $p_\theta$. As $p^e_{\theta}(s,z|x,y) \!=\! \frac{p^e_{\theta}(s,z|x)p_{\theta}(y|s)}{p^e_{\theta}(y|x)}$ for our DAG in Fig.~\ref{fig:lacim}, we can similarly reparameterize $q^e_{\psi}(s,z|x,y)$ as  $\frac{q^e_{\psi}(s,z|x)q_{\psi}(y|s)}{q^e_{\psi}(y|x)}$. Since the goal of $q_\psi$ is to mimic the behavior of $p_\theta$, we can replace $q_{\psi}(y|s)$ with $p_\theta(y|s)$. Besides, according to Causal Markov Condition, we have that $p^e_{\theta}(x,y,s,z) = p_{\theta}(x|s,z)p^e_{\theta}(s,z)p_{\theta}(y|s)$, with $p_{\theta}(x|s,z),p_{\theta}(y|s)$ shared across all environments. The $\mathcal{L}^e_{\theta,\psi}$ can be rewritten as: 
			\begin{equation}
			\begin{aligned}
			\mathcal{L}^e_{\theta,\psi} &= \mathbb{E}_{p^e(x,y)}\Big[ -\log{q^e_{\psi}(y|x)}-\\
			\quad\quad& \mathbb{E}_{q^e_\psi(s,z|x)} \frac{p_{\theta}(y|s)}{q^e_\psi(y|x)} \log{ \frac{p_\theta(x|s,z)p^e_\theta(s,z)}{q^e_\psi(s,z|x)} } \Big],
			\end{aligned}
			\label{eq:csvae}
			\end{equation}
			where $q^e_\psi(y|x) = \int_{\mathcal{S}} q^e_{\psi}(s|x)p_{\theta}(y|s)ds$. We parameterize the prior model $p^e_\theta(s,z)$ and inference model $q^e_{\psi}(s,z|x)$ as $p_\theta(s,z|I^e)$ and $q_{\psi}(s,z|x,I^e)$, in which $I^e$ (of environment $e$) denotes the domain \emph{index} that can be represented by the one-hot encoded vector with length $m := |\mathcal{E}_{\mathrm{train}}|$. 
			The overall loss function is: 
			\begin{align}
			\label{eq:lacim-unpool}
			\mathcal{L}_{\theta,\psi} = \sum_{e \in \mathcal{E}_{\mathrm{train}}} \mathcal{L}^e_{\theta,\psi}.
			\end{align}
			The training datasets $\{\mathcal{D}^e\}_{e \in \mathcal{E}_{\mathrm{train}}}$ are applied to optimize the prior models $p(s,z|I^e)$, inference models $\{q_\psi(s,z|x,I^e)\}_e$, generative model $p_\theta(x|s,z)$ and predictive model $p_\theta(y|s)$. Particularly, the parameters of generative models $p_{\theta}(x|s,z),p_{\theta}(y|s)$ are shared among all environments, corresponding to the invariance property of CI across all domains.

			
			
			\subsubsection{Inference \& Test.} 	
			\label{sec:inference}
			According to Prop.~\ref{prop:CI}, the $p(x|do(s),do(z)) = p(x|s,z)$ and $p(y|s)$ are invariant across all domains. Therefore, for any test sample $x$ generated from $(s^{\star},z^{\star})$, we can leverage the learned $p_\theta(x|s,z)$ to estimate $s,z$, then apply the $p_\theta(y|s)$ for prediction. Specifically, we first optimize $s,z$ via
			\begin{equation}
			\label{eq:optimize}
			\max_{s,z} \log{p_\theta(x|s,z)} + \lambda_s \Vert s \Vert_2^2 + \lambda_z \Vert z \Vert_2^2,
			\end{equation}
			with hyperparameters $\lambda_s>0$ and $\lambda_z>0$, by adopting the strategy in \cite{schott2018towards} that we first sample $k$ from $\mathcal{N}(0,I)$ and select the one that maximizes the \eqref{eq:optimize} as initial point, then we implement Adam to optimize for $T$ iterations. The implementation details and optimization effect are shown in supplement~\ref{appx:optimize-sz}. 
			
			After obtaining the estimated $s^\star,z^\star$, we then implement the learned $p_\theta(y|s^\star)$ for prediction: $\tilde{y}:=\argmax_y p_\theta(y|s^\star)$.


			\section{Experiments}
			
			We evaluate LaCIM on synthetic data to verify the identifiability in theorem~\ref{thm:iden} and OOD challenges: object classification with sample selection bias (Non-I.I.D. Image dataset with Contexts (NICO)); Hand-Writing Recognition with confounding bias (Colored MNIST (CMNIST)); prediction of Alzheimer's Disease (Alzheimer’s Disease Neuroimaging Initiative (ADNI \url{www.loni.ucla.edu/ADNI}). 
			

			\subsection{Simulation}
			\label{exp-sim}


			To verify the identifiability claim and effectiveness of our learning method, we implement LaCIM on synthetic data. The domain index $I^e \in \mathbb{R}^m$ denotes the one-hot encoded vector with $m=5$. To verify the effectiveness of training on multiple diverse domains ($m>1$), we also implement LaCIM by pooling data from all $m$ domains together, namely pool-LaCIM for comparison. We randomly generate $m=5$ datasets (with generating process introduced in supplement~\ref{appx-exp-sim}) and run 20 times for each. We compute the metric mean correlation coefficient (MCC) adopted in \cite{khemakhem2020variational}, which measures the goodness of identifiability under permutation by introducing cost optimization to assign each learned component to the source component. This measurement is aligned with the goal of $\sim_p$-identifiability, which allows us to distinguish $S$ from $Z$. Table~\ref{table:sim-con-2} shows the superiority of our LaCIM over pool-LaCIM in terms of $S, Z$ under permutation, by means of multiple diverse experiments. Besides, we conduct LaCIM on $m=3,5$ with the same total number of samples. It yields that more environments can perform better; and that even $m=3$ still performs much better than pool-LaCIM. To illustrate the learning effect, we visualize the learned $Z$ in Fig.~\ref{fig:sim-visualize}, with $S$ left in supplement~\ref{appx-exp-sim} due to space limit.  

			\begin{figure}[ht]
				\vspace{-0.2cm}
				\centering
				\begin{tabular}{ccccccc}
					\includegraphics[width=0.14\textwidth]{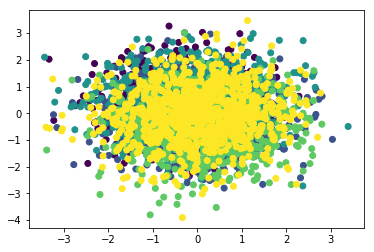} & \includegraphics[width=0.147\textwidth]{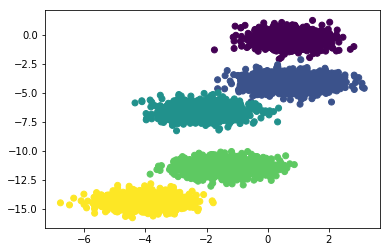} & \includegraphics[width=0.14\textwidth]{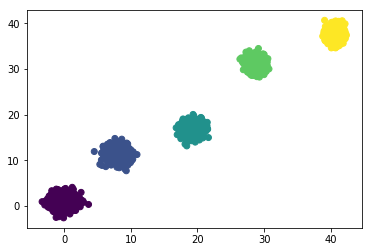} \\
					(a)  pool-LaCIM & (b) LaCIM & (c) $p_{\theta^\star}(z|D)$ 
				\end{tabular}
				\label{fig:sim-visualize}
				\caption{Estimated posterior by (a) pool-LaCIM; (b) LaCIM and (c) the ground-truth. As shown, the LaCIM can identify the $Z$ (up to permutation and point-wise transformation), which validates the \eqref{eq:z} in theorem~\ref{thm:iden}.}
			\end{figure}
			
			\begin{table*}[ht]
				\caption{Accuracy (\%) of OOD prediction. Average over ten runs.} 
				\label{table:acc} 
				\centering
				\setlength{\tabcolsep}{1.3pt}
				\renewcommand{\arraystretch}{1.2}
				\resizebox{\textwidth}{!}{
					\begin{tabular}{l|l|l|l|l|l|l|l|l|l}
						\toprule
						\multirow{3}{*}{\diagbox[trim=l,height=4.5\line]{Method}{Dataset}} & \multicolumn{4}{c|}{NICO} & \multicolumn{2}{c|}{CMNIST } & \multicolumn{3}{c}{ADNI $(m=2)$} \\
						\cmidrule{2-10}
						&  \multicolumn{2}{c|}{$m=8$} & \multicolumn{2}{c|}{$m=14$} & \multicolumn{2}{c|}{$m=2$} & $D$: Age & $D$: TAU & \multirow{2}{*}{\# Params} \\
						\cmidrule{2-9}
						&  ACC & \# Params & ACC & \# Params & ACC & \# Params & ACC & ACC &  \\
						\hline
						CE $X \to Y$ & $60.3 \pm 2.8$ & 18.08M & $59.3 \pm 2.1$ & 18.08M & $91.9 \pm 0.9$ & 1.12M & $62.1 \pm 3.2$ &  $64.3 \pm 1.0$ & 28.27M \\
						\hline
						DANN &  $58.9 \pm 1.7$ & 19.13M & $60.1 \pm 2.6$ & 26.49M & $84.8 \pm 0.7$ & 1.1M & $61.0 \pm 1.5$ & $65.2 \pm 1.1$ & 30.21M \\
						\hline 
						MMD-AAE & $60.8 \pm 3.4$ & 19.70M & $64.8 \pm 7.7$ & 19.70M & $92.5 \pm 0.8$ & 1.23M & $60.3 \pm 2.2$ & $65.2 \pm 1.5$ & 36.68M \\
						\hline
						{DIVA} &  {$58.8 \pm 3.4$} & 14.86M & {$58.1 \pm 1.4$} & 14.87M & {$86.1 \pm 1.0$} & 1.69M & {$61.8 \pm 1.8$} & {$64.8 \pm 0.8$} & 33.22M \\
						\hline
						IRM & $61.4 \pm 3.8$ & 18.08M &  $62.8 \pm 4.6$ & 18.08M & $92.9 \pm 1.2$ & 1.12M & $62.2 \pm 2.6$ &  $65.2 \pm 1.1$ & 28.27M\\
						\hline
						sVAE & $60.4 \pm 2.1$ & 18.25M & $64.3 \pm 1.2$ & 19.70M & $93.6 \pm 0.9$ & 0.92M & $62.7 \pm 2.5$  &  $66.6 \pm 0.8$ & 37.78M \\
						\midrule
						LaCIM (\textbf{Ours}) & \textbf{\bm{$63.2 \pm 1.7$}}  & 18.25M & \textbf{$\mathbf{66.4 \pm 2.2}$} & 19.70M & \textbf{$\mathbf{96.6 \pm 0.3}$} & 0.92M & \textbf{$\mathbf{63.8 \pm 1.1}$} & \textbf{$\mathbf{67.3 \pm 0.9}$} & 37.78M \\
						\bottomrule
					\end{tabular}
				}
			\end{table*}
			\subsection{Real-world OOD Challenge}
			\label{sec:expm-ood}

			We present our LaCIM's results on three OOD tasks.
			
			\textbf{Dataset.}  We describe the datasets as follows {(the $X$ denotes the input; the $Y$ denote the label)}:
			
			\textit{NICO}: we evaluate the cat/dog classification in ``Animal'' dataset in NICO, a benchmark for non-i.i.d problem in \cite{he2019towards}. Each animal is associated with ``grass",``snow" contexts. The $D$ denotes sampler's attributes. We consider two settings: $m=8$ and $m=14$. The $C,Z,S$ respectively denote the (time,whether) of sampling, the context and semantic shape of cat/dog.
			
			\textit{CMNIST}: We relabel the digits 0-4 and 5-9 as $y=0$ and $y=1$, based on MNIST. Then we color $p^e$ ($1-p^e$) of images with $y=0$ ($y=1$) as green and color others as red. We set $m=2$ with $p^{e_1} = 0.95, p^{e_2} = 0.99$. The $D$ can denote the attributes of the painter. We do not flip $y$ with $25\%$ like \cite{arjovsky2019invariant} \footnote{We conduct the flipping setting in supplementary~\ref{appx-exp-cmnist}.}, since doing so will cause the digit correlated rather than causally related to the label, which is beyond our scope. The $Z,S$ respectively represent the color and number. The $C$ can also denote (time,whether) for which the painter $D$ draws the number and color, \emph{e.g.}, the painter tends to draw red 0 more often than green 1 in the sunny morning.
			
			\textit{ADNI.} The $\mathcal{Y}:=\{0,1,2\}$, with 0,1,2 respectively denoting AD, Mild Cognitive Impairment and Normal Control. The $X$ is structural Magnetic resonance imaging. The $m=2$. The $D$ respectively denotes Age, TAU (a biomarker \cite{humpel2011cerebrospinal}). The $S$ ($Z$) denotes the disease-related (-unrelated) brain regions. The $C$ can be the hormone level that affects the brain structure development.
			
			\textbf{Compared Baselines.} We compare with (i) Cross-Entropy (CE) from $X \to Y$ (CE $X \to Y$), (ii) domain-adversarial neural network (DANN) for domain adaptation \cite{ganin2016domain}, (iii) Maximum Mean Discrepancy with Adversarial Auto-Encoder (MMD-AAE) for domain generalization \cite{li2018domain}, (iv) Domain Invariant Variational Autoencoders (DIVA) \cite{ilse2019diva}, (v) Invariant Risk Mnimization (IRM) \cite{arjovsky2019invariant},  (vi) Supervised VAE: our LaCIM implemented by VAE without disentangling $S,Z$ and we call it sVAE for simplicity. 
			
			
			\textbf{Implementation Details.} The network structures of $q^e_\psi(s,z|x)$, $p_\theta(x|s,z)$ and $p_\theta(y|s)$ for CMNIST, NICO and ADNI are introduced in supplement~\ref{appx-exp-cmnist},~\ref{appx-exp-nico},~\ref{appx:expm-ad}, Tab.~\ref{tab:framework},~\ref{tab:module}. We implement SGD as optimizer, with learning rate (lr) 0.5 and weight decay (wd) $1e$-$5$ for CMNIST; lr 0.01 with decaying 0.2$\times$ every 60 epochs, wd $5e$-$5$ for NICO and ADNI (wd is $2e$-$4$). The batch-size are set to 256, 30 and 4 for CMNIST, NICO, ADNI.

			\textbf{Main Results \& Discussions.} We report accuracy over three runs for each method. As shown in Tab.~\ref{table:acc} \footnote{On NICO, we implement ConvNet with Batch Balancing as a specifically benchmark in \cite{he2019towards}. The results are $60 \pm 1$ on $m=8$ and $62.33 \pm 3.06$ on $m=14$.} our LaCIM performs consistently better than others on all applications. The advantage over IRM and CE $X \to Y$ can be contributed to our learning method guided by the causal structure in Fig.~\ref{fig:lacim} and identification of true causal mechanisms. Further, the improvement over sVAE is benefited from our separation of $Y$-causative factor (\textit{a.k.a}, $S$) from others to avoid spurious correlation. Besides, as shown from results on NICO, a larger $m$ (with the total number of samples $n$ fixed) can bring further benefit, which may due to the easier satisfaction of the diversity condition in theorem~\ref{thm:iden}.

			\textbf{Interpretability.} We visualize learned $S$ (and also $Z$) as a side proof of interpretability. We consider two visualization methods, \emph{i.e.}, gradient method \cite{simonyan2013deep} and interpolation, respectively applied on NICO and CMNIST. 
			
			Specifically, for gradient method, we select the dimension of $S$ that has the highest correlation with $y$ among all dimension of $S$, and visualize the derivatives of such dimension of $S$ with respect to the image. For CE $x\to y$, we visualize the derivatives of predicted class scores with respect to the image. As shown in Fig.~\ref{fig:cat-dog-visualize}, LaCIM (the 3rd column) can identify more explainable semantic features than the CE $X \to Y$ which can learn the background information, which verifies the identifiability and effectiveness of the learning method. Supplement~\ref{appx-exp-nico} provides more results. 
			\begin{figure}[h!]
				\centering
				\begin{tabular}{ccccccc}
					& \includegraphics[width=0.2\textwidth]{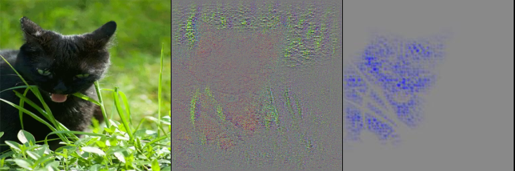}
					& \includegraphics[width=0.2\textwidth]{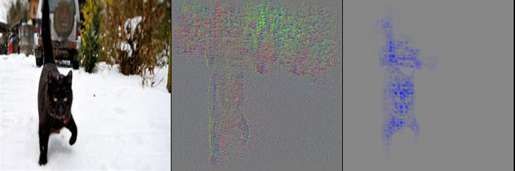}\\
					& (a)  Cat on grass & (b) Cat on snow \\
					& \includegraphics[width=0.2\textwidth]{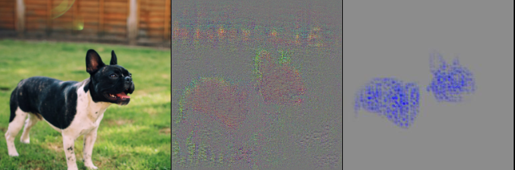}
					& \includegraphics[width=0.2\textwidth]{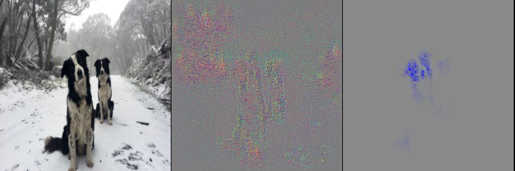} 
					\\
					& (c) Dog on grass & (d) Dog on snow
				\end{tabular}
				\label{fig:cat-dog-visualize}
				\caption{Visualization via gradient \cite{simonyan2013deep}. From the left to right: original image, CE $X \to Y$ and LaCIM.}
				\vspace{-0.2cm}
			\end{figure}
			
			For interpolation, we visualize the generated image by interpolating (one dimension of) $S$ (and $Z$) with fixed $Z$ (and $S$). As shown in Fig.~\ref{fig:vis}, the generated sequential images in 1st row looks more like ``7" from ``0" as $s$ increases; while the sequential images in the 2nd row changes from red to green as $z$ increases. This result can verify the disentanglement of learned $S$ and $Z$, \emph{i.e.}, which respectively capture the digit and color related features. 
			\begin{figure}
				\centering
				\includegraphics[width=0.3\textwidth]{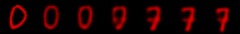}\\
				\includegraphics[width=0.3\textwidth]{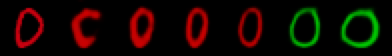}
				\caption{Interpolation of $S$ (and $Z$) with $Z$ (and $S$) fixed.}
				\label{fig:vis}
				\vspace{-0.35cm}
			\end{figure}

			\textbf{Results on Intervened Data.} We test the robustness of our model on intervened data generated from NICO. Each image is generated from a paired image (image A, image B): combining the scene of image A with the animal from image B. This is equivalent to intervention on the latent space. We generate 120 images. As shown in Tab~\ref{tab:interven}, our LaCIM can outperform others.
			\vspace{-0.4cm}
			\begin{table}[h!]
				\centering 
				\small
				\caption{ACC on intervened dataset from NICO.}
				\begin{tabular}{c|c|c|c}
					\hline
					Method & IRM   & DANN  & NCBB  \\
					\hline 
					\textbf{ACC}   & 50.00 & 49.17 & 49.17 \\
					\hline 
					Method & MMD-AAE & DIVA  & LaCIM (\textbf{Ours}) \\
					\hline 
					\textbf{ACC} &  49.17   & 50.00 & \textbf{55.00}        \\ \hline
				\end{tabular}
				\label{tab:interven}
			\end{table}
			

			
			\section{Conclusions \& Discussions}
			
			
			We propose identifying \emph{causal invariance} that is robust to a set of causal models augmented with an auxiliary domain variable that brings about the distribution shift, in order for out-of-distribution prediction. To model high-level concepts of sensory-level data, we introduce latent factors with explicit disentanglement with respect to the output. Under high degree of diversity of multiple domains, we can prove that the $Y$-causative factor can be disentangled from $Y$-non-causative factor and that the causal invariance can be identified. A possible drawback of our model lies in our requirement of the number of environments (which may be not satisfied in some scenarios) for identifiability, and the relaxation of which is left in the future work.

			\newpage
			\bibliographystyle{agsm}
			\begin{small}
				\bibliography{ref_causal}
			\end{small}

			
			\newpage
			\onecolumn
			\section{Supplementary Materials}

			
			\subsection{O.O.D Generalization error Bound}
			\label{appx:proofs-ood-bound}
			\vspace{0.1cm}
			
			Denote $\mathbb{E}_p[y|x] := \int_{\mathcal{Y}} y p(y|x) dy$ for any $x,y \in \mathcal{X} \times \mathcal{Y}$. We have 
			$\mathbb{E}_{p^e}[y|s] = \int_{\mathcal{Y}} yp(y|s) dy$ according to that $p(y|s)$ is invariant across $\mathcal{E}$, we can omit $p^e$ in $\mathbb{E}_{p^e}[y|s]$ and denote $g(S):=\mathbb{E}[Y|S]$. Then, the OOD bound $\big| \mathbb{E}_{p^{e_1}}(y|x) - \mathbb{E}_{p^{e_2}}(y|x) \big|, \ \forall (x,y)$ is bounded as follows:
			\begin{theorem}[OOD genearlization error]
				\label{thm:ood}
				Consider two causal models in LaCIM $\mbox{P}^{e_1}$ and $\mbox{P}^{e_2}$, suppose that their densities, \textit{i.e.}, $p^{e_1}(s|x)$ and $p^{e_2}(s|x)$ are absolutely continuous having support $(-\infty, \infty)$. For any $(x,y) \in \mathcal{X} \times \mathcal{Y}$, assume that 
				\begin{itemize}[noitemsep,topsep=0pt]
					\item $g(S)$ is a Lipschitz-continuous function;
					\item $\pi_x(s) := \frac{p^{e_2}(s|x)}{p^{e_1}(s|x)}$ is differentiable and $\mathbb{E}_{p^{e_1}}\left[\pi_x(S)\big|g(S)-\mu_1\big|\right] < \infty$ with $\mu_1 := \mathbb{E}_{p^{e_1}}[g(S)|X=x] = \int_{\mathcal{S}} g(s)p^{e_1}(s|x)ds$;
				\end{itemize}
				then we have $\big| \mathbb{E}_{p^{e_1}}(y|x) - \mathbb{E}_{p^{e_2}}(y|x) \big| \leq \Vert g' \Vert_{\infty} \Vert \pi_x' \Vert_\infty \mathrm{Var}_{p^{e_1}}(S|X=x)$.
				\label{thm.ood_gene_error}\end{theorem}
			
			When $e_1 \in \mathcal{E}_{\mathrm{train}}$ and $e_2 \in \mathcal{E}_{\mathrm{test}}$, the theorem~\ref{thm:ood} describes the error during generalization on $e_2$ for the strategy that trained on $e_1$. The bound is mainly affected by: (i) the Lipschitz constant of $g$, \textit{i.e.}, $\Vert g \Vert_{\infty}$; (ii) $\Vert \pi'_x \Vert_{\infty}$ which measures the difference between $p^{e_1}(s,z)$ and $p^{e_2}(s,z)$; and (iii) the $\mathrm{Var}_{p^{e_1}}(S|x)$ that measures the intensity of $x \to (s,z)$. These terms can be roughly categorized into two classes: (i),(iii) which are related to the property of CI and gave few space for improvement; and the (ii) that describes the distributional change between two environments. Specifically for the first class, the (i) measures the smoothness of $\mathbb{E}(y|s)$ with respect to $s$. The smaller value of $\Vert g' \Vert_{\infty}$ implies that the flatter regions give rise to the same prediction result, hence easier transfer from $e_1$ to $e_2$ and vice versa. For the term (iii), consider the deterministic setting that $\varepsilon_x = 0$ (leads to $\mathrm{Var}_{p^{e_1}}(S|x)=0$), then $s$ can be determined from $x$ for generalization if the $f$ is bijective function.
			
			The term (ii) measures the distributional change between posterior distributions $p^{e_1}(s|x)$ and $p^{e_2}(s|x)$, which contributes to the difference during prediction: $\big| \mathbb{E}_{p^{e_1}}(y|x) - \mathbb{E}_{p^{e_2}}(y|x)\big| = \int_{\mathcal{S}} (p^{e_1}(s|x) - p^{e_1}(s|x))p_{f_y}(y|s)ds$. Such a change is due to the inconsistency between priors $p^{e_1}(s,z)$ and $p^{e_2}(s,z)$, which is caused by different value of the confounder $d$.

			\begin{proof}
				In the following, we will derive the upper bound
				\begin{equation*}
				\big\lvert\mathbb{E}_{p^{e_{1}}}\left[Y|X\!=\!x\right]
				-\mathbb{E}_{p^{e_{2}}}\left[Y|X\!=\!x\right]\big\rvert
				\leq
				\|g^{\prime}\|_{\infty}\|\pi_{x}^{\prime}\|_{\infty}\operatorname{Var}_{p^{e_{1}}}\left(S|X=x\right),
				\end{equation*}
				where $\pi_{x}(s)=: \frac{p^{e_{2}}(s|x)}{p^{e_{1}}(s|x)}$ and $g(s)$ is assumed to be Lipschitz-continuous.

				To begin with, note that
				\begin{equation*}
				\mathbb{E}[Y| X]=\mathbb{E}[\mathbb{E}(Y| X, S)| X]=\mathbb{E}[g(S)| X]=\int g(s)p(s|x)ds.
				\label{eq.y_expect}
				\end{equation*}
				Let $p_{1}(s|x) = p^{e_{1}}(s|x)$, $p_{2}(s|x) = p^{e_{2}}(s|x)$.
				For ease of notations, we use $P_{1}$ and $P_{2}$ denote the distributions with densities $p_{1}(s|x)$ and $p_{2}(s|x)$ and suppose $S_{1}\sim P_{1}$ and  $S_{2}\sim P_{2}$, where $x$ is omitted as the following analysis is conditional on a fixed $X\!=\!x$.

				Then we may rewrite the difference of conditional expectations as
				$$\mathbb{E}_{p^{e_{2}}}[Y | X=x]-\mathbb{E}_{p^{e_{1}}}[Y |X=x]
				=\mathbb{E}(g(S_{2}))-\mathbb{E}(g(S_{1})),$$
				where $\mathbb{E}[g(S_{j}))]=\int g(s)p_{j}(s|x)ds$ denotes the expectation over $P_{j}$. 
				
				Let $\mu_{1}: = \mathbb{E}_{p^{e_1}}[g(S)|X=x] =\mathbb{E}[g(S_{1})]= \int g(s)p_{1}(s|x)ds$. Then 
				\begin{equation*}
				\mathbb{E}_{p^{e_{2}}}\left[Y|X=x\right]
				-\mathbb{E}_{p^{e_{1}}}\left[Y|X=x\right]
				=\mathbb{E}(g(S_{2}))-\mathbb{E}(g(S_{1}))
				=\mathbb{E}\left[g(S_{2})-\mu_{1}\right].
				\label{eq.1}
				\end{equation*}
				Further, we have the following transformation
				\begin{equation}
				\mathbb{E}\left[g(S_{2})-\mu_{1}\right]
				=\int (g(s)-\mu_{1})\pi_{x}(s)p_{1}(s|x)ds
				=\mathbb{E}\left[(g(S_{1})-\mu_{1})\pi_{x}(S_{1})\right].
				\label{eq.key_step1}
				\end{equation}
				
				In the following, we will use the results of the Stein kernel function.
				Please refer to Definition \ref{def.stein} for a general definition.
				Particularly, for the distribution $P_{1}\sim p_{1}(s|x)$, the Stein kernel $\tau_{1}(s)$ is 
				\begin{equation}
				\tau_{1}(s)=\frac{1}{p_{1}(s|x)} \int_{-\infty}^{s}(\mathbb{E}(S_{1})-t) p_{1}(t|x) d t,
				\end{equation}
				where $\mathbb{E}(S_{1})=\int s\cdot p_{1}(s|x)ds$. Further, 
				we define $(\tau_{1}\circ g)(s)$ as
				\begin{equation}
				(\tau_{1}\circ g)(s)
				\!=\!\frac{1}{p_{1}(s|x)} \int_{-\infty}^{s}(\mathbb{E}(g(S_{1}))-g(t)) p_{1}(t|x)  d t
				\!=\!\frac{1}{p_{1}(s|x)} \int_{-\infty}^{s}(\mu_{1}-g(t)) p_{1}(t|x) d t.
				\label{eq.tau(h)1}
				\end{equation}

				Under the second condition listed in Theorem \ref{thm.ood_gene_error}, we may apply the result of Lemma \ref{prop1}.
				Specifically, by the equation (\ref{eq.tau_x_integ}), we have
				\begin{equation*}
				\mathbb{E}\left[(g(S_{1})-\mu_{1})\pi_{x}(S_{1})\right]=\mathbb{E}\left[(\tau_{1}\circ g)(S_{1}) \pi_{x}^{\prime}(S_{1})\right].
				\end{equation*}
				Then under the first condition in Theorem \ref{thm.ood_gene_error}, we can obtain the following inequality by Lemma \ref{prop2},
				\begin{equation}
				\begin{aligned}
				\mathbb{E}\left[(\tau_{1}\circ g)(S_{1}) \pi_{x}^{\prime}(S_{1})\right]\!&\!=
				\mathbb{E}\left[\left(\frac{(\tau_{1}\circ g)}{\tau_{1}} \pi_{x}^{\prime}\tau_{1}\right)(S_{1})\right]
				\leq \mathbb{E}\left[\Big|\frac{(\tau_{1}\circ g)}{\tau_{1}}(S_{1})\Big|\cdot\Big| \pi_{x}^{\prime}\tau_{1}(S_{1})\Big|\right]\\
				\!&\!\leq 
				\|g^{\prime}\|_{\infty}\mathbb{E}\left[|\left(\pi_{x}^{\prime}\tau_{1}\right)(S_{1})|\right]
				\leq
				\|g^{\prime}\|_{\infty}\|\pi_{x}^{\prime}\|_{\infty} \mathbb{E}\left[|\tau_{1}(S_{1})|\right].\label{eq.main2}
				\end{aligned}
				\end{equation}
				In the following, we show that the Stein kernel is non-negative, which enables $\mathbb{E}\left[|\tau_{1}(S_{1})|\right]=\mathbb{E}\left[\tau_{1}(S_{1})\right]$.
				According to the definition,
				$\tau_{1}(s) = \frac{1}{p_{1}(s|x)}\int_{-\infty}^{s}(\mathbb{E}(S_{1})-t)p_{1}(t|x)dt$,
				where $\mathbb{E}(S_{1})\!=\! \int_{-\infty}^{\infty}t\cdot p_1(t|x)dt$.
				Let $F_{1}(s)=\int_{-\infty}^{s}p_{1}(t|x)dt$ be the distribution function for $P_{1}$.
				Note that
				\begin{eqnarray*}
					\int_{-\infty}^{s}\mathbb{E}(S_{1})p_{1}(t|x)dt&=&F_{1}(s)\mathbb{E}(S_{1})=F_{1}(s)\operatorname{E}(S_{1}),\\
					\int_{-\infty}^{s}tp_{1}(t|x)dt&=&F_{1}(s)\int_{-\infty}^{s}t\frac{p_{1}(t|x)}{F_{1}(s)}dt
					=F_{1}(s)\operatorname{E}(S_{1}|S_{1}\leq s)\leq F_{1}(s)\operatorname{E}(S_{1}),
				\end{eqnarray*}
				The last inequality is based on {\small$\operatorname{E}(S_{1}|S_{1}\leq s)-\operatorname{E}(S_{1})\leq 0$} that can be proved as the following
				\begin{equation*}
				\begin{aligned}
				\int_{-\infty}^{s} t \frac{p_{1}(t|x)}{F_{1}(s)} d t-\int_{-\infty}^{\infty} t p_{1}(t|x)d t&=\int_{-\infty}^{s} t\left(\frac{1}{F_{1}(s)}-1\right) p_{1}(t|x) d t-\int_{s}^{\infty} t p_{1}(t|x) d t \\ 
				& \leq s \int_{-\infty}^{s}\left(\frac{1}{F_{1}(s)}-1\right) p_{1}(t|x) d t-s \int_{s}^{\infty} p_{1}(t|x) =0.
				\end{aligned}
				\label{eq.exp}
				\end{equation*}
				Therefore, $\tau_{1}(s)\geq0$ and hence $ \mathbb{E}\left[|\tau_{1}(S_{1})|\right]= \mathbb{E}\left[\tau_{1}(S_{1})\right]$ in (\ref{eq.main2}). 
				
				Besides, by equation (\ref{eq.tau_var}), the special case of Lemma \ref{prop1}, we have $$\mathbb{E}\left[\tau_{1}(S_{1})\right]=\operatorname{Var}(S_{1})=\operatorname{Var}_{p^{e_{1}}}(S|X=x).$$ 
				To sum up,
				\begin{eqnarray*}
					\mathbb{E}\left[(\tau_{1}\circ g)(S_{1}) \pi^{\prime}_{x}(S_{1})\right]
					\leq
					\|g^{\prime}\|_{\infty}\|\pi_{x}\|_{\infty} \mathbb{E}\left[\tau_{1}(S_{1})\right]=
					\|g^{\prime}\|_{\infty}\|\pi_{x}^{\prime}\|_{\infty}\operatorname{Var}_{p^{e_{1}}}(S|X=x).
					\label{eq.main3}
				\end{eqnarray*}
			\end{proof}

			
			\vspace{0.2cm}
			\begin{definition}[\bf the Stein Kernel $\tau_{P}$ of distribution $P$]
				\emph{
					Suppose $X\!\sim \!P$ with density $p$. 
					The Stein kernel of $P$ is the function $x \mapsto \tau_{P}(x)$ defined by
					\begin{equation}
					\tau_{P}(x)
					=\frac{1}{p(x)} \int_{-\infty}^{x}(\mathbb{E}(X)-y) p(y) d y,
					\end{equation}
					where Id is the identity function for $\mathrm{Id}(x)=x$.
					More generally, for a function $h$ satisfying $\mathbb{E}[|h(X)|]<\infty$, define $(\tau_{P}\circ h)(x)$ as
					\begin{equation*}
					(\tau_{P}\circ h)(x)
					=\frac{1}{p(x)} \int_{-\infty}^{x}(\mathbb{E}(h(X))-h(y)) p(y) d y.
					\label{eq.tau(h)}
					\end{equation*}
				}\label{def.stein}
			\end{definition}
			
			\begin{lemma}
				For a differentiable function $\varphi$
				such that
				$\mathbb{E}[\lvert (\tau_{P}\circ h)(x) \varphi^{\prime}(X)\rvert]<\infty$, we have
				\begin{equation}
				\mathbb{E}\left[(\tau_{P}\circ h)(x) \varphi^{\prime}(X)\right]=\mathbb{E}[(h(X)-\mathbb{E}(h(X)) \varphi(X)].\label{eq.tau_x_integ}
				\end{equation}
				\label{prop1}
			\end{lemma}
			\vspace{-1cm}
			\begin{proof}
				Let $\mu_{h}=:\mathbb{E}(h(X))$. 
				As $\mathbb{E}(h(X)-\mu_{h})=0$,
				\begin{equation*}
				(\tau_{P}\circ h)(x)
				=\frac{1}{p(x)} \int_{-\infty}^{x}(\mu_{h}-h(y)) p(y) d y
				=\frac{-1}{p(x)} \int_{x}^{\infty}(\mu_{h}-h(y)) p(y) d y.
				\end{equation*}
				Then
				\begin{eqnarray*}
					&&\!\mathbb{E}\left[(\tau_{P}\circ h)(x) \varphi^{\prime}(X)\right]\!=\!
					\int_{-\infty}^{0} (\tau_{P}\circ h)(x) \varphi^{\prime}(x)p(x)dx+\int_{0}^{\infty} (\tau_{P}\circ h)(x) \varphi^{\prime}(x)p(x)dx\\
					&&\!=\!
					\int_{-\infty}^{0} \int_{-\infty}^{x}(\mu_{h}-h(y))p(y) \varphi^{\prime}(x)dydx
					-
					\int_{0}^{\infty} \int_{x}^{\infty}(\mu_{h}-h(y))p(y) \varphi^{\prime}(x)dydx\\
					&&\!=\!
					\int_{-\infty}^{0} \int_{y}^{0}(\mu_{h}-h(y))p(y) \varphi^{\prime}(x)dxdy
					-
					\int_{0}^{\infty} \int_{0}^{y}(\mu_{h}-h(y))p(y) \varphi^{\prime}(x)dxdy\\
					&&\!=\!
					\int_{-\infty}^{0} \int_{0}^{y}(h(y)-\mu_{h})p(y) \varphi^{\prime}(x)dxdy
					+
					\int_{0}^{\infty} \int_{0}^{y}(h(y)-\mu_{h})p(y) \varphi^{\prime}(x)dxdy\\
					&&\!=\!
					\int_{-\infty}^{\infty} (h(y)-\mu_{h})p(y)\left(\int_{0}^{y} \varphi^{\prime}(x)dx\right) dy \!=\!
					\int_{-\infty}^{\infty} (h(y)-\mu_{h})p(y)(\varphi(y)-\varphi(0))dy\\
					&&\!=\!
					\int_{-\infty}^{\infty} (h(y)-\mu_{h})p(y)(\varphi(y))dy\!=\!
					\mathbb{E}[(h(X)-\mathbb{E}(h(X))\varphi(X)]
				\end{eqnarray*}
				Particularly, taking $h(X)=X$ and $\varphi(X)=X-\mathbb{E}(X)$, we immediately have
				\begin{equation}
				\mathbb{E}(\tau_{P}(X)) = \operatorname{Var}(X)
				\label{eq.tau_var}
				\end{equation}
			\end{proof}

			\begin{lemma}
				Assume that $\mathbb{E}(|X|)<\infty$ and  the density $p$ is locally absolutely continuous on $(-\infty, \infty)$ and $h$ is a Lipschitz continuous function.
				Then we have $|f_{h}|\leq \|h^{\prime}\|_{\infty}$ for
				\begin{equation*}
				f_{h}(x)=\frac{(\tau_{P}\circ h)(x)}{\tau_{P}(x)}
				=\frac{\int_{-\infty}^{x}(\mathbb{E}(h(X))-h(y))p(y)dy}{\int_{-\infty}^{x}(\mathbb{E}(X)-y)p(y)dy}.
				\end{equation*}
				\label{prop2}
			\end{lemma}
			\begin{proof}
				This is a special case of Corollary 3.15 in \cite{dobler2015stein}, taking the constant $c=1$.
			\end{proof}

			

			\subsection{Proof of the Equivalence of Definition~\ref{p-identifiable}} 
			\label{appx:proofs-pid-equiv}
			
			\begin{proposition}
				\label{prop:p-equi-rela}
				The binary relation $\sim_p$ defined in Def.~\ref{p-identifiable} is an equivalence relation.
			\end{proposition}

			\begin{proof}
				The equivalence relation should satisfy three properties as follows:
				\begin{itemize}
					\item \emph{Reflexive} property: The $\theta \sim_p \theta$ with $M_z$, $M_s$ being identity matrix and $a_s$, $a_z$ being 0. 
					\item \emph{Symmtric} property: If $\theta \sim_p \tilde{\theta}$, then there exists block permutation matrices $M_z$ and $M_s$ such that 
					\begin{align}
					& \mathbf{T}^s([f_x]_{\mathcal{S}}^{-1}(x)) = M_s \tilde{\mathbf{T}}^s([\tilde{f}_x]_{\mathcal{S}}^{-1}(x)) + a_s, \;
					\mathbf{T}^z([f_x]_{\mathcal{Z}}^{-1}(x)) = M_z \tilde{\mathbf{T}}^z([\tilde{f}_x]_{\mathcal{Z}}^{-1}(x)) + a_z, \nonumber \\
					& p_{f_y}(y|[f_x]_{\mathcal{S}}^{-1}(x)) = p_{\tilde{f}_y}(y|[\tilde{f}_x]_{\mathcal{S}}^{-1}(x)). \nonumber
					\end{align}
					The we have $M_s^{-1}$ and $M_z^{-1}$ are also block permutation matrices and such that:
					\begin{align}
					& \tilde{\mathbf{T}}^s([\tilde{f}_x]_{\mathcal{S}}^{-1}(x)) = M_s^{-1}\mathbf{T}^s([f_x]_{\mathcal{S}}^{-1}(x)) + (-a_s), \  \tilde{\mathbf{T}}^s([\tilde{f}_x]_{\mathcal{Z}}^{-1}(x)) = M_z^{-1} \mathbf{T}^s([f_x]_{\mathcal{Z}}^{-1}(x)) + (-a_z), \nonumber \\
					& p_{\tilde{f}_y}(y|[\tilde{f}_x]_{\mathcal{S}}^{-1}(x)) = p_{f_y}(y|[f_x]_{\mathcal{S}}^{-1}(x)). \nonumber
					\end{align}
					Therefore, we have $\tilde{\theta} \sim_p \theta$. 
					\item \emph{Transitive} property: if $\theta_1 \sim_p \theta_2$ and $\theta_2 \sim_p \theta_3$ with $\theta_i := \{f^i_x,f^i_y,\mathbf{T}^{s,1},\mathbf{T}^{z,1}, \bm{\Gamma}^{s,i},\bm{\Gamma}^{z,i}\}$, then we have 
					\begin{align}
					& \mathbf{T}^{s,1}((f^1_{x,s})^{-1}(x)) = M^1_s \mathbf{T}^{s,2}((f^2_{x,s})^{-1}(x)) + a^1_s, \nonumber \\
					& \mathbf{T}^{z,1}((f^1_{x,z})^{-1}(x)) = M^1_z \mathbf{T}^{z,2}((f^2_{x,z})^{-1}(x)) + a^2_z, \nonumber \\
					& \mathbf{T}^{s,2}((f^2_{x,s})^{-1}(x)) = M^2_s \mathbf{T}^{s,3}((f^3_{x,s})^{-1}(x)) + a^2_s, \nonumber \\
					& \mathbf{T}^{z,2}((f^2_{x,z})^{-1}(x)) = M^2_z \mathbf{T}^{z,3}((f^3_z)^{-1}(x)) + a^3_{x,z} \nonumber
					\end{align}
					for block permutation matrices $M^1_s, M^1_z, M^2_s, M^2_z$ and vectors $a^1_s, a^2_s,a^1_z, a^2_z$. 
					Then we have
					\begin{align}
					& \mathbf{T}^{s,1}((f^1_{x,s})^{-1}(x)) = M^2_s M^1_s \mathbf{T}^{s,3}((f^3_{x,s})^{-1}(x)) + (M^2_s a^1_s) + a^2_s, \nonumber \\
					& \mathbf{T}^{z,1}((f^1_{x,z})^{-1}(x)) = M^2_z M^1_z \mathbf{T}^{z,3}((f^3_{x,z})^{-1}(x)) + (M^2_z a^1_z) + a^2_z. \nonumber 
					\end{align}
					Besides, it is apparent that
					\begin{align}
					p_{f^1_y}(y|(f^1_{x})_s^{-1}(x)) = p_{f^2_y}(y|(f^2_{x})_s^{-1}(x)) = p_{f^3_y}(y|(f^3_{x})_s^{-1}(x)). 
					\end{align}
					Therefore, we have $\theta_1 \sim_p \theta_3$ since $M^2_s M^1_s$ and $M^2_z M^1_z$ are also permutation matrices. 
				\end{itemize}
				With above three properties satisfied, we have that $\sim_p$ is a equivalence relation. 
			\end{proof}

			
			\subsection{Proof of Theorem~\ref{thm:iden}}
			\label{appx:proofs-iden}

			{In the following, we write $p^e(x,y)$ as $p(x,y|d^e)$ and also $\Gamma^{t=s,z}_{c,d^e} := \Gamma^{t=s,z}(c,d^e), A^t_{c,d^e,i} = A^t_i(c,d^e)$ for $t=s,z$. To prove the theorem~\ref{thm:iden}, we first prove the theorem~\ref{thm:iden-new} for the simplest case when $c|d^e := d^e$, then we generalize to the case when $\mathcal{C} := \cup_r \{c_r\}$. The overall roadmap is as follows: we first prove the $\sim_A$-identifiability in theorem~\ref{thm:iden-a}, and the combination of which with lemma~\ref{lemma:k1},~\ref{lemma:k2} give theorem~\ref{thm:iden-new} in the simplest case when $c|d^e = d^e$. Then we generalize the case considered in theorem~\ref{thm:iden-new} to the more general case when $\mathcal{C} :=  \cup_r \{c_r\}$. }
			
			\begin{theorem}[$\sim_p$-identifiability]
				\label{thm:iden-new}
				For $\theta$ in the LaCIM $p^e_\theta(x,y) \in \mathcal{P}_{\exp}$ for any $e \in \mathcal{E}_{\mathrm{train}}$, we assume that {\bf(1)} the CI satisfies that $f_x$, $f'_x$ and $f''_x$ are continuous and that $f_x, f_y$ are bijective; {\bf(2)} that the $\{T^t_{i,j}\}_{j \in [k_t]}$ are linearly indepndent and $T^t_{i,j}$ are twice differentiable for any $t=s,z, i \in [q_t], j \in [k_t]$; {\bf(3)} the exogenous variables satisfy that the characteristic functions of $\varepsilon_x,\varepsilon_y$ are almost everywhere nonzero; {\bf(4)} the number of environments, \textit{i.e.}, $m \geq \max(q_s*k_s, q_z*k_z)+1$ and $ \left[\bm{\Gamma}^{t=s,z}_{d^{e_2}} - \bm{\Gamma}^{t=s,z}_{d^{e_1}},...,\bm{\Gamma}^{t=s,z}_{d^{e_m}} - \bm{\Gamma}^{t=s,z}_{d^{e_1}} \right]$ have full column rank for both $t=s$ and $t=z$, we have that the parameters $\theta:=\{f_x,f_y, \mathbf{T}^s, \mathbf{T}^z\}$ are $\sim_p$ identifiable.
			\end{theorem}
			
			To prove theorem~\ref{thm:iden-new}, We first prove the $\sim_A$-identifiability that is defined as follows:
			
			\begin{definition}[$\sim_A$-identifiability]
				\label{A-identifiable}
				The definition is the same with the one defined in~\ref{p-identifiable}, with $M_s,M_z$ being invertible matrices which are not necessarily to be the permutation matrices in Def.~\ref{p-identifiable}.
			\end{definition}
			\begin{proposition}
				\label{prop:s-equi-rela}
				The binary relation $\sim_A$ defined in Def.~\ref{A-identifiable} is an equivalence relation. \end{proposition}
			\begin{proof}
				The proof is similar to that of proposition~\ref{prop:p-equi-rela}.
			\end{proof}
			The following theorem states that any LaCIM that belongs to $\mathcal{P}_{\exp}$ is $\sim_A$-identifiable. 
			

			\begin{theorem}[$\sim_A$-identifiability]
				\label{thm:iden-a}
				For $\theta$ in the LaCIM $p^e_\theta(x,y) \in \mathcal{P}_{\exp}$ for any $e \in \mathcal{E}_{\mathrm{train}}$, 
				{we assume {\bf(1)} the CI satisfies that $f_x,f_y$ are bijective; {\bf(2)} that the $\{T^t_{i,j}\}_{j \in [k_t]}$ are linearly indepndent and $T^t_{i,j}$ are differentiable for any $t=s,z, i \in [q_t], j \in [k_t]$; {\bf(3)} the exogenous variables satisfy that the characteristic functions of $\varepsilon_x,\varepsilon_y$ are almost everywhere nonzero; {\bf(4)} the number of environments, \textit{i.e.}, $m \geq \max(q_s*k_s, q_z*k_z)+1$ and 
					$\left[[\bm{\Gamma}^{t}_{d^{e_2}} - \bm{\Gamma}^{t}_{d^{e_1}}]^\mathsf{T},...,[\bm{\Gamma}^{t}_{d^{e_m}} - \bm{\Gamma}^{t}_{d^{e_1}}]^\mathsf{T}\right]^\mathsf{T}$ have full column rank for $t=s,z$, we have that the parameters $\{f_x,f_y, \mathbf{T}^s, \mathbf{T}^z\}$ are $\sim_p$ identifiable.}
			\end{theorem}
			
			\begin{proof}
				Suppose that $\theta=\{f_x,f_y,\mathbf{T}^s,\mathbf{T}^z\}$ and $\tilde{\theta} =  \{\tilde{f}_x,\tilde{g}_y,\tilde{\mathbf{T}}^s,\tilde{\mathbf{T}}^z\}$ share the same observational distribution for each environment $e \in \mathcal{E}_{\mathrm{train}}$, \textit{i.e.}, 
				\begin{align}
				\label{eq:marginal}
				p_{f_x,f_y,\mathbf{T}^s,\bm{\Gamma}^s,\mathbf{T}^z,\bm{\Gamma}^z}(x,y|d^e) = p_{\tilde{f}_x,\tilde{f}_y,\tilde{\mathbf{T}}^s,\tilde{\bm{\Gamma}}^s,\tilde{\mathbf{T}}^z,\tilde{\bm{\Gamma}}^z}(x,y|d^e).
				\end{align}
				Then we have 
				\begingroup
				\allowdisplaybreaks
				\begin{align}
				& p_{f_x,f_y,\mathbf{T}^s,\bm{\Gamma}^s,\mathbf{T}^z,\bm{\Gamma}^z}(x|d^e) = p_{\tilde{f}_x,\tilde{f}_y,\tilde{\mathbf{T}}^s,\tilde{\bm{\Gamma}}^s,\tilde{\mathbf{T}}^z,\tilde{\bm{\Gamma}}^z}(x|d^e) \label{eq:iden-a-step1-x-1} \\
				\Longrightarrow & \int_{\mathcal{S} \times \mathcal{Z}} p_{f_x}(x|s,z)p_{\mathbf{T}^s,\bm{\Gamma}^s,\mathbf{T}^z,\bm{\Gamma}^z}(s,z|d^e) dsdz = \int_{\mathcal{S} \times \mathcal{Z}} p_{\tilde{f}_x}(x|s,z)p_{\tilde{\mathbf{T}}^s,\tilde{\bm{\Gamma}}^s,\tilde{\mathbf{T}}^z,\tilde{\bm{\Gamma}}^z}(s,z|d^e) dsdz  \label{eq:iden-a-step1-x-2} \\
				\Longrightarrow & \ \  \int_{\mathcal{X}} p_{\varepsilon_x}(x-\bar{x})p_{\mathbf{T}^s,\bm{\Gamma}^s,\mathbf{T}^z,\bm{\Gamma}^z}(f_x^{-1}(\bar{x})|d^e) \mathrm{vol}J_{f_x^{-1}}(\bar{x}) d\bar{x} \label{eq:iden-a-step1-x-3} \\
				& \ \ \ \ \ \ \ \ \ \ \ \ \ \ \ \ \ \ \ \ \ \ = \int_{\mathcal{X} } p_{\varepsilon_x}(x-\bar{x}) p_{\tilde{\mathbf{T}}^s,\tilde{\bm{\Gamma}}^s,\tilde{\mathbf{T}}^z,\tilde{\bm{\Gamma}}^z}(\tilde{f}_x^{-1}(\bar{x})|d^e) \mathrm{vol}J_{\tilde{f}_x^{-1}}(\bar{x}) d\bar{x} \label{eq:iden-a-step1-x-4} \\
				\Longrightarrow & \ \ \int_{\mathcal{X}} \tilde{p}_{\mathbf{T}^s,\bm{\Gamma}^s,\mathbf{T}^z,\bm{\Gamma}^z,f_x}(\bar{x}|d^e)p_{\varepsilon_x}(x-\bar{x})d\bar{x} = \int_{\mathcal{X}} \tilde{p}_{\tilde{\mathbf{T}}^s,\tilde{\bm{\Gamma}}^s,\tilde{\mathbf{T}}^z,\tilde{\bm{\Gamma}}^z,\tilde{f}_x}(\bar{x}|d^e)p_{\varepsilon_x}(x-\bar{x})d\bar{x} \label{eq:iden-a-step1-x-5} \\
				\Longrightarrow & \ \  (\tilde{p}_{\mathbf{T}^s,\bm{\Gamma}^s,\mathbf{T}^z,\bm{\Gamma}^z,f_x} * p_{\varepsilon_x})(x|d^e) = (\tilde{p}_{\tilde{\mathbf{T}}^s,\tilde{\bm{\Gamma}}^s,\tilde{\mathbf{T}}^z,\tilde{\bm{\Gamma}}^z,\tilde{f}_x}) * p_{\varepsilon_x}(x|d^e) \label{eq:iden-a-step1-x-6} \\
				\Longrightarrow & \ \ 
				F[\tilde{p}_{\mathbf{T}^s,\bm{\Gamma}^s,\mathbf{T}^z,\bm{\Gamma}^z,f_x}](\omega)\varphi_{\varepsilon_x}(\omega) = F[\tilde{p}_{\tilde{\mathbf{T}}^s,\tilde{\bm{\Gamma}}^s,\tilde{\mathbf{T}}^z,\tilde{\bm{\Gamma}}^z,\tilde{f}_x}](\omega)\varphi_{\varepsilon_x}(\omega) \label{eq:iden-a-step1-x-7} \\
				\Longrightarrow & \ \  F[\tilde{p}_{\mathbf{T}^s,\bm{\Gamma}^s,\mathbf{T}^z,\bm{\Gamma}^z,f_x}](\omega) = F[\tilde{p}_{\tilde{\mathbf{T}}^s,\tilde{\bm{\Gamma}}^s,\tilde{\mathbf{T}}^z,\tilde{\bm{\Gamma}}^z,\tilde{f}_x}](\omega) \label{eq:iden-a-step1-x-8} \\
				\Longrightarrow & \ \  \tilde{p}_{\mathbf{T}^s,\bm{\Gamma}^s,\mathbf{T}^z,\bm{\Gamma}^z,f_x}(x|d^e) = \tilde{p}_{\tilde{\mathbf{T}}^s,\tilde{\bm{\Gamma}}^s,\tilde{\mathbf{T}}^z,\tilde{\bm{\Gamma}}^z,\tilde{f}_x}(x|d^e) \label{eq:iden-a-step1-x-9}
				\end{align}
				\endgroup
				where $\mathrm{vol}J_f(X) := \det(J_f(X))$ for any square matrix $X$ and function $f$ with ``$J$" standing for the Jacobian. The $\tilde{p}_{\mathbf{T}^s,\bm{\Gamma}^s,\mathbf{T}^z,\bm{\Gamma}^z,f_x}(x)$ in \eqref{eq:iden-a-step1-x-5} is denoted as $ p_{\mathbf{T}^s,\bm{\Gamma}^s,\mathbf{T}^z,\bm{\Gamma}^z}(f_x^{-1}(x|d^e) \mathrm{vol}J_{f^{-1}}(x)$. The '*' in \eqref{eq:iden-a-step1-x-6} denotes the convolution operator. The $F[\cdot]$ in \eqref{eq:iden-a-step1-x-7} denotes the Fourier transform, where $\phi_{\varepsilon_x}(\omega) = F[p_{\varepsilon_x}](\omega)$. Since we assume that the $\varphi_{\varepsilon_x}(\omega)$ is non-zero almost everywhere, we can drop it to get \eqref{eq:iden-a-step1-x-9}. Similarly, we have that:
				\begingroup
				\allowdisplaybreaks
				\begin{align}
				& p_{f_y,\mathbf{T}^s,\bm{\Gamma}^s}(y|d^e) = p_{\tilde{f}_y,\tilde{\mathbf{T}}^s,\tilde{\bm{\Gamma}}^s}(y|d^e) \label{eq:iden-a-step1-y-1} \\
				\Longrightarrow & \int_{\mathcal{S}} p_{f_y}(y|s)p_{\mathbf{T}^s,\bm{\Gamma}^s}(s|d^e) ds = \int_{\mathcal{S}} p_{\tilde{f}_y}(y|s)p_{\tilde{\mathbf{T}}^s,\tilde{\bm{\Gamma}}^s}(s|d^e) ds  \label{eq:iden-a-step1-y-2} \\
				\Longrightarrow & \ \  \int_{\mathcal{Y}} p_{\varepsilon_y}(y-\bar{y})p_{\mathbf{T}^s,\bm{\Gamma}^s}(f_y^{-1}(\bar{y})|d^e) \mathrm{vol}J_{f_y^{-1}}(\bar{y}) d\bar{y} \label{eq:iden-a-step1-y-3} \\
				& \ \ \ \ \ \ \ \ \ \ \ \ \ \ \ \ \ \ \ \ \ \ = \int_{\mathcal{Y}} p_{\varepsilon_y}(y-\bar{y}) p_{\tilde{\mathbf{T}}^s,\tilde{\bm{\Gamma}}^s}(\tilde{f}_y^{-1}(\bar{y})|d^e) \mathrm{vol}J_{\tilde{g}^{-1}}(\bar{y}) d\bar{y} \label{eq:iden-a-step1-y-4} \\
				\Longrightarrow & \ \ \int_{\mathcal{S}} \tilde{p}_{\mathbf{T}^s,\bm{\Gamma}^s,f_y}(\bar{y}|d^e)p_{\varepsilon_y}(y-\bar{y})d\bar{y} = \int_{\mathcal{S}} \tilde{p}_{\tilde{\mathbf{T}}^s,\tilde{\bm{\Gamma}}^s,\tilde{f}_y}(\bar{y}|d^e)p_{\varepsilon_y}(y-\bar{y})d\bar{y} \label{eq:iden-a-step1-y-5} \\
				\Longrightarrow & \ \  (\tilde{p}_{\mathbf{T}^s,\bm{\Gamma}^s,f_y} * p_{\varepsilon_y})(y|d^e) = (\tilde{p}_{\tilde{\mathbf{T}}^s,\tilde{\bm{\Gamma}}^s,\tilde{f}_y} * p_{\varepsilon_y})(y|d^e) \label{eq:iden-a-step1-y-6} \\
				\Longrightarrow & \ \ 
				F[\tilde{p}_{\mathbf{T}^s,\bm{\Gamma}^s,f_y}](\omega)\varphi_{\varepsilon_y}(\omega) = F[\tilde{p}_{\tilde{\mathbf{T}}^s,\tilde{\bm{\Gamma}}^s,\tilde{f}_y}](\omega)\varphi_{\varepsilon_y}(\omega) \label{eq:iden-a-step1-y-7} \\
				\Longrightarrow & \ \  F[\tilde{p}_{\mathbf{T}^s,\bm{\Gamma}^s,f_y}](\omega) = F[\tilde{p}_{\tilde{\mathbf{T}}^s,\tilde{\bm{\Gamma}}^s,\tilde{f}_y}](\omega) \label{eq:iden-a-step1-y-8} \\
				\Longrightarrow & \ \  \tilde{p}_{\mathbf{T}^s,\bm{\Gamma}^s,f_y}(y) = \tilde{p}_{\tilde{\mathbf{T}}^s,\tilde{\bm{\Gamma}}^s,\tilde{f}_y}(y) \label{eq:iden-a-step1-y-9}, 
				\end{align}
				\endgroup
				and that
				\begingroup
				\allowdisplaybreaks
				\begin{align}
				& p_{f_x,f_y\mathbf{T}^s,\bm{\Gamma}^s,\mathbf{T}^z,\bm{\Gamma}^z}(x,y|d^e) = p_{\tilde{f}_x,\tilde{f}_y,\tilde{\mathbf{T}}^s,\tilde{\bm{\Gamma}}^s,\tilde{\mathbf{T}}^z,\tilde{\bm{\Gamma}}^z}(x,y|d^e) \label{eq:iden-a-step1-xy-1} \\
				\Longrightarrow & \int_{\mathcal{S} \times \mathcal{Z}} p_{f_x}(x|s,z)p_{f_y}(y|s)p_{\mathbf{T}^s,\bm{\Gamma}^s,\mathbf{T}^z,\bm{\Gamma}^z}(s,z|d^e) dsdz \nonumber \\
				& \ \ \ \ \ \ \ \ \ \ \ \ \ \ \ \ \ \ \ \ \ \  = \int_{\mathcal{S} \times \mathcal{Z}} p_{\tilde{f}}(x|s,z)p_{\tilde{f}_y}(y|s)p_{\tilde{\mathbf{T}}^s,\tilde{\bm{\Gamma}}^s,\tilde{\mathbf{T}}^z,\tilde{\bm{\Gamma}}^z}(s,z|d^e) dsdz  \label{eq:iden-a-step1-xy-2} \\
				\Longrightarrow & \ \  \int_{\mathcal{V}} p_{\varepsilon}(v-\bar{v})p_{\mathbf{T}^s,\bm{\Gamma}^s,\mathbf{T}^z,\bm{\Gamma}^z}(h^{-1}(\bar{v})|d^e) \mathrm{vol}J_{h^{-1}}(\bar{v}) d\bar{v} \label{eq:iden-a-step1-xy-3} \\
				& \ \ \ \ \ \ \ \ \ \ \ \ \ \ \ \ \ \ \ \ \ \ = \int_{\mathcal{V}} p_{\varepsilon}(v-\bar{v}) p_{\tilde{\mathbf{T}}^s,\tilde{\bm{\Gamma}}^s,\tilde{\mathbf{T}}^z,\tilde{\bm{\Gamma}}^z}(\tilde{h}^{-1}(\bar{v})|d^e) \mathrm{vol}J_{\tilde{h}^{-1}}(\bar{v}) d\bar{v} \label{eq:iden-a-step1-xy-4} \\
				\Longrightarrow & \ \ \int_{\mathcal{S} \times \mathcal{Z}} \tilde{p}_{\mathbf{T}^s,\bm{\Gamma}^s,\mathbf{T}^z,\bm{\Gamma}^z,h,c}(\bar{v}|d)p_{\varepsilon}(v-\bar{v})d\bar{v} = \int_{\mathcal{S} \times \mathcal{Z}} \tilde{p}_{\tilde{\mathbf{T}}^s,\tilde{\bm{\Gamma}}^s,\tilde{\mathbf{T}}^z,\tilde{\bm{\Gamma}}^z,\tilde{h},d^e}(\bar{v}|d^e)p_{\varepsilon}(v-\bar{v})d\bar{v} \label{eq:iden-a-step1-xy-5} \\
				\Longrightarrow & \ \  (\tilde{p}_{\mathbf{T}^s,\bm{\Gamma}^s,\mathbf{T}^z,\bm{\Gamma}^z,h} * p_{\varepsilon})(v) = (\tilde{p}_{\tilde{\mathbf{T}}^s,\tilde{\bm{\Gamma}}^s,\tilde{\mathbf{T}}^z,\tilde{\bm{\Gamma}}^z,\tilde{h}} * p_{\varepsilon})(v) \label{eq:iden-a-step1-xy-6} \\
				\Longrightarrow & \ \ 
				F[\tilde{p}_{\mathbf{T}^s,\bm{\Gamma}^s,\mathbf{T}^z,\bm{\Gamma}^z,h}](\omega)\varphi_{\varepsilon}(\omega) = F[\tilde{p}_{\tilde{\mathbf{T}}^s,\tilde{\bm{\Gamma}}^s,\tilde{\mathbf{T}}^z,\tilde{\bm{\Gamma}}^z,\tilde{h}}](\omega)\varphi_{\varepsilon}(\omega) \label{eq:iden-a-step1-xy-7} \\
				\Longrightarrow & \ \  F[\tilde{p}_{\mathbf{T}^s,\bm{\Gamma}^s,\mathbf{T}^z,\bm{\Gamma}^z,h}](\omega) = F[\tilde{p}_{\tilde{\mathbf{T}}^s,\tilde{\bm{\Gamma}}^s,\tilde{\mathbf{T}}^z,\tilde{\bm{\Gamma}}^z,\tilde{h}}](\omega) \label{eq:iden-a-step1-xy-8} \\
				\Longrightarrow & \ \  \tilde{p}_{\mathbf{T}^s,\bm{\Gamma}^s,\mathbf{T}^z,\bm{\Gamma}^z,h}(v) = \tilde{p}_{\tilde{\mathbf{T}}^s,\tilde{\bm{\Gamma}}^s,\tilde{\mathbf{T}}^z,\tilde{\bm{\Gamma}}^z,h}(v) \label{eq:iden-a-step1-xy-9}, 
				\end{align}
				\endgroup
				where $v := [x^{\top},y^{\top}]^{\top}$, $\varepsilon := [\varepsilon_x^{\top},\varepsilon_y^{\top}]^{\top}$, $h(v) = [[f_x]_{\mathcal{Z}}^{-1}(x)^{\top},f_y^{-1}(y)^{\top}]^{\top}$. According to \eqref{eq:iden-a-step1-y-9}, we have
				\begingroup
				\allowdisplaybreaks
				\begin{align}
				& \log{\mathrm{vol}J_{f_y}(y)} + \sum_{i=1}^{q_s} \left( \log{B_i(f_{y,i}^{-1}(y)}) - \log{A_i(d^e)} + \sum_{j=1}^{k_s} T^s_{i,j}(f_{y,i}^{-1}(y))\Gamma^s_{i,j}(d^e)  \right) \nonumber \\
				& \quad \quad \quad \quad =  \log{\mathrm{vol}J_{\tilde{f}_y}(y)} + \sum_{i=1}^{q_s} \left( \log{\tilde{B}_i(\tilde{f}_{y,i}^{-1}(y)}) - \log{\tilde{A}_i(d^e)} + \sum_{j=1}^{k_s} \tilde{T}^s_{i,j}(\tilde{f}_{y,i}^{-1}(y))\tilde{\Gamma}^s_{i,j}(d^e)  \right) \label{eq:log}
				\end{align}
				Suppose that the assumption (4) holds, then we have 
				\begin{align}
				\langle \mathbf{T}^s(f_y^{-1}(y)),\bm{\overline{\Gamma}}^s(d^{e_k})\rangle + \sum_i \log{ \frac{A_i(d^{e_1})}{A_i(d^{e_k})} } = \langle \tilde{\mathbf{T}}^s(\tilde{f}_y^{-1}(y)),\bm{\overline{\tilde{\Gamma}}}^s(d^{e_k})\rangle + \sum_i \log{ \frac{\tilde{A}_i(d^{e_1})}{\tilde{A}_i(d^{e_k})} }
				\end{align}
				\endgroup
				for all $k \in [m]$, where $\bm{\bar{\Gamma}}(d) = \bm{\Gamma}(d) - \bm{\Gamma}(d^{e_1})$. Denote $\tilde{b}_s(k) = \sum_i \frac{\tilde{A}^s_i(d^{e_1})A^s_i(d^{e_k})}{\tilde{A}^s_i(d^{e_k})A^s_i(d^{e_1})}$ for $k \in [m]$, then we have
				\begin{align}
				\label{eq:equi-g}
				\bm{\overline{\Gamma}}^{s,\top} \mathbf{T}^s(f_y^{-1}(y)) = \bm{\overline{\tilde{\Gamma}}}^{s,\top} \tilde{\mathbf{T}}^s(\tilde{f}_y^{-1}(y)) + \tilde{b}_s,
				\end{align}
				Similarly, from \eqref{eq:iden-a-step1-x-9} and \eqref{eq:iden-a-step1-xy-9}, there exists $\tilde{b}_z,\tilde{b}_s$ such that 
				\begin{align}
				\label{eq:equi-f}
				\bm{\overline{\Gamma}}^{s,\top} \mathbf{T}^s([f_x]_{\mathcal{S}}^{-1}(x)) + \bm{\overline{\Gamma}}^{z,\top} \mathbf{T}^z([f_x]_{\mathcal{Z}}^{-1}(x)) = \bm{\overline{\tilde{\Gamma}}}^{s,\top} \tilde{\mathbf{T}}^s([\tilde{f}_x]_{\mathcal{S}}^{-1}(x)) + \bm{\overline{\tilde{\Gamma}}}^{z,\top} \tilde{\mathbf{T}}^z([\tilde{f}_x]_{\mathcal{Z}}^{-1}(x)) + \tilde{b}_z + \tilde{b}_s, 
				\end{align}
				where $\tilde{b}_z(k) = \sum_i \frac{\tilde{A}^z_i(d^{e_1})A^z_i(d^{e_k})}{\tilde{A}^z_i(d^{e_k})A^z_i(d^{e_1})}$ for $k \in [m]$; and that, 
				\begin{align}
				\label{eq:equi-fg}
				\bm{\overline{\Gamma}}^{s,\top} \mathbf{T}^s(f_y^{-1}(y)) + \bm{\overline{\Gamma}}^{z,\top} \mathbf{T}^z([f_{x}^{-1}]_{\mathcal{Z}}(x)) = \bm{\overline{\tilde{\Gamma}}}^{s,\top} \tilde{\mathbf{T}}^s(\tilde{f}_y^{-1}(y)) + \bm{\overline{\tilde{\Gamma}}}^{z,\top} \tilde{\mathbf{T}}^z([\tilde{f}_{x}^{-1}]_{\mathcal{Z}}(x)) + \tilde{b}_z + \tilde{b}_s.
				\end{align}
				Substituting \eqref{eq:equi-g} to \eqref{eq:equi-f} and \eqref{eq:equi-fg}, we have that 
				\begin{align}
				\bm{\overline{\Gamma}}^{z,\top} \mathbf{T}^z([f_{x}^{-1}]_{\mathcal{Z}}(y)) = \bm{\overline{\tilde{\Gamma}}}^{z,\top} \tilde{\mathbf{T}}^z([\tilde{f}_{x}^{-1}]_{\mathcal{Z}}(y)) + \tilde{b}_z, \ \bm{\overline{\Gamma}}^{s,\top} \mathbf{T}^s([f_{x}^{-1}]_{\mathcal{S}}(y)) = \bm{\overline{\tilde{\Gamma}}}^{s,\top} \tilde{\mathbf{T}}^s([\tilde{f}_{x}^{-1}]_{\mathcal{S}}(y)) + \tilde{b}_s.
				\end{align}
				According to assumption (4), the $\bm{\overline{\Gamma}}^{s,\top}$ and $\bm{\overline{\Gamma}}^{z,\top}$ have full column rank. Therefore, we have that 
				\begingroup
				\allowdisplaybreaks
				\begin{align}
				& \mathbf{T}^z([f_{x}^{-1}]_{\mathcal{Z}}(x)) = \left( \bm{\overline{\Gamma}}^{z} \bm{\overline{\Gamma}}^{z,\top} \right)^{-1} \bm{\overline{\tilde{\Gamma}}}^{z,\top} \tilde{\mathbf{T}}^z([\tilde{f}_{x}^{-1}]_{\mathcal{Z}}(x)) + \left( \bm{\overline{\Gamma}}^{z} \bm{\overline{\Gamma}}^{z,\top} \right)^{-1}\tilde{b}_z \label{eq:invert-1} \\
				& \mathbf{T}^s([f_{x}^{-1}]_{\mathcal{S}}(x)) = \left( \bm{\overline{\Gamma}}^{s} \bm{\overline{\Gamma}}^{s,\top} \right)^{-1} \bm{\overline{\tilde{\Gamma}}}^{s,\top} \tilde{\mathbf{T}}^s([\tilde{f}_{x}^{-1}]_{\mathcal{S}}(x)) + \left( \bm{\overline{\Gamma}}^{s} \bm{\overline{\Gamma}}^{s,\top} \right)^{-1}\tilde{b}_s. \label{eq:invert-2} \\
				& \mathbf{T}^s(f_y^{-1}(y)) = \left( \bm{\overline{\Gamma}}^{s} \bm{\overline{\Gamma}}^{s,\top} \right)^{-1} \bm{\overline{\tilde{\Gamma}}}^{s,\top} \tilde{\mathbf{T}}^s(\tilde{f}_y^{-1}(y)) + \left( \bm{\overline{\Gamma}}^{s} \bm{\overline{\Gamma}}^{s,\top} \right)^{-1}\tilde{b}_s. \label{eq:invert-3} 
				\end{align}
				\endgroup
				Denote $M_z := \left( \bm{\overline{\Gamma}}^{z} \bm{\overline{\Gamma}}^{z,\top} \right)^{-1} \bm{\overline{\tilde{\Gamma}}}^{z,\top}$, $M_s := \left( \bm{\overline{\Gamma}}^{s} \bm{\overline{\Gamma}}^{s,\top} \right)^{-1} \bm{\overline{\tilde{\Gamma}}}^{s,\top}$ and $a_s = \left( \bm{\overline{\Gamma}}^{s} \bm{\overline{\Gamma}}^{s,\top} \right)^{-1}\tilde{b}_s$, $a_z = \left( \bm{\overline{\Gamma}}^{z} \bm{\overline{\Gamma}}^{z,\top} \right)^{-1}\tilde{b}_z$. The left is to prove that $M_z$ and $M_s$ are invertible matrices. Denote $\bar{x} = f^{-1}(x)$. Applying the \citep[Lemma 3]{khemakhem2020variational} we have that there exists $k_s$ points $\bar{x}^1,...,\bar{x}^{k_s}$, $\tilde{\bar{x}}^1,...,\tilde{\bar{x}}^{k_z}$ such that 
				$\left( (\mathbf{T}^{s})'_i(\bar{x}^1),...,(\mathbf{T}^{s})'_i(\bar{x}^{k_s}) \right)$ for each $i \in [q_s]$ and $\left( (\mathbf{T}^{z})'_i(\tilde{\bar{x}}^1,...,(\mathbf{T}^{z})'_i(\tilde{\bar{x}}^{k_z}) \right)$ for each $i \in [q_t]$ are linearly independent. By differentiating \eqref{eq:invert-1} and \eqref{eq:invert-2} for each $\bar{x}^i$ with $i \in [q_s]$ and $\tilde{\bar{x}}^i$ with $i \in [q_z]$ respectively, we have that
				\begin{align}
				\left( J_{\mathbf{T}^s}(\bar{x}^1),...,J_{\mathbf{T}^s}(\bar{x}^{k_s}) \right) = M_s \left( J_{\mathbf{T}^s \circ \tilde{f}_x^{-1} \circ f_x}(\bar{x}^1),...,J_{\mathbf{T}^s \circ \tilde{f}_x^{-1} \circ f}(\bar{x}^{k_s}) \right) \\
				\left( J_{\mathbf{T}^z}(\tilde{\bar{x}}^1),...,J_{\mathbf{T}^z}(\tilde{\bar{x}}^{k_z}) \right) = M_z \left( J_{\mathbf{T}^z \circ \tilde{f}_x^{-1} \circ f_x}(\tilde{\bar{x}}^1),...,J_{\mathbf{T}^z \circ \tilde{f}_x^{-1} \circ f_x}(\tilde{\bar{x}}^{k_z}) \right).
				\end{align}
				The linearly independence of $\left( (\mathbf{T}^{s})'_i(\bar{x}^1),...,(\mathbf{T}^{s})'_i(\bar{x}^{k_s}) \right), \left( (\mathbf{T}^{z})'_i(\tilde{\bar{x}}^1,...,(\mathbf{T}^{z})'_i(\tilde{\bar{x}}^{k_z}) \right)$ imply that the $\left( J_{\mathbf{T}^s}(\bar{x}^1),...,J_{\mathbf{T}^s}(\bar{x}^{k_s}) \right)$ and $\left( J_{\mathbf{T}^z}(\tilde{\bar{x}}^1),...,J_{\mathbf{T}^z}(\tilde{\bar{x}}^{k_z}) \right)$ are invertible, which implies the invertibility of matrix $M_s$ and $M_z$. The rest is to prove $p_{f_y}(y|[f_x]_{\mathcal{S}}^{-1}(x)) = p_{\tilde{f}_y}(y|[\tilde{f}_x]_{\mathcal{S}}^{-1}(x))$. This can be shown by applying \eqref{eq:iden-a-step1-xy-2} again. Specifically, according to \eqref{eq:iden-a-step1-xy-2}, we have that 
				\begin{align}
				& \int_{\mathcal{X}} p_{\varepsilon_x}(x-\bar{x})p(y|[f_x]_{\mathcal{S}}^{-1}(\bar{x}))
				p_{\mathbf{T}^s,\bm{\Gamma}^s,\mathbf{T}^z,\bm{\Gamma}^z}(f^{-1}(\bar{x})|d^e) \mathrm{vol}J_{f^{-1}}(\bar{x}) d\bar{x} \nonumber \\
				& \quad \quad \quad \quad \quad   = \int_{\mathcal{X}} p_{\varepsilon_x}(x-\bar{x})p(y|[\tilde{f}_x]_{\mathcal{S}}^{-1}(\bar{x}))
				p_{\mathbf{T}^s,\bm{\Gamma}^s,\mathbf{T}^z,\bm{\Gamma}^z}(\tilde{f}^{-1}(\bar{x})|d^e) \mathrm{vol}J_{\tilde{f}^{-1}}(\bar{x}) d\bar{x}.  \label{eq:iden-f-y-step1-xy-1}
				\end{align}
				Denote $l_{\mathbf{T}^s,\bm{\Gamma}^s,\mathbf{T}^z,\bm{\Gamma}^z,f_y,f_x,y}(x) := p_{f_y}(y|[f_x]_{\mathcal{S}}^{-1}(\bar{x}))
				p_{\mathbf{T}^s,\bm{\Gamma}^s,\mathbf{T}^z,\bm{\Gamma}^z}(f^{-1}(\bar{x})|d^e) \mathrm{vol}J_{f_x^{-1}}(\bar{x})$, we have 
				\begin{align}
				& \int_{\mathcal{X}} p_{\varepsilon_x}(x-\bar{x}) l_{\mathbf{T}^s,\bm{\Gamma}^s,\mathbf{T}^z,\bm{\Gamma}^z,f_y,f_x,y}(\bar{x}) d\bar{x}  = \int_{\mathcal{X}} p_{\varepsilon_x}(x-\bar{x})l_{\mathbf{\tilde{T}}^s,\bm{\tilde{\Gamma}}^s,\mathbf{\tilde{T}}^z,\bm{\tilde{\Gamma}}^z,\tilde{f}_y,\tilde{f}_x,y}(\bar{x}) d\bar{x} \label{eq:iden-y-step1} \\
				\Longrightarrow & (l_{\mathbf{T}^s,\bm{\Gamma}^s,\mathbf{T}^z,\bm{\Gamma}^z,f_y,f_x,y} * p_{\varepsilon_x})(x|d^e) = (l_{\mathbf{\tilde{T}}^s,\bm{\tilde{\Gamma}}^s,\mathbf{\tilde{T}}^z,\bm{\tilde{\Gamma}}^z,\tilde{f}_y,\tilde{f}_x,y} * p_{\varepsilon_x})(x|d^e) \label{eq:iden-y-step2} \\
				\Longrightarrow & F[l_{\mathbf{\tilde{T}}^s,\bm{\tilde{\Gamma}}^s,\mathbf{\tilde{T}}^z,\bm{\tilde{\Gamma}}^z,\tilde{f}_y,\tilde{f}_x,y}](\omega)\varphi_{\varepsilon_x}(\omega) = F[l_{\mathbf{T}^s,\bm{\Gamma}^s,\mathbf{T}^z,\bm{\Gamma}^z,f_y,f_x,y}](\omega)\varphi_{\varepsilon_x}(\omega)  \label{eq:iden-y-step3} \\
				\Longrightarrow & F[l_{\mathbf{T}^s,\bm{\Gamma}^s,\mathbf{T}^z,\bm{\Gamma}^z,f_y,f_x,y}](\omega) = F[l_{\mathbf{\tilde{T}}^s,\bm{\tilde{\Gamma}}^s,\mathbf{\tilde{T}}^z,\bm{\tilde{\Gamma}}^z,\tilde{f}_y,\tilde{f}_x,y}](\omega) \label{eq:iden-y-step4} \\
				\Longrightarrow & l_{\mathbf{T}^s,\bm{\Gamma}^s,\mathbf{T}^z,\bm{\Gamma}^z,f_y,f_x,y}(x) = l_{\mathbf{\tilde{T}}^s,\bm{\tilde{\Gamma}}^s,\mathbf{\tilde{T}}^z,\bm{\tilde{\Gamma}}^z,\tilde{f}_y,\tilde{f}_x,y}(x) \label{eq:iden-y-step5} \\
				\Longrightarrow & p_{f_y}(y|[f_x]_{\mathcal{S}}^{-1}(x))
				p_{\mathbf{T}^s,\bm{\Gamma}^s,\mathbf{T}^z,\bm{\Gamma}^z}(f^{-1}(x)|d^e) \mathrm{vol}J_{f_x^{-1}}(x) \nonumber \\
				& \quad \quad \quad \quad  = p_{\tilde{f}_y}(y|[\tilde{f}_x]_{\mathcal{S}}^{-1}(x)) p_{\mathbf{\tilde{T}}^s,\bm{\tilde{\Gamma}}^s,\mathbf{\tilde{T}}^z,\bm{\tilde{\Gamma}}^z}(\tilde{f}^{-1}(x)|d^e) \mathrm{vol}J_{\tilde{f}_x^{-1}}(x). \label{eq:iden-y-step6}
				\end{align}
				Taking the $\log$ transformation on both sides of \eqref{eq:iden-y-step6}, we have that 
				\begin{align}
				& \log{p_{f_y}(y|[f_x]_{\mathcal{S}}^{-1}(x))} + \log{p_{\mathbf{T}^s,\bm{\Gamma}^s,\mathbf{T}^z,\bm{\Gamma}^z}(f^{-1}(x)|d^e)} + \log{\mathrm{vol}J_{f_x^{-1}}(x)} \nonumber \\
				& \quad \quad \quad \quad = \log{p_{\tilde{f}_y}(y|[\tilde{f}_x]_{\mathcal{S}}^{-1}(x))} + \log{p_{\mathbf{\tilde{T}}^s,\bm{\tilde{\Gamma}}^s,\mathbf{\tilde{T}}^z,\bm{\tilde{\Gamma}}^z}(\tilde{f}^{-1}(x)|d^e)} + \log{\mathrm{vol}J_{\tilde{f}_x^{-1}}(x)}. \label{eq:log-y}
				\end{align}
				Subtracting \eqref{eq:log-y} with $y_2$ from \eqref{eq:log-y} with $y_1$, we have 
				\begin{align}
				& \frac{p_{f_y}(y_2|[f_x]_{\mathcal{S}}^{-1}(x))}{p_{f_y}(y_1|[f_x]_{\mathcal{S}}^{-1}(x))} = \frac{p_{\tilde{f}_y}(y_2|[\tilde{f}_x]_{\mathcal{S}}^{-1}(x))}{p_{\tilde{f}_y}(y_1|[\tilde{f}_x]_{\mathcal{S}}^{-1}(x))} \\
				\Longrightarrow & \int_{\mathcal{Y}} \frac{p_{f_y}(y_2|[f_x]_{\mathcal{S}}^{-1}(x))}{p_{f_y}(y_1|[f_x]_{\mathcal{S}}^{-1}(x))} dy_2 = \int_{\mathcal{Y}} \frac{p_{\tilde{f}_y}(y_2|[\tilde{f}_x]_{\mathcal{S}}^{-1}(x))}{p_{\tilde{f}_y}(y_1|[\tilde{f}_x]_{\mathcal{S}}^{-1}(x))} dy_2 \\
				\Longrightarrow & p_{f_y}(y_1|[f_x]_{\mathcal{S}}^{-1}(x)) = p_{\tilde{f}_y}(y_1|[\tilde{f}_x]_{\mathcal{S}}^{-1}(x)),
				\end{align}
				for any $y_1 \in \mathcal{Y}$. This completes the proof. 
			\end{proof}
			
			\paragraph{Understanding the assumption {\bf(4)} in Theorem~\ref{thm:iden-a} and~\ref{thm:iden-new}.} Recall that we assume the $D$ in LaCIM is the source variable for generating data in corresponding domain. Here we also use the $\mathcal{D}$ to denote the space of $D$, then we have the following theoretical conclusion that the as long as the image set of $\mathcal{D}$ is not included in any sets with Lebesgue measure 0, the assumption (4) holds. This conclusion means that the assumption {\bf(4)} holds generically. For more general conclusion of assumption (4) with $\left[[\bm{\Gamma}^{t}_{d^{e_2}} - \bm{\Gamma}^{t}_{d^{e_1}}]^\mathsf{T},...,[\bm{\Gamma}^{t}_{d^{e_m}} - \bm{\Gamma}^{t}_{d^{e_1}}]^\mathsf{T}\right]^\mathsf{T}$ replaced by $\left[[\bm{\Gamma}^{t}_{c_2,d^{e_1}} - \bm{\Gamma}^{t}_{c_1,d^{e_1}}]^\mathsf{T},...,[\bm{\Gamma}^{t}_{c_R, d^{e_m}} - \bm{\Gamma}^{t}_{c_1,d^{e_1}}]^\mathsf{T}\right]^\mathsf{T}$, we have the similar conclusion with the set $\mathcal{D}$ replaced by $\mathcal{C} \otimes \mathcal{D}$. 
			
			\begin{theorem}
				\label{thm:assump-4-understand}
				Denote $h^{t=s,z}(d) := \left( \Gamma^t_{1,1}(d)-\Gamma^t_{1,1}(d^{e_1}),...,\Gamma^t_{q_t,k_t}(d)-\Gamma^t_{1,1}(d^{e_1})\right)^{\top}$, $h(\mathcal{D}) :=  h^s(\mathcal{S}) \oplus h^z(\mathcal{Z}) \subset \mathbb{R}^{q_z*k_z} \oplus  \mathbb{R}^{q_s*k_s}$, then assumption (4) holds if $h(\mathcal{D})$ is not included in any zero-measure set of $\mathbb{R}^{q_z*k_z} \oplus \mathbb{R}^{q_s*k_s}$. Denote $r_s := q_s * k_s$ and $r_z := q_z * k_z$. 
			\end{theorem}
			
			\begin{proof}
				With loss of generality, we assume that $r_s \leq r_z$. Denote $Q$ as the set of integers $q$ such that there exists $d^{e_2},...,d^{q+1}$ that the $\mathrm{rank}([h^z(d^{e_2}),...,h^z(d^{e_{q+1}})]) = \min(q,r_z)$ and $\mathrm{rank}([h^s(d^{e_2}),...,h^s(d^{e_{q+1}})]) = \min(q,r_s)$. Denote $u:= \max(Q)$. We discuss two possible cases for $u$, respectively:
				\begin{itemize}
					\item Case 1. $u < r_s \leq r_z$. Then there exists $d^{e_2},...,d^{e_{u+1}}$ s.t. $h^z(d^{e_2}),...,h^z(d^{e_{u+1}})$ and $h^s(d^{e_2}),...,h^s(d^{e_{u+1}})$ are linearly independent. Then $\forall c$, we have $h^z(d) \in L(h^z(d^{e_2}),...,h^z(d^{e_{u+1}}))$ or $h^s(d) \in L(h^s(d^{e_2}),...,h^s(d^{e_{u+1}}))$. Therefore, so we have
					$h^z(d) \oplus h^s(d) \in \left[ L(h^z(d^{e_2}),...,h^z(d^{e_{u+1}})) \oplus \mathbb{R}^{r_s} \right] \cup  \left[\mathbb{R}^{r_z} \oplus  L(h^s(d^{e_2}),...,h^s(d^{e_{u+1}}))  \right] $, which has measure 0 in $\mathbb{R}^{r_z} \oplus   \mathbb{R}^{r_s}$. 
					\item Case 2. $r_s \leq u < r_z$. Then there exists $d^{e_2},...,d^{e_{u+1}}$ s.t. $h^z(d^{e_2}),...,h^z(d^{e_{u+1}})$ are linearly independent and $rank([h^s(d^{e_{1}}),...,h^s(d^{e_{u}})]) = r_s$. Then $\forall c$, we have $h^z(d) \in L(h^z(d^{e_{1}}),...,h^z(d^{e_{u+1}}))$, which means that $h^z(d) \oplus h^s(d) \in  L(h^z(d^{e_{1}}),...,h^z(d^{e_{u+1}})) \oplus \mathbb{R}^{r_s}$, which has measure 0 in $\mathbb{R}^{r_z} \oplus \mathbb{R}^{r_s}$. 
				\end{itemize}
				The above two cases are contradict to the assumption that $h(\mathcal{D})$ is not included in any zero-measure set of $\mathbb{R}^{r_z} \oplus \mathbb{R}^{r_s}$. 
			\end{proof}

			\begin{lemma}
				\label{lemma:k2}
				Consider the cases when $k_s \geq 2$. Then suppose the assumptions in theorem~\ref{thm:iden-a} are satisfied. Further assumed that 
				\begin{itemize}
					\item The sufficient statistics $\mathbf{T}^s_{i,j}$ are twice differentiable for each $i \in [q_s]$ and $j \in [k_s]$.
					\item $f_y$ is twice differentiable. 
				\end{itemize}
				Then we have $M_s$ in theorem~\ref{thm:iden-a} is block permutation matrix.
			\end{lemma}

			\begin{proof}
				Directly applying \citep[Theorem 2]{khemakhem2020variational} with $f_x,A,b,\mathbf{T},x$ replaced by $f_y,M_s,a_s,\mathbf{T}^s,y$.
			\end{proof}

			\begin{lemma}
				\label{lemma:k1}
				Consider the cases when $k_s =1$. Then suppose the assumptions in theorem~\ref{thm:iden-a} are satisfied. Further assumed that 
				\begin{itemize}
					\item The sufficient statistics $\mathbf{T}^s_{i}$ are not monotonic for $i \in [q_s]$.
					\item $g$ is smooth.
				\end{itemize}
				Then we have $M_s$ in theorem~\ref{thm:iden-a} is block permutation matrix.
			\end{lemma}
			
			\begin{proof}
				Directly applying \citep[Theorem 3]{khemakhem2020variational} with $f_x,A,b,\mathbf{T},x$ replaced by $f_y,M_s,a_s,\mathbf{T}^s,y$.
			\end{proof}

			\begin{proof}[Proof of Theorem~\ref{thm:iden-new}]
				According to theorem~\ref{thm:iden-a}, there exist invertible matrices $M_s$ and $M_z$ such that 
				\begin{align}
				& \mathbf{T}(f_x^{-1}(x)) = A \tilde{\mathbf{T}}(\tilde{f}_x^{-1}(x)) + b \nonumber \\
				& \mathbf{T}^s([f_{x}^{-1}]_{\mathcal{S}}(x)) = M_s \tilde{\mathbf{T}}^s([\tilde{f}_{x}^{-1}]_{\mathcal{S}}(x)) + a_s. \nonumber \\
				& \mathbf{T}^s(f_y^{-1}(y)) = M_s \tilde{\mathbf{T}}^s(\tilde{f}_y^{-1}(y)) + a_s, \nonumber
				\end{align}
				where $\mathbf{T} = [\mathbf{T}^{s,\top}, \mathbf{T}^{z,\top}]^{\top}$, and 
				\begin{align}
				A =  \left( \begin{array}{cc} M_s & 0 \\ 0 & M_z \end{array} \right).
				\end{align}
				By further assuming that the sufficient statistics $\mathbf{T}^s_{i,j}$ are twice differentiable for each $i \in [q_s]$ and $j \in [k_s]$ for $k_s \geq 2$ and not monotonic for $k_s = 1$. Then we have that $M_s$ is block permutation matrix. By further assuming that $\mathbf{T}^z_{i,j}$ are twice differentiable for each $i \in [n_z]$ and $j \in [k_z]$ for $k_z \geq 2$ and not monotonic for $k_z = 1$ and applying the lemma~\ref{lemma:k2} and~\ref{lemma:k1} respectively, we have that $A$ is block permutation matrix. Therefore, $M_z$ is also a block permutation matrix. 
			\end{proof}
			
			\begin{proof}[Proof of Theorem~\ref{thm:iden}]
				We consider the general case when $\mathcal{C} := \cup_{r=1}^R \{c_r\}_{r=[R]}$. We have that 
				\begin{equation}
				\sum_{r=1}^{R}p_{\theta}(x,y|c_r)\operatorname{P}(C\!=\!c_r|d^e) = \sum_{r=1}^{R}p_{\tilde{\theta}}(x,y|c_r)\operatorname{P}(C\!=\!c_R|d^e).
				\label{eq.sum}
				\end{equation}
				The \eqref{eq:iden-a-step1-x-9} for each $e$ here can be replaced by 
				\begin{align}
				\label{eq:sum}
				\sum_{r=1}^R p(c_r|d^e)\tilde{p}_{\mathbf{T}^s,\bm{\Gamma}^s,\mathbf{T}^z,\bm{\Gamma}^z,f_x}(x|c_r) = \sum_{r=1}^R \tilde{p}(c_r|d^e)\tilde{p}_{\tilde{\mathbf{T}}^s,\tilde{\bm{\Gamma}}^s,\tilde{\mathbf{T}}^z,\tilde{\bm{\Gamma}}^z,\tilde{f}_x}(x|d^e).
				\end{align}
				According to \citet[Corollary 3]{barndorff1965identifiability}, if we additionally assume that 
				\begin{align*}
				\{ \left( \mathbf{T}^s([f^{-1}]_{\mathcal{S}}(x)), \mathbf{T}^z([f^{-1}]_{\mathcal{Z}}(x)) \right); \mathcal{B}(x) > 0\} \text{ contains a non-empty set}, 
				\end{align*}
				then we have that the $p(c_r|d^e) = \tilde{p}(c_r|d^e)$ for each $r \in [R], e$. In other words, the $L:= [P(C|d^{e_1})^{\mathsf{T}},...,P(C|d^{e_m})^{\mathsf{T}}]$ can be identified. 
				Let $\Delta = [\tilde{p}_{\mathbf{T}^s,\bm{\Gamma}^s,\mathbf{T}^z,\bm{\Gamma}^z,f_x}(x|c_1)-\tilde{p}_{\tilde{\mathbf{T}}^s,\tilde{\bm{\Gamma}}^s,\tilde{\mathbf{T}}^z,\tilde{\bm{\Gamma}}^z,\tilde{f}_x}(x|c_1), \cdots, \tilde{p}_{\mathbf{T}^s,\bm{\Gamma}^s,\mathbf{T}^z,\bm{\Gamma}^z,f_x}(x|c_m)-\tilde{p}_{\tilde{\mathbf{T}}^s,\tilde{\bm{\Gamma}}^s,\tilde{\mathbf{T}}^z,\tilde{\bm{\Gamma}}^z,\tilde{f}_x}(x|c_m)]^{\mathsf{T}}$, then the concantenation of \eqref{eq:sum} in a matrix form can be written as $L\Delta = 0$. Since we have assumed in assumption (5) theorem~\ref{thm:iden} that the $L$ has full column rank, therefore we have that $\Delta = 0$, \emph{i.e.} $\tilde{p}_{\mathbf{T}^s,\bm{\Gamma}^s,\mathbf{T}^z,\bm{\Gamma}^z,f_x}(x|c_r) = \tilde{p}_{\tilde{\mathbf{T}}^s,\tilde{\bm{\Gamma}}^s,\tilde{\mathbf{T}}^z,\tilde{\bm{\Gamma}}^z,\tilde{f}_x}(x|c_r)$ for each $r \in [R]$. The left proof is the same with the one in theorem~\ref{thm:iden-new}. 
			\end{proof}

			\subsection{Proof of Theorem~\ref{thm:p-iden}}
			\label{appx:proofs-p-iden}

			\begin{proof}[Proof of Theorem~\ref{thm:p-iden}] Due to ~\eqref{eq.sum}, it is suffices to prove the conclusion for every $c_r \in \{c_r\}_{r \in [R]}$. Motivated by \citet[Theorem 2]{barron1991approximation} that the distribution $p^e(s,z)$ defined on bounded set can be approximated by a sequence of exponential family with sufficient statistics denoted as polynomial terms, therefore the $\mathbf{T}^{t=s,z}$ are twice differentiable hence satisfies the assumption (2) in theorem~\ref{thm:iden} and assumption (1) in lemma~\ref{lemma:k2}. Besides, the lemma 4 in \cite{barron1991approximation} informs us that the KL divergence between $p_{\theta_0}(s,z|c_r)$ ($\theta_0 := (f_x,f_y,\bm{T}^z,\bm{T}^s,\bm{\Gamma}_0^z,\bm{\Gamma}_0^s$) and $p_{\theta_1}(s,z|c_r)$ ($\theta_1 := (f_x,f_y,\bm{T}^z,\bm{T}^s,\bm{\Gamma}_1^z,\bm{\Gamma}_1^s$) (the $p_{\theta_0}(s,z|c_r), p_{\theta_1}(s,z|c_r)$ belong to exponential family with polynomial sufficient statistics terms) can be bounded by the $\ell_2$ norm of $[(\bm{\Gamma}^s(c_r) - \bm{\Gamma}^s_1(c_r))^{\top},(\bm{\Gamma}^z_0(c_r) - \bm{\Gamma}^z_1(c_r))^{\top}]^{\top}$. Therefore, $\forall \epsilon > 0$, there exists a open set of $\Gamma(c_r)$ such that the $D_{\mathrm{KL}}(p(s,z|c_r),p_{\theta}(s,z|c_r)) < \epsilon$. Such an open set is with non-zero Lebesgue measurement therefore can satisfy the assumption (4) in theorem~\ref{thm:iden}, according to result in theorem~\ref{thm:assump-4-understand}. The left is to prove that for any $p$ defined 
				by a LaCIM following Def.~\ref{def:lacim}, there is a sequence of $\{p_m\}_{n} \in \mathcal{P}_{\exp}$ such that the $d_{\mathrm{Pok}}(p,p_n) \to 0$ that is equivalent to $p_n \overset{d}{\to} p$. For any $A,B$, we consider to prove that 
				\begin{align}
				I_{n} \overset{\Delta}{=} \bigg| p(x\in A,y\in B|c_r) - p_n(x \in A, y_n \in B|c_r) \bigg| \to 0,
				\end{align}
				where $p_n(x \in A, y_n \in B|c_r) = \int_{\mathcal{S}} \int_{\mathcal{Z}} p(x \in A|s,z)p(y_n \in B|s)p_n(s,z|c_r) dsdz $ with 
				\begin{align}
				y_n(i) = \frac{\exp( (f_{y,i}(\bm{s}) + \varepsilon_{y,i})/T_n ) }{\sum_i \exp( (f_{y,i}(\bm{s}) + \varepsilon_{y,i})/T_n ) }, \ i = 1,...,k, 
				\end{align}
				for $y \in \mathbb{R}^k$ denoting the $k$-dimensional one-hot vector for categorical variable and $\varepsilon_{y,1,...,_k}$ are Gumbel i.i.d. According to \citep[Proposition 1]{maddison2016concrete} that the $y_n(i) \overset{d}{\to} y(i)$ with 
				\begin{align}
				\label{eq:y}
				p(y(i) = 1) = \frac{\exp( f_{y,i}(\bm{s}) ) }{\sum_i \exp( (f_{y,i}(\bm{s})  )}, \ as \ T_n \to 0.
				\end{align}
				As long as $f_{y}$ is smooth, we have that the $p(y_n|s)$ is continuous. We have that 
				\begingroup
				\allowdisplaybreaks
				\begin{align}
				I_{n} & = \Big| p(x\in A,y\in B|c_r) - \int_{\mathcal{S} \times \mathcal{Z}} p(x\in A|s,z)p(y_n \in B|s)p_n(s,z|c_r) dsdz \Big| \nonumber \\
				& \leq \Big| p(x\in A,y\in B|c_r) - p(x\in A,y_n \in B|c_r)  \Big| \nonumber \\
				& \quad \quad \quad \quad  + \Big| p(x\in A,y_n \in B|c_r) - \int_{\mathcal{S} \times \mathcal{Z}} p(x\in A|s,z)p(y_n \in B|s)p_n(s,z|c_r)dsdz  \Big| \nonumber \\
				& = \Big| \int_{\mathcal{S} \times \mathcal{Z}} p(x \in A|s,z)\left( p(y\in B|s) - p(y_n \in B|s) \right) p(s,z|c_r) dsdz   \Big| \nonumber \\
				& \quad \quad \quad \quad + \Big| \int_{\mathcal{S} \times \mathcal{Z}} p(x \in A|s,z)p(y_n \in B|s)\left( p(s,z|c_r) - p_n(s,z|c_r) \right)   \Big| \nonumber \\
				& \leq \underbrace{\Big|   \int_{M_s \times M_z}  p(x \in A|s,z)\left( p(y\in B|s) - p(y_n \in B|s) \right) p(s,z|c_r) dsdz     \Big|}_{I_{n,1}} \nonumber \\
				& + \underbrace{\Big|   \int_{(M_s \times M_z)^{c_r}}  p(x \in A|s,z)\left( p(y\in B|s) - p(y_n \in B|s) \right) p(s,z|c_r) dsdz     \Big|}_{I_{n,2}} \nonumber \\
				& + \underbrace{\Big| \int_{M_s \times M_z} p(x \in A|s,z)p(y_n \in B|s)\left( p(s,z|c_r) - p_n(s,z|c_r) \right)   \Big|}_{I_{n,3}} \nonumber \\ 
				& +  \underbrace{\Big| \int_{(M_s \times M_z)^{c_r}} p(x \in A|s,z)p(y_n \in B|s)\left( p(s,z|c_r) - p_n(s,z|c_r) \right)   \Big|}_{I_{n,4}}. 
				\end{align}
				\endgroup
				For $I_{n,1}$, if $y$ is itself additive model with $y = f_y(\bm{s}) + \varepsilon_y$, then we just set $y_n \overset{d}{=} y$, then we have that $I_{n,1} = 0$. Therefore, we only consider the case when $y$ denotes the categorical variable with softmax distribution, \textit{i.e.}, \eqref{eq:y}. $\forall c_r \in \mathcal{C}:=\{c_1,...,c_R\}$ and $\forall \epsilon > 0$, there exists $M^{c_r}_s$ and $M^{c_r}_z$ such that $p(s,z \in M^{c_r}_s \times M^{c_r}_z|c_r) \leq \epsilon$; Denote $M_s \overset{\Delta}{=} \cup_{k=1}^{m} M_s^{c_r}$ and $M_z \overset{\Delta}{=} \cup_{k=1}^{m} M_z^{c_r}$, we have that $p(s,z \in M_s \times M_z|c) \leq 2 \epsilon$ for all $c_r \in \mathcal{C}$. Since $\forall s_1 \in M_s$, $\exists N_{s_1}$ such that $\forall n \geq N_{s_1}$, we have that $\Big| p(y \in B|s_1) - p(y \in B|s_1)| \leq \epsilon $ from that $y_n \overset{d}{\to} y$. Besides, there exists open set $\mathcal{O}_{s_1}$ such that $\forall s \in \mathcal{O}_{s_1}$ and
				\begin{align}
				\Big| p(y \in B|s_1) - p(y \in B|s_1)| \leq \epsilon,  \ \Big| p(y_n \in B|s_1) - p(y_n \in B|s_1)| \leq \epsilon. \nonumber 
				\end{align}
				Again, according to Heine–Borel theorem, there exists finite $s$, namely $s_1,...,s_l$ such that $M_s \subset \cup_{i=1}^l \mathcal{O}(s_i)$. Then there exists $N \overset{\Delta}{=} \max\{N_{s_1},...,N_{s_l} \}$ such that $\forall n \geq N$, we have that 
				\begin{align}
				\big| p(y\in B|s) - p(y_n \in B|s) \big| \leq 3\epsilon, \ \forall s \in M_s.
				\end{align}
				Therefore, $I_{n,1} \leq \int_{M_s \times M_z} 3\epsilon p(x\in A|s,z)p(s,z|c) dsdz \leq 3\epsilon$. Hence, $I_{n,1} \to 0$ as $n \to \infty$. Besides, we have that $I_{n,2} \leq \int_{M_s \times M_z} 2\epsilon p(s,z|c_r) dsdz \leq 2\epsilon$. Therefore, we have that $\big| \int_{\mathcal{S} \times \mathcal{Z}} p(x \in A|s,z)\left( p(y\in B|s) - p(y_n \in B|s) \right) p(s,z|c_r) dsdz   \big| \to 0$ as $n \to \infty$. For $I_{n,3}$, we have that 
				\begin{align}
				I_{n,3} & = \bigg| \int_{M_s \times M_z} p(x\in A|s,z)p(y_n \in B|s)\mathbbm{1}(s,z\in  M_s \times M_z)\left( p(s,z|c_r) - p_n(s,z|c_r) \right)dsdz       \bigg| \nonumber \\
				& \leq \underbrace{\bigg| \int_{M_s \times M_z} p(x \in A|s,z)p(y_n \in B|s)p(s,z|c_r) \left( \frac{1}{p(s,z\in  M_s \times M_z|c_r)} - 1 \right) dsdz \bigg|}_{I_{n,3,1}} \nonumber \\
				& + \underbrace{\bigg| \int_{M_s \times M_z} p(x \in A|s,z)p(y_n \in B|s)p(s,z|c_r) \left( \frac{1}{p(s,z\in  M_s \times M_z|c_r)} - 1 \right) dsdz \bigg|}_{I_{n,3,2}}.  \nonumber \\ 
				\end{align}
				The $I_{n,3,1} \leq \frac{\epsilon}{1-\epsilon}$. Denote $\tilde{p}(s,z|c_r) := \frac{p(s,z|c_r)\mathbbm{1}(s,z\in  M_s \times M_z)}{p(s,z\in  M_s \times M_z|c_r)}$, according to \citep[Theorem 2]{barron1991approximation}, there exists a sequence of $p_n(s,z|c)$ defined on a compact support $M_s \times M_z$ such that $\forall c_r \in \mathcal{C}$, we have that 
				\begin{align}
				p_n(s,z|c_r) \overset{d}{\to} p(s,z|c_r). \nonumber 
				\end{align}
				Applying again the Heine–Borel theorem, we have that $\forall \epsilon$, $\exists N$ such that $\forall n \geq N$, we have 
				\begin{align}
				\Big| \tilde{p}(s,z|c_r) - p_n(s,z|c_r) \Big| \leq \epsilon, 
				\end{align}
				which implies that $I_{n,3,2} \to 0$ as $n \to \infty$ combining with the fact that $p(x,y|s,z)$ is continuous with respect to $s,z$. For $I_{n,4}$, we have that 
				\begin{align}
				I_{n,4} = \bigg| \int_{M_s \times M_z} p(x\in A|s,z)p(y_n \in B|s) p(s,z|c_r) \bigg| \leq  \bigg| \int_{M_s \times M_z} p(s,z|c_r) \bigg| \leq \epsilon,
				\end{align}
				where the first equality is from that the $p_n(s,z|c_r)$ is defined on $M_s \times M_z$. Then we have that 
				\begin{align}
				\bigg| \int_{\mathcal{S} \times \mathcal{Z}} p(x \in A|s,z)p(y_n \in B|s)\left( p(s,z|c_r) - p_n(s,z|c_r) \right)   \bigg| \to 0, \ as \ n \to \infty. 
				\end{align}
				The proof is completed. 
			\end{proof}

			\subsection{Reparameterization for LaCIM}
			\label{sec:repara}
			We provide an alternative training method to avoid parameterization of prior $p(s,z|I^e)$ to increase the diversity of generative models in different environments. Specifically, motivated by \cite{hyvarinen1999nonlinear} that any distribution can be transformed to isotropic Gaussian with the density denoted by $p_{\mathrm{Gau}}$, we have that for any $e \in \mathcal{E}_{\mathrm{train}}$, we have 
			\begin{align}
			p^e(x,y) & = \int_{\mathcal{S} \times \mathcal{Z}} p_{f_x}(x|s,z)p_{f_y}(y|s)p(s,z|I^e)dsdz \nonumber \\
			& = \int_{\mathcal{S} \times \mathcal{Z}} p(x|(\varphi^{e}_s)^{-1}(s'),(\varphi^{e}_z)^{-1}(z'))p(y|\varphi_s(s'))p(s',z') ds'dz', \nonumber 
			\end{align}
			with $s',z' := \varphi^e_s(s), \varphi^e_z(z) \sim \mathcal{N}(0,I)$. We can then rewrite ELBO for LaCIM for environment $e$ as:
			\begin{align}
			\label{eq:lacim-unpool}
			\mathcal{L}^e_{\theta,\psi,\varphi^e} &  = \mathbb{E}_{p^e(x,y)}\left[ -\log{q_{\psi}^e(y|x)} \right] \nonumber \\
			& \ \ \ \ \ \ \ \ \ \ + \mathbb{E}_{p^e(x,y)}\left[ - \mathbb{E}_{q_{\psi}^e(s,z|x)}\frac{q_{\psi}(y|(\varphi^{e}_s)^{-1}(s))}{q_{\psi}^e(y|x)}\log{ \frac{p_{\theta}( (\varphi^{e}_s)^{-1}(s),(\varphi^{e}_z)^{-1}(z))p(s,z)}{q_{\psi}^e(s,z|x)} }   \right],
			\end{align}
			where $p(s,z)$ denotes the density function of isotropic gaussian.


			\subsection{Identifiability} 
			\label{sec:related-iden}
			
			{Earlier works that identify the latent confounders rely on strong assumptions regarding the causal structure, such as the linear model from latent to observed variable or ICA in which the latent component are independent \citet{silva2006learning}, or noise-free model \citet{shimizu2009estimation,davies2004identifiability}. The \citet{hoyer2008nonlinear,janzing2012identifying} extend to the additive noise model (ANM) and other causal discovery assumptions. Although the \citet{lee2019leveraging} relaxed the constraints put on the causal structure, it required the latent noise is with small strength, which does not match with many realistic scenarios, such as the structural MRI of Alzheimer's Disease considered in our experiment. The works which also based on the independent component analysis (ICA), \emph{i.e.}, the latent variables are (conditionally) independent, include \citet{davies2004identifiability,eriksson2003identifiability}; recently, a series of works extend the above results to deep nonlinear ICA \citep{hyvarinen2016unsupervised,hyvarinen2019nonlinear,khemakhem2020variational,khemakhem2020ice,teshima2020few}. However, these works require that the value of confounder of these latent variables is fixed, which cannot explain the spurious correlation in a single dataset. In contrast, our result incorporate these scenarios by assuming that each sample has a specific value of the confounder. }

			\subsection{Comparison with existing works}
			\label{appx:related}

			\subsubsection{Comparisons with data augmentation \& architecture design}
			
			The goal of data augmentation \cite{shorten2019survey} is increase the variety of the data distribution, such as geometrical transformation \cite{kang2017patchshuffle, taylor2017improving}, flipping, style transfer \cite{gatys2015neural}, adversarial robustness \cite{madry2017towards}. On the other way round, an alternative kind of approaches is to integrate into the model corresponding modules that improve the robustness to some types of variations, such as \cite{worrall2017harmonic, marcos2016learning}. 
			
			However, these techniques can only make effect because they are included in the training data for neural network to memorize \cite{zhang2016understanding}; besides, the improvement is only limited to some specific types of variation considered. As analyzed in \cite{xie2020risk,krueger2020out}, the data augmentation trained with empirical risk minimization or robust optimization \cite{ben2009robust} such as adversarial training \cite{madry2017towards,sagawa2019distributionally} can only achieve robustness on interpolation (convex hull) rather than extrapolation of training environments.

			\subsubsection{Comparisons with existing works in domain adaptation}
			
			Apparently, the main difference lies in the problem setting that (i) the domain adaptation (DA) can access the input data of the target domain while ours cannot; and (ii) our methods need multiple training data while the DA only needs one source domain. For methodology, our LaCIM shares insights but different with DA. Specifically, both methods assume some types of invariance that relates the training domains to the target domain. For DA, one stream is to assume the same conditional distribution shared between the source and the target domain, such as covariate shift \cite{huang2007correcting, ben2007analysis, johansson2019support, sugiyama2008direct} in which $P(Y|X)$ are assumed to be the same across domains, concept shift \cite{zhang2013domain} in which the $P(X|Y)$ is assumed to be invariant. Such an invariance is related to representation, such as $\Phi(X)$ in \cite{zhao2019learning} and  $P(Y|\Phi(X))$ in \cite{pan2010domain,ganin2016domain,magliacane2018domain}.

			However, these assumptions are only distribution-level rather than the underlying causation which takes the data-generating process into account. Taking the image classification again as an example, our method first propose a causal graph in which the latent factors are introduced as the explanatory/causal factors of the observed variables. These are supported by the framework of generative model \cite{khemakhem2020variational,khemakhem2020ice,kingma2014auto,suter2019robustly} which has natural connection with the causal graph \cite{scholkopf2019causality} that the edge in the causal graph reflects both the causal effect and also the generating process. Until now, perhaps the most similar work to us are \cite{romeijn2018intervention} and \cite{teshima2020few} which also need multiple training domains and get access to a few samples in the target domain. Both work assumes the similar causal graph with us but unlike our LaCIM, they do not separate the latent factors which can not explain the spurious correlation learned by supervised learning \cite{ilse2020designing}. Besides, the multiple training datasets in  \cite{romeijn2018intervention} refer to intervened data which may hard to obtain in some applications. We have verified in our experiments that explicitly disentangle the latent variables into two parts can result in better OOD prediction power than mixing them together.

			\subsubsection{Comparisons with domain generalization}
			\label{sec:dg}
			For domain generalization (DG),  similar to the invariance assumption in DA, a series of work proposed to align the representation $\Phi(X)$ that assumed to be invariant across domains \cite{li2017learning, li2018domain, muandet2013domain}.  As discussed above, these methods lack the deep delving of the underlying causal structure and precludes the variations of unseen domains. 
			
			Recently, a series of works leverage causal invariance to enable OOD generalization on unseen domains, such as \cite{ilse2019diva} which learns the representation that is domain-invariant. Notably, the Invariant Causal Prediction \cite{peters2016causal} formulates the assumption in the definition of Structural Causal Model and assumes that $Y = X_{\mathcal{S}} \beta_{\mathcal{S}}^\star + \varepsilon_Y$ where $\varepsilon_Y$ satisfies Gaussian distribution and $\mathcal{S}$ denotes the subset of covariates of $X$. The \cite{rojas2018invariant,buhlmann2018invariance} relaxes such an assumption by assuming the invariance of $f_y$ and noise distribution $\varepsilon_y$ in $Y \gets f_y(X_{\mathcal{S}},\varepsilon_y)$ which induces $P(Y|X_{\mathcal{S}})$. The similar assumption is also adopted in \cite{kuang2018stable}. However, these works causally related the output to the observed input, which may not hold in many real applications in which the observed data is sensory-level, such as audio waves and pixels. It has been discussed in \citet{bengio2013representation, bengio2017consciousness} that the causal factors should be high-level abstractions/concepts. The \cite{heinze2017conditional} considers the style transfer setting in which each image is linear combination of shape-related variable and contextual-related variable, which respectively correspond to $S$ and $Z$ in our LaCIM in which the nonlinear mechanism (rather than linear combination in \cite{heinze2017conditional}) is allowed. Besides, during testing, our method can generalize to the OOD sample with intervention such as adversarial noise and contextual intervention. 
			
			Recently, the most notable work is Invariant Risk Minimization \cite{arjovsky2019invariant}, which will be discussed in detail in the subsequent section.

			\subsubsection{Comparisons with Invariant Risk Minimization \cite{arjovsky2019invariant} and references therein}
			
			The Invariant Risk Minimization (IRM) \cite{arjovsky2019invariant} assumes the existence of invariant representation $\Phi(X)$ that induces the optimal classifier for all domains, \textit{i.e.}, the $\mathbb{E}[Y|Pa(Y)]$ is domain-independent in the formulation of SCM. Similar to our LaCIM, the $Pa(Y)$ can refer to latent variables. Besides, to identify the invariance and the optimal classifier, the training environments also need to be diverse enough. As aforementioned, this assumption is almost necessary to differentiate the invariance mechanism from the variant ones. 
			
			The difference of our LaCIM with IRM lies in two aspects: the definition of $Y$ and the methodology. For the label, the IRM defines it as the one obtained after the image (\emph{e.g.}, one label the ``dog" based on the image he/she observes); while the label $Y$ is generated concurrently with $X$, that is, the $Y$ is dependent on the semantic features he/she observed. Consider the following scenario as an illustration: the photographer takes a image $X$ and record the label $Y$ at the same time. Besides, in terms of methodology, the theoretical claim of IRM only holds in linear case; in contrast, the CI $f_x, f_y$ are allowed to be nonlinear. 
			
			Some other works share the similar spirit with or based on IRM. The Risk-Extrapolation (REx) \cite{krueger2020out} proposed to enforce the similar behavior of $m$ classifiers with variance of which proposed as the regularization function. The work in \cite{xie2020risk} proposed a Quasi-distribution framework that can incorporate empirical risk minimization, robust optimization and REx. It can be concluded that the robust optimization only generalizes the convex hull of training environments (defined as interpolation) and the REx can generalize extrapolated combinations of training environments. This work lacks model of underlying causal structure, although it performs similarly to IRM experimentally. Besides, the \cite{teney2020unshuffling} proposed to unpool the training data into several domains with different environment and leverages \cite{arjovsky2019invariant} to learn invariant information for classifier. Recently, the \cite{bellot2020generalization} also assumes the invariance to be generating mechanisms and can generalize the capability of IRM when unobserved confounder exist. However, this work also lacks the analysis of identifiability result.

			\subsection{Implementation Details and More Results for Simulation}
			\label{appx-exp-sim}
			
			\textbf{Data Generation} We set $m=5$. We set $q_{d}=q_s=q_z=q_y=2$ and $q_x=4$. For each environment $e \in [m]$ with $m=5$, we generate 1000 samples $\mathcal{D}^e = \{x_i,y_i\} \overset{i.i.d}{\sim} \int p_{f_x}(x|s,z)p_{f_y}(y|s)p(s,z|c)p(c|d^e)dsdzdc$. The $d^e = \left(\mathcal{N}(0,I_{q_{d} \times q_{d}}) + 5 *e\right)*2$; the $c|d^e \sim \mathcal{N}(d^e,I)$; the $s,z|c \sim  \mathcal{N}\left(\mu_{\theta^{\star}_{s,z}}(s,z|c),\sigma^2_{\theta^{\star}_{s,z}}(s,z|c) \right)$ with $\mu_{\theta^{\star}_{s,z}} = A^\mu_{s,z} * c$ and $\log{\sigma_{\theta^{\star}_{s,z}}} = A^\sigma_{s,z} * c$ ($A^\mu_{s,z}$, $A^\sigma_{s,z}$ are random matrices); the $x|s,z \sim \mathcal{N}\left(\mu_{\theta^{\star}_{x}}(x|s,z),\sigma^2_{\theta^{\star}_{x}}(x|s,z) \right)$ with $\mu_{\theta^{\star}_{s,z}} = h(A^{\mu,3}_{x} * h(A^{\mu,2}_{x} * h(A^{\mu,2}_{x}*[s^\top,z^\top]^\top])))$ and $\log{\sigma_{\theta^{\star}_{s,z}}} = h(A^{\sigma,3}_{x} * h(A^{\sigma,2}_{x} * h(A^{\sigma,2}_{x}*[s^\top,z^\top]^\top])))$ ($h$ is LeakyReLU activation function with slope $=0.5$ and $A^{\mu,i=1,2,3}_{x}$,$A^{\sigma,i=1,2,3}_{x}$ are random matrices); the $y|s$ is similarly to $x|s,z$ with $A^{\mu,i=1,2,3}_{x}$,$A^{\sigma,i=1,2,3}_{x}$ respectively replaced by  $A^{\mu,i=1,2,3}_{y}$,$A^{\sigma,i=1,2,3}_{y}$.

			\textbf{Implementation Details} We parameterize $p_{\theta}(s,z|I^e)$, $q_{\psi}(s,z|x,y,I^e)$, $p_{\theta}(x|s,z)$ and $p_{\theta}(y|s)$ as 3-layer MLP with the LeakyReLU activation function. The Adam with learning rate $5\times 10^{-4}$ is implemented for optimization. We set the batch size as 512 and run for 2,000 iterations in each trial. 
			
			\textbf{Visualization. } As shown from the visualization of $S$ is shown in Fig.~\ref{fig:sim-visualize}, our LaCIM can identify the causal factor $S$. 
			\begin{figure}[ht]
				\vspace{-0.2cm}
				\centering
				\begin{tabular}{ccccccc}
					\hspace{-0.15in} \includegraphics[width=0.25\textwidth]{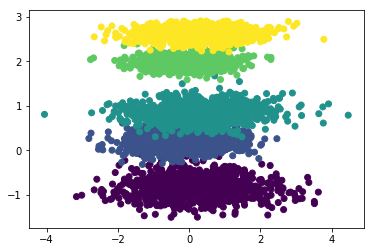} &  	
					\hspace{-0.15in} \includegraphics[width=0.25\textwidth]{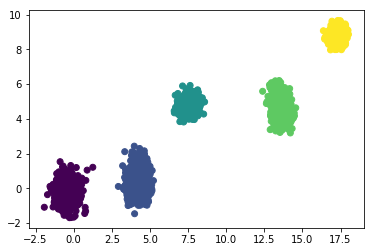} & 	\hspace{-0.15in} \includegraphics[width=0.25\textwidth]{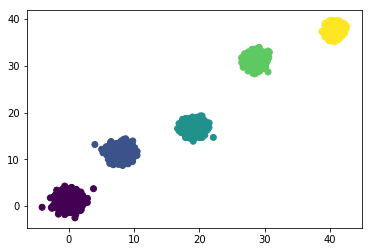} \\
					(a)  pool-LaCIM & (b) LaCIM & (c) $p_{\theta^\star}(s|D)$ 
				\end{tabular}
				\label{fig:sim-visualize}
				\caption{Estimated posterior by (a) pool-LaCIM; (b) LaCIM and (c) the ground-truth. As shown, the LaCIM can identify the $S$ (up to permutation and point-wise transformation), which validates the \eqref{eq:z} in theorem~\ref{thm:iden}.}
			\end{figure}


			\subsection{Implementation Details for Optimization over $S,Z$} 
			\label{appx:optimize-sz}
			Recall that we first optimize $s^*,z^*$ according to 
			\begin{equation*}
			s^*,z^* = \arg\max_{s,z} \log p_{\theta}(x|s,z). 
			\end{equation*}
			We first sample some initial points from each posterior distribution $q^e_{\psi}(s|x)$ and then optimize for 50 iterations. We using Adam as optimizer, with learning rate as 0.002 and weight decay 0.0002. The Fig.~\ref{fig:optimize-zs} shows the optimization effect of one run in CMNIST. As shown, the test accuracy keeps growing as iterates. For time saving, we chose to optimize for 50 iterations. 
			
			\begin{figure}[ht]
				\centering
				\includegraphics[width=0.3\textwidth]{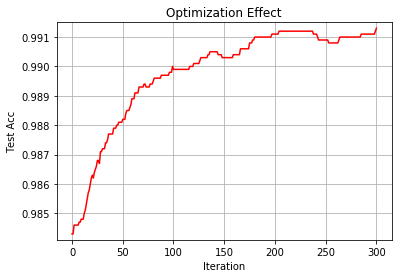}
				\label{fig:optimize-zs}
				\caption{The optimization effect in CMNIST, starting from the point with initial sampling from inference model $q$ of each branch. As shown, the test accuracy increases as iterates. }
			\end{figure}

			\subsection{Implementations For Baseline}
			
			The networks of CE $X \to Y$ contains two parts: (i) feature extractor, followed by (ii) classifier. The network structure of the feature extractor and classifier for CE $X \to Y$ is the same with that of our encoder and our $p_\theta(y|s)$. We adopt the same structure for IRM as CE $X \to Y$. DANN adopts the same structure of CE $X \to Y$ and a additional domain classifier which is the same as that of $p_{\theta}(y|s)$. sVAE adopt the same structure as LaCIM-$d$ with the exception that the $p_\theta(y|s)$ is replaced by $p_\theta(y|z,s)$. MMD-AAE adopt the same structure of encoder, decoder and classifier as LaCIM and a additional 2-layer MLP with channel 256-256-$dim_z$ is used to extract latent $z$.} The detailed number of parameters and channel size on each dataset for each method are summarized in Tab.~\ref{tab:framework},~\ref{tab:module}.


		\subsection{Supplementary for Colored MNIST}
		\label{appx-exp-cmnist}
		
		\textbf{Implementation details} The network structure for inference model is composed of two parts, with the first part shared among all environments and multiple branches corresponding to each environment for the second part. The network structure of the first-part encoder is composed of four blocks, each block is the sequential of Convolutional Layer (Conv), Batch Normalization (BN), ReLU and max-pooling with stride 2. The output number of feature map is accordingly 32, 64, 128, 256. The second part network structure that output the mean and log-variance of $S,Z$ is Conv-bn-ReLU(256) $\to$ Adaptive (1) $\to$ FC(256, 256) $\to$ ReLU $\to$ FC(256, $q_{t=s,z}$) with FC stands for fully-connected layer. The structure of $\varphi_{t=s,z}$ in \eqref{eq:lacim-unpool} is FC($q_t$, 256) $\to$ ReLU $\to$ FC(256, $q_t$). The network structure for generative model $p_{\theta}(x|s,z)$ is the sequential of three modules: (i) Upsampling with stride 2; (ii) four blocks of Transpose-Convolution (TConv), BN and ReLU with respective output dimension being 128, 64, 32, 16; (iii) Conv-BN-ReLU-Sigmoid with number of channels in the output as 3, followed by cropping step in order to make the image with the same size as input dimension, \textit{i.e.}, $3 \times 28 \times 28$. The network structure for generative model $p_{\theta}(y|s)$ is commposed of FC (512) $\to$ BN $\to$ ReLU $\to$ FC (256) $\to$ BN $\to$ ReLU $\to$ FC ($|\mathcal{Y}|$). The $q_{t=s,z}$ is set to 32. We implement SGD as optimizer with learning rate 0.5, weight decay $1e-5$ and we set batch size as 256. The total training epoch is 80.

		We first explain why we do not flip $y$ with $25\%$ in the manuscript, and then provide further exploration of our method for the setting with flipping $y$. 
		
		\textbf{Invariant \textit{Causation} v.s. Invariant \textit{Correlation} by Flipping $y$ in \cite{arjovsky2019invariant}} The $y$ is further flipped with $25\%$ to obtain the final label in IRM setting and this step is omitted in ours. The difference lies in the definition for the label $Y$ and the invariance. Our LaCIM defines invariance as the causal relation between $S$ and the label $Y$, while the one in IRM can be correlation since randomly flipping $Y$ can break the relations between $S$ and $Y$. As illustrated in Handwritting Sample Form in Fig.~\ref{fig:nist} in \cite{grother1995nist}, the generting direction should be $Y \to X$. If we denote $\tilde{Y}$ as the flipped $Y$ (\textit{a.k.a}, the final label in IRM), then the causal graph should be $X \gets Y \to \tilde{Y}$. In this case, the $\tilde{Y}$ is correlated rather than causally related to the digit $X$. For our LaCIM, we define the label as interpretable human label, which can approximate to the ground-truth label $y$ for any image $x$ since it can capture the causal relation between digits and the label. 
		
		
		\begin{figure}[ht]
			\centering
			\includegraphics[width=0.85\textwidth]{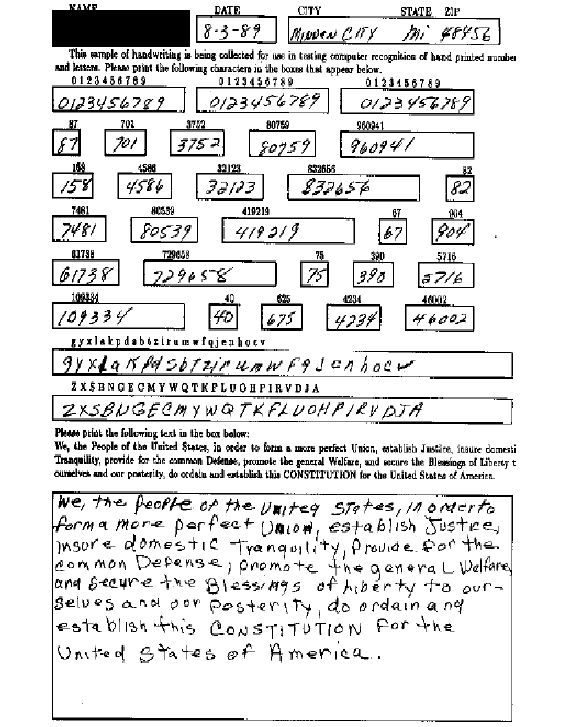}
			\label{fig:nist}
			\caption{Hand-writting Sample Form. The writer print the digit/character (\textit{i.e.}, $X$) with the label (\textit{i.e.}, $Y$) provided first.}
		\end{figure}
		
		\textbf{Experiment with IRM setting} We further conduct the experiment on IRM setting, with the final label $y$ defined by flipping original label with $25\%$, and further color $p^{e}$ proportions of digits with corresponding color-label mapping. If we assume the original ground-truth label to be the effect of the digit number of $S$, then the anti-causal relation with $Z$ and $Y$ can make the identifiability of $S$ difficult in this flipping scenario. Note that the causal effect between $S$ and $Y$ is invariant across domains, therefore we adopt to regularize the branch of inferring $S$ to be shared among inference models for multiple environments.  Besides, we regularize the causal effect between $S$ and $Z$ to be shared among different environments via pairwise regularization. The combined loss is formulated as:
		\begin{align}
		\tilde{\mathcal{L}}_{\psi,\theta} = \mathcal{L}_{\psi,\theta} + \frac{\gamma}{2m^2} \sum_{i=1}^m\sum_{j=1}^m \Vert \mathbb{E}_{(x,y) \sim p^{e_i}(x,y)}[y|x] - \mathbb{E}_{(x,y) \sim p^{e_j}(x,y)}[y|x] \Vert_2^2, \nonumber 
		\end{align}
		where $\gamma > 0$ denotes the regularization hyperparameter. The $q^e_{\psi}(s,z|x)$ in \eqref{eq:lacim-unpool} factorized as $q_{\psi^e_z}(z)q_{\psi_s}(s)$ and $\varphi_s$ shared among $m$ environments. The appended loss is coincide with recent study Risk-Extropolation (REx) in \cite{krueger2020out}, with the difference of separating $Y$-causative factor $S$ from others. We name such a training method as LaCIM-REx. For implementation details, in addition to shared encoder regarding $S$, we set learning rate as 0.1, weight decay as 0.0002, batch size as 256.  we have that $p(y|x) = \int_{\mathcal{S}} q_{\psi_s}(s|x)p_{\theta}(y|\varphi_s(s))$ for any $x$. We consider two settings: setting\#1 with $m2$ and $p^{e_1} = 0.9, p^{e_2} = 0.8$; and setting\#2 with $m=4$ with $p^{e_1} = 0.9, p^{e_2} = 0.8, p^{e_3} = 0.7, p^{e_4} = 0.6$. We only report the number of IRM since the cross entropy performs poorly in both settings. As shown, our model performs comparably than IRM \cite{arjovsky2019invariant} due to separation of $S$ znd $Z$.

		\begin{table}[h!]
			\centering
			\begin{tabular}{c|c|c|c|c|c}
				\hline
				&  IRM &  LaCIM-REx (\textbf{Ours}) \\ 
				\hline 
				$m=2$ & $67.15 \pm 3.79$  & $\mathbf{67.57 \pm 1.37}$  \\
				\hline 
				$m=4$ & $69.37 \pm 1.14$  & $\mathbf{69.50 \pm 0.57}$  \\
				\hline 
			\end{tabular}
			\label{table:acc-cmnist} 
			\caption{Accuracy (\%) of Colored MNIST on IRM setting in \cite{arjovsky2019invariant}. Average over three runs.} 
		\end{table}

		\subsection{Supplementary for NICO}
		\label{appx-exp-nico}
		
		\textbf{Implementation Details}  Due to size difference among images, we resize each image into 256$\times$256. The network structure of $p_{\theta}(z,s|I^e), q_{\psi}(z,s|x,I^e), p_{\theta}(x|z,s), p_{\theta}(y|s)$ for cat/dog classification is the same with the one implemented in early prediction of Alzheimer's Disease with exception of 3D convolution/Deconvolution replaced by 2D ones. For each model, we train for 200 epochs using sgd, with learning rate (lr) set to 0.01, and after every 60 epochs the learning rate is multiplied by lr decay parameter that is set to 0.2. The weight decay coefficients parameter is set to $5 \times 10^{-4}$. The batch size is set to 30. The training environments which is characterized by $c$ can be referenced in Table~\ref{table:cat-dog-environment}. For visualization, we implemented the gradient-based method \cite{simonyan2013deep} to visualize the neuron (in fully connected layer of CE $x\to y$ and the $s$ layer of LaCIM that is most correlated to label $y$. 
		
		\textbf{The $D$ for $m$ environments} We summarize the $D$ of $m=8$ and $m=14$ environments in Table~\ref{table:cat-dog-environment}. Since the distribution of $S,Z$ depends on $D$, we simply define $D$ as the parameterization of $S,Z$. In this context, such a parameterization refers to the proportions of (dog in grass, dog in snow; cat in grass, cat in snow); therefore $D \in \mathbb{R}^4$. As shown, the value of $D$ in the test domain is the extrapolation of the training environments, \textit{i.e.}, the $d^{\mathrm{test}}$ is not included in the convex hull of $\{d^{e_i}\}_{i=1}^{14}$. 
		\begin{table}[t!]
			\centering
			\footnotesize
			\begin{tabular}{l|cccc}
				\toprule
				& cat$\%$ on grass & dog$\%$ on grass & cat$\%$ on snow & cat$\%$ on snow \\ 
				\midrule
				&  \multicolumn{4}{c}{Training Environment} \\
				\midrule
				Env\#1 ($d^{e_1}$)  & 0.6  & 0.4 & 0.1 & 0.9  \\
				Env\#2 ($d^{e_2}$)  &  0.8 & 0.2 & 0.1 & 0.9  \\
				Env\#3 ($d^{e_3}$) & 0.5 & 0.5 & 0.2 &  0.8  \\
				Env\#4 ($d^{e_4}$)  & 0.8  & 0.2 & 0.2 & 0.8  \\
				Env\#5 ($d^{e_5}$)  & 0.7  & 0.3 & 0.2 & 0.8  \\
				Env\#6 ($d^{e_6}$)  & 0.8  & 0.2 & 0.3 & 0.7  \\
				Env\#7 ($d^{e_7}$)  & 0.7  & 0.3 & 0.3 & 0.7  \\
				Env\#8 ($d^{e_8}$)  & 0.9  & 0.1 & 0.3 & 0.7  \\
				\hline 
				Env\#9 ($d^{e_9}$) & 0.4  & 0.6 & 0.3 & 0.7  \\
				Env\#10 ($d^{e_{10}}$) & 0.6  & 0.4 & 0.3 & 0.7  \\
				Env\#11 ($d^{e_{11}}$)  & 0.5  & 0.5 & 0.4 & 0.6  \\
				Env\#12 ($d^{e_{12}}$)  & 0.4  & 0.6 & 0.4 & 0.6  \\
				Env\#13 ($d^{e_{13}}$)  & 0.7  & 0.3 & 0.4 & 0.6  \\
				Env\#14 ($d^{e_{14}}$)  & 0.8  & 0.2 & 0.4 & 0.6  \\
				\midrule
				&  \multicolumn{4}{c}{\textcolor{blue}{Testing Environment}} \\
				\midrule
				\textcolor{blue}{Env Test $d^{\mathrm{test}}$}  & \textcolor{blue}{0.2}  & \textcolor{blue}{0.8} & \textcolor{blue}{0.8} & \textcolor{blue}{0.2} \\
				\bottomrule
			\end{tabular}
			\label{table:cat-dog-environment} 
			\caption{Training and test environments (characterized by $D$)} 
		\end{table}

		\textbf{More Visualization Results} Fig.~\ref{fig:cat-dog-visualize-appx} shows more visualization results.
		\begin{figure}[h!]
			\centering
			\begin{tabular}{ccccccc}
				\centering
				\includegraphics[width=0.45\textwidth]{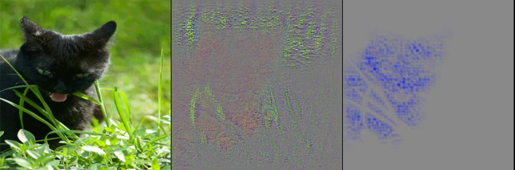}
				& \includegraphics[width=0.45\textwidth]{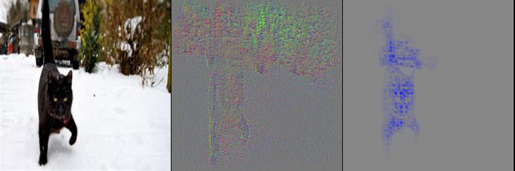} 
				\\
				\includegraphics[width=0.45\textwidth]{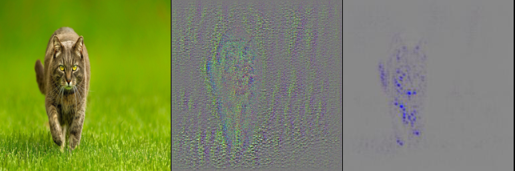}
				& \includegraphics[width=0.45\textwidth]{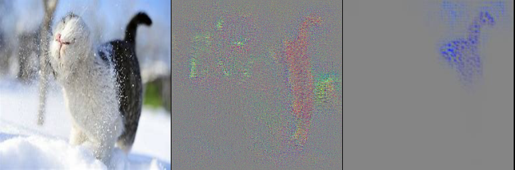}
				\\
				\includegraphics[width=0.45\textwidth]{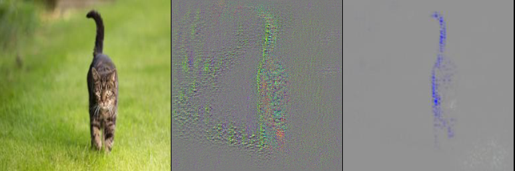}
				& \includegraphics[width=0.45\textwidth]{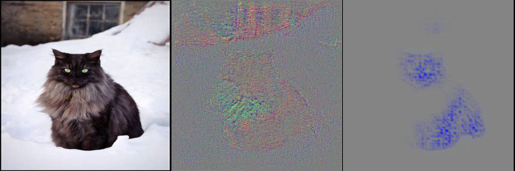}
				\\
				\includegraphics[width=0.45\textwidth]{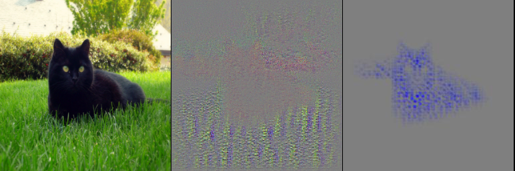}
				& \includegraphics[width=0.45\textwidth]{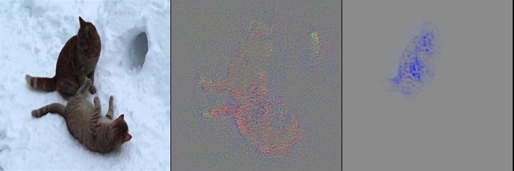}
				\\
				(a) Cat on grass & (b) Cat on snow \\
				\includegraphics[width=0.45\textwidth]{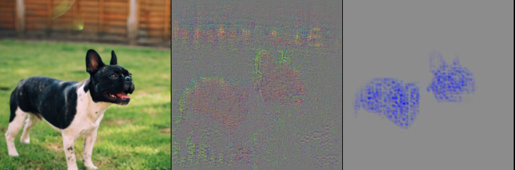}
				& \includegraphics[width=0.45\textwidth]{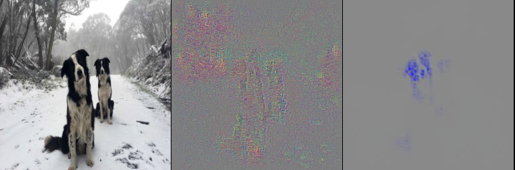}
				\\
				\includegraphics[width=0.45\textwidth]{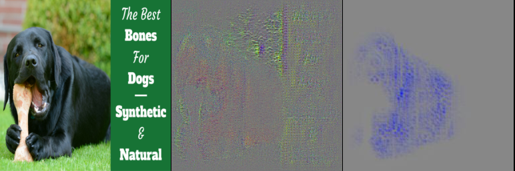}
				& \includegraphics[width=0.45\textwidth]{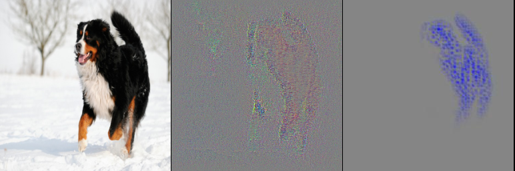}
				\\
				\includegraphics[width=0.45\textwidth]{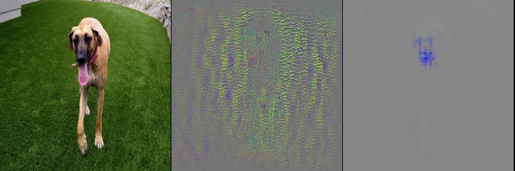}
				& \includegraphics[width=0.45\textwidth]{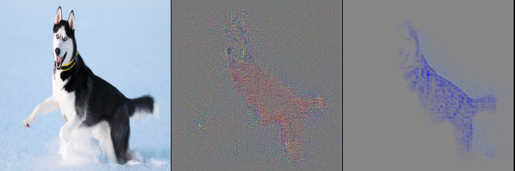}
				\\
				\includegraphics[width=0.45\textwidth]{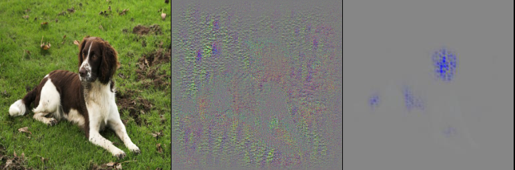}
				& \includegraphics[width=0.45\textwidth]{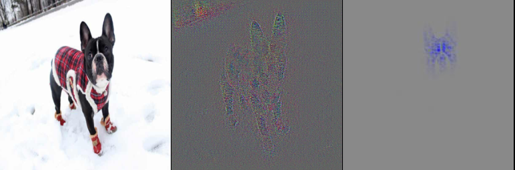}
				\\
				(c) Dog on grass & (d) Dog on snow
			\end{tabular}
			\caption{Visualization on the NICO via gradient-based method \cite{simonyan2013deep} for CE $X \to Y$ and LaCIM. The selected images are (a) cat on grass, (b) cat on snow, (c) dog on grass and (d) dog on snow.}
			\label{fig:cat-dog-visualize-appx}
		\end{figure}

		\begin{figure*}[]
			\centering
			\includegraphics[width=0.8\linewidth]{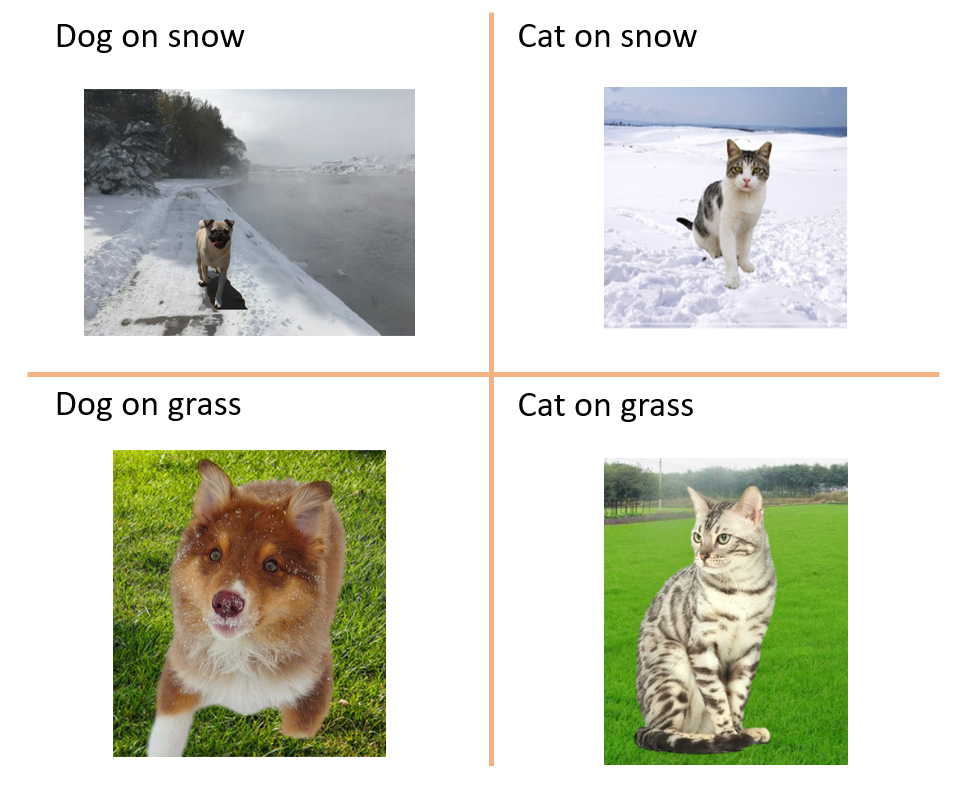}
			\caption{The constructed interventional dataset which includes of dog on snow, dog on grass, cat on snow, and dog on grass.  
			}
			\label{fig:interven}
		\end{figure*}
		
		\textbf{Generation of Intervened Data.} For generating an intervened sample, we replace the scene of an image with the scene from the another image, as shown in Fig.~\ref{fig:interven}. This process can be viewed as breaking the dependency between $Z$ and $Z$. We generate 120 images, including 30 images of types: cat on grass, dog on grass, cat on snow, and dog on grass.

		\subsection{Disease Prediction of Alzheimer's Disease}
		\label{appx:expm-ad}
		
		\textbf{Dataset Description.} The dataset contains in total 317 samples with 48 AD, 75 NC, and 194 MCI.
		
		\textbf{Denotation of Attributes $D$.} The $D \in \mathbb{R}^9$ includes personal attributes (\textit{e.g.}, age \cite{guerreiro2015age}, gender \cite{vina2010women} and education years \cite{mortimer1997brain} that play as potential risks of AD), gene ($\varepsilon_4$ allele), and biomarkers (\textit{e.g.}, changes of CSF, TAU, PTAU, amyloid$_{\beta}$, cortical amyloid deposition (AV45) \cite{humpel2011cerebrospinal}). 
		
		\textbf{Implementation Details} The $S,Z \in \mathbb{R}^{64}$. For the shared part of $q_{\psi}(s,z|x,I^e)$, we concatenate outputs of feature extractors of $X$ and $I^e$: the feature extractor for $x$ is composed of four Convolution-Batch Normalization-ReLU (CBNR) blocks and four Convolution-Batch Normalization-ReLU-MaxPooling (CBNR-MP) blocks with structure 64 BNR $\to$ 128 CBNR-MP $\to$ 128 CBNR $\to$ 256 CBNR-MP $\to$ 256 CBNR $\to$ 512 CBNR-MP $\to$ 512 CBNR $\to$ 1024 CBNR-MP; the feature extractor of $I^e$ is composed of three Fully Connection-Batch Normalization-ReLU (FC-BNR) blocks with structure 128 $\to$ 256 $\to$ 512. for the part specific to each domain, $\mu_{s,z}(x,d)$ and $\log{\sigma_{s,z}(x,d)}$ are generated by the sub-network which is composed of 1024 FC-BNR $\to$ 1024 FC-BNR $\to$ $q_{z,s}$ FC-BNR. The $z,s$ can be reparameterized by $\mu_{s,z}(x,d)$ and $\log{\sigma_{s,z}(x,d)}$ are fed into a sub-network which is composed of $q_{z,s}$  FC-BNR $\to$ 1024 FC-BNR $\to$ $q_{z,s}$ FC-BNR to get rid of the constraint of Gaussian distribution. For the prior model $p_{\theta}(s,z|I^e)$, it shares the same structure without feature extractor of $x$. For $p_{\theta}(x|s,z)$, the network is composed of three DeConvolution-Batch Normalization-ReLU (DCBNR) blocks and three Convolution-Batch Normalization-ReLU (CBNR) blocks, followed by a
		convolutional layer, with structure 256 DCBNR $\to$ 256 CBNR $\to$ 128 DCBNR $\to$ 128 CBNR $\to$ 64 DCBNR $\to$ 64 CBNR $\to$ 48 Conv. For $p_{\theta}(y|s)$, the network is composed of 256 FC-BNR $\to$ 512 FC-BNR $\to$ 3 FC-BNR. For prior model $p_{\theta}(s,z|I^e) \mathcal{N}(\mu_{s,z}(I^e), \mathrm{diag}(\sigma^2_{s,z}(I^e)))$ the $\mu_{s,z}(x,I^e)$ and $\log{\sigma_{s,z}}(x,I^e)$ are parameterized by Multi Perceptron Neural Network (MLP). The decoders $p_{\theta}(x|s,z)$ are $p_{\theta}(y|s)$ parameterized by Deconvolutional neural network. For all methods, we train for 200 epochs using SGD with weight decay $2\times 10^{-4}$ and learning rate 0.01 and is multiplied by 0.2 after every 60 epochs. The batch size is set to 4. 
		
		
		\textbf{The $D$ variable in training and test.} The selected attributes include Education Years, Age, Gender (0 denotes male and 1 denotes female), AV45, amyloid$_{\beta}$ and TAU. We split the data into $m=2$ training environments and test according to different value of $D$. The Tab.~\ref{table:ad-main-c} describes the data distribution in terms of number of samples, the value of $D$ (Age and TAU).

		\begin{table}[t!]
			\centering
			\footnotesize
			\begin{tabular}{c|cccccccc}
				\toprule
				& Training Env\#1  & Training Env\#1 & Test \\ 
				\midrule
				&  \multicolumn{3}{c}{Age} \\
				\midrule
				Number of AD  & 17  & 17 & 14  \\
				Number of MCI  & 76  & 83 & 35   \\
				Number of NC & 34 & 27 & 14\\
				Average value of $d$ (years): & 68.75 & 72.78 & 81.74 \\
				\midrule
				&  \multicolumn{3}{c}{TAU} \\
				\hline
				Number of AD  & 11  & 22 & 15  \\
				Number of MCI  & 75  & 78 & 41  \\
				Number of NC & 40 & 27 & 18 \\
				Average value of $d$: &  215.34 & 286.69 & 471.72 \\
				\bottomrule
			\end{tabular}
			\label{table:ad-main-c} 
			\caption{Training and test environments (characterized by $c$) in early prediction of AD} 
		\end{table}

		\section{Robustness on Security}
		\label{sec:expm-robust}
		
		We consider the DeepFake-related security problem, which targets on detecting small perturbed fake images that can spread fake news. The \citet{rossler2019faceforensics++} provides FaceForensics++ dataset from 1000 Youtube videos for training and 1,000 benchmark images from other sources (OOD) for testing. We split the train data into $m=2$ environments according to video ID. The considerable result in Tab.~\ref{table:acc-deepfake} verifies potential value on security.  
		
		\textbf{Implementation Details.} We implement data augmentations, specifically images with 30 angle rotation, with flipping horizontally with $50\%$ probability. We additionally apply random compressing techniques, such as JpegCompression. For inference model, we adopt Efficient-B5 \cite{Tan2019EfficientNetRM}, with the detailed network structure as: FC(2048, 2048) $\to$ BN $\to$ ReLU $\to$ FC(2048, 2048) $\to$ BN $\to$ ReLU $\to$ FC(2048, $q_{t=s,z}$). The structure of reparameterization, \textit{i.e.}, $\varphi_{t=s,z}$ is FC($q_{t=s,z}$, 2048) $\to$ BN $\to$ ReLU $\to$ FC(2048, 2048) $\to$ BN $\to$ ReLU $\to$ FC(2048, $q_{t=s,z}$). The network structure for generative model, \textit{i.e.}, $p_\psi(x|s,z)$ is TConv-BN-ReLU($q_{t=s,z}$, 256) $\to$ TConv-BN-ReLU(256, 128) $\to$ TConv-BN-ReLU(128, 64)$\to$ TConv-BN-ReLU(64, 32) $\to$ TConv-BN-ReLU(32, 32) $\to$ TConv-BN-ReLU(32, 16) $\to$ TConv-BN-ReLU(16, 16) $\to$ Conv-BN-ReLU(16, 3) $\to$ Sigmoid, followed by cropping the image to the same size $3\times 224 \times 224$. We set $q_{t=s,z}$ as 1024. We implement SGD as optimizer, with learning rate 0.02, weight decay 0.00005, and run for 9 epochs. 
		
		\begin{table}[h!]
			\centering
			\begin{tabular}{c|c|c}
				\hline 
				CE $X\to Y$  & IRM  & LaCIM (\textbf{Ours}) \\
				\hline 
				$82.8 \pm 0.99$  & $ 83.4 \pm 0.59$ & $\mathbf{84.47 \pm 0.90 }$ \\
				\hline
			\end{tabular}
			\label{table:acc-deepfake}
			\vspace{-0.2cm}
			\caption{Accuracy (\%) of robustness on FaceForensics++. Average over three runs.}
		\end{table}

		\section{Network Structure}

		\begin{sidewaystable}[]
			\tiny
			\centering
			\caption{General framework table for our method and baselines on $\mathrm{Data} \in \{\mbox{CMNIST},\mbox{NICO},\mbox{ADNI},\mbox{DeepFake}\}$ Dataset. We denote the dimension of $z$ or $s$ as $\mbox{dim}_{z,s}$. We list the output dimension (\emph{e.g.} the channel number) of each module, if it is different from the one in Tab.~\ref{tab:module}.}
			\label{tab:framework}
			\begin{tabular}{|c|c|c|c|c|c|c|c|}
				\hline
				\diagbox{Dataset}{Method} & CE $X \to Y$ & MMD-AAE & DANN & DIVA  & LaCIM \\
				\hline
				$\mathrm{Data}$:CMNIST &  
				\makecell[c]{
					$\mbox{Enc}_x^{\mathrm{Data}}$\\FC(256,$\mbox{dim}_{z}$)\\
					$\mbox{Dec-CE}_y^{\mathrm{Data}}$}
				& 
				\makecell[c]{$\mbox{Enc}_x^{\mathrm{Data}}$\\
					FC-BN-ReLU(256,256)\\
					FC(256,256) $\to z$ \\
					$\mbox{Dec}_y^{\mathrm{Data}}$; $\mbox{Dec}_x^{\mathrm{Data}}$} & 
				\makecell[c]{$\mbox{Enc}_x^{\mathrm{Data}}$\\
					$\mbox{DANN-CLS}_y^{\mathrm{Data}}$; $\mbox{DANN-CLS}_y^{\mathrm{Data}}$
				}
				&
				\makecell[c]{
					$p^{\mathrm{Data}}_\theta(x|z_d,z_x,z_y)$\\
					$p^{\mathrm{Data}}_{\theta_d}(z_d|d)$ \\ $p^{\mathrm{Data}}_{\theta_y}(z_y|y)$ \\
					$q^{\mathrm{Data}}_{\phi_d}(z_d|x)$\\
					$q^{\mathrm{Data}}_{\phi_x}(z_x|x)$\\
					$q^{\mathrm{Data}}_{\phi_y}(z_y|x)$
				}
				& 
				\makecell[c]{$\mbox{Enc}_x^{\mathrm{Data}}$\\
					$\mbox{Enc}_{z,s}^{\mathrm{Data}}$ $\times$ $m$ \\
					$\Phi_{z,s}^{\mathrm{Data}}$ $\times$ $m$\\
					$\mbox{Dec}_y^{\mathrm{Data}}$;$\mbox{Dec}_x^{\mathrm{Data}}$
				}  \\ 
				\hline
				\# of Params &  1.12M  & 1.23M & 1.1M & 1.69M & 0.92M \\ 
				\hline
				hyper-Params & \makecell[c]{lr: 0.1\\wd:0.00005} & \makecell[c]{lr: 0.01\\wd: 0.0001} & \makecell[c]{lr: 0.1\\wd: 0.0002} & \makecell[c]{lr: 0.001\\wd: 0.00001}  & \makecell[c]{lr: 0.01\\wd: 0.0002} \\ 
				\midrule
				$ \mathrm{Data}:NICO$ &  \makecell[c]{$\mbox{Enc}_x^{\mathrm{Data}}$\\
					FC(1024,$\mbox{dim}_{z}$)\\
					$\mbox{Dec-CE}_y^{\mathrm{Data}}$}
				& 
				\makecell[c]{$\mbox{Enc}_x^{\mathrm{Data}}$\\
					FC-BN-ReLU(1024,1024)\\
					FC(1024,1024) $\to z$ \\
					$\mbox{Dec}_y^{\mathrm{Data}}$; $\mbox{Dec}_x^{\mathrm{Data}}$} & 
				\makecell[c]{$\mbox{Enc}_x^{\mathrm{Data}}$\\
					$\mbox{DANN-CLS}_y^{\mathrm{Data}}$; $\mbox{DANN-CLS}_y^{\mathrm{Data}}$
				}
				&
				\makecell[c]{
					$p^{\mathrm{Data}}_\theta(x|z_d,z_x,z_y)$\\
					$p^{\mathrm{Data}}_{\theta_d}(z_d|d)$ \\ $p^{\mathrm{Data}}_{\theta_y}(z_y|y)$ \\
					$q^{\mathrm{Data}}_{\phi_d}(z_d|x)$\\
					$q^{\mathrm{Data}}_{\phi_x}(z_x|x)$\\
					$q^{\mathrm{Data}}_{\phi_y}(z_y|x)$
				}
				& \makecell[c]{$\mbox{Enc}_x^{\mathrm{Data}}$\\
					$\mbox{Enc}_{z,s}^{\mathrm{Data}}$ $\times$ $m$ \\
					$\Phi_{z,s}^{\mathrm{Data}}$ $\times$ $m$\\
					$\mbox{Dec}_y^{\mathrm{Data}}$;$\mbox{Dec}_x^{\mathrm{Data}}$ }  \\ \hline
				\# of Params ($m=8$) &  18.08M & 19.70M & 19.13M & 14.86M & 18.25M \\ \hline
				\# of Params ($m=14$) &  18.08M & 19.70M & 26.49M & 14.87M & 19.70M \\ \hline
				hyper-Params & \makecell[c]{lr: 0.01\\wd: 0.0002}  & \makecell[c]{lr: 0.2\\wd: 0.0001} & \makecell[c]{lr: 0.05\\wd: 0.0005} & \makecell[c]{lr: 0.001\\wd: 0.0001} &  \makecell[c]{lr: 0.01\\wd: 0.0001} \\
				\midrule
				$\mathrm{Data}$:ADNI &  \makecell[c]{$\mbox{Enc}_x^{\mathrm{Data}}$\\
					FC(1024,$\mbox{dim}_{z}$)\\
					$\mbox{Dec-CE}_y^{\mathrm{Data}}$}
				& 
				\makecell[c]{$\mbox{Enc}_x^{\mathrm{Data}}$\\
					FC-BN-ReLU(1024,1024)\\
					FC(1024,1024) $\to z$ \\
					$\mbox{Dec}_y^{\mathrm{Data}}$; $\mbox{Dec}_x^{\mathrm{Data}}$} & 
				\makecell[c]{$\mbox{Enc}_x^{\mathrm{Data}}$\\
					$\mbox{DANN-CLS}_y^{\mathrm{Data}}$; $\mbox{DANN-CLS}_y^{\mathrm{Data}}$
				}
				&
				\makecell[c]{
					$p^{\mathrm{Data}}_\theta(x|z_d,z_x,z_y)$\\
					$p^{\mathrm{Data}}_{\theta_d}(z_d|d)$ \\ $p^{\mathrm{Data}}_{\theta_y}(z_y|y)$ \\
					$q^{\mathrm{Data}}_{\phi_d}(z_d|x)$\\
					$q^{\mathrm{Data}}_{\phi_x}(z_x|x)$\\
					$q^{\mathrm{Data}}_{\phi_y}(z_y|x)$
				}
				& \makecell[c]{$\mbox{Enc}_x^{\mathrm{Data}}$\\
					$\mbox{Enc}_{z,s}^{\mathrm{Data}}$ $\times$ $m$ \\
					${\Phi}_{z,s}^{\mathrm{Data}}$ $\times$ $m$\\
					$\mbox{Dec}_y^{\mathrm{Data}}$;$\mbox{Dec}_x^{\mathrm{Data}}$ }  \\ \hline
				\# of Params &  28.27M & 36.68M & 30.21M & 33.22M & 37.78M \\ \hline
				hyper-Params & \makecell[c]{lr: 0.01\\wd: 0.0002}  & \makecell[c]{lr: 0.005\\wd: 0.0002} & \makecell[c]{lr: 0.01\\wd: 0.0002} & \makecell[c]{lr: 0.005\\wd: 0.0001} &  \makecell[c]{lr: 0.01\\wd: 0.0002} \\
				\bottomrule
			\end{tabular}
		\end{sidewaystable}

		\begin{table}[]
			\scriptsize
			\centering
			\caption{Network Structure of Modules used in our method and baselines.}
			\label{tab:module}
			\begin{tabular}{|c|c|c|c|}
				\hline
				Method &  CMNIST & NICO & ADNI \\ \hline
				$\mbox{Enc}_x^{\mathrm{Data}}$&  \makecell[c]{Conv-BN-ReLU($\mbox{dim}_{\mbox{input}}$,64,3,1,1)\\
					MaxPool(2)\\
					Conv-BN-ReLU(64,128,3,1,1)\\
					MaxPool(2)\\
					Conv-BN-ReLU(128,256,3,1,1)\\
					MaxPool(2)\\
					Conv-BN-ReLU(256,256,3,1,1)\\
					AdaptivePool(1)\\
					Flatten()
				} & \makecell[c]{Conv-BN-ReLU($\mbox{dim}_{\mbox{input}}$,128,3,1,1)\\
					Conv-BN-ReLU(128,256,3,2,0)\\
					MaxPool(2)\\
					Conv-BN-ReLU(256,256,3,1,1)\\
					Conv-BN-ReLU(256,512,3,1,1)\\
					MaxPool(2)\\
					Conv-BN-ReLU(512,512,3,1,1)\\
					Conv-BN-ReLU(512,512,3,1,1)\\
					MaxPool(2)\\
					Conv-BN-ReLU(512,512,3,1,1)\\
					Conv-BN-ReLU(512,1024,3,1,1)\\
					AdaptivePool(1)\\
					Flatten()
				} &
				\makecell[c]{Conv3d-BN-ReLU($\mbox{dim}_{\mbox{input}}$,128,3,1,1)\\
					Conv3d-BN-ReLU(128,256,3,2,0)\\
					MaxPool(2)\\
					Conv3d-BN-ReLU(256,256,3,1,1)\\
					Conv3d-BN-ReLU(256,512,3,1,1)\\
					MaxPool(2)\\
					Conv3d-BN-ReLU(512,512,3,1,1)\\
					Conv3d-BN-ReLU(512,512,3,1,1)\\
					MaxPool(2)\\
					Conv3d-BN-ReLU(512,512,3,1,1)\\
					Conv3d-BN-ReLU(512,1024,3,1,1)\\
					AdaptivePool(1)\\
					Flatten()
				} \\ \hline
				$\mbox{Dec}_x^{\mathrm{Data}}$ & 
				\makecell[c]{
					UnFlatten()\\
					Upsample(2)\\
					Tconv-BN-ReLU($\mbox{dim}_{\mbox{input}}$,128,2,2,0)\\
					Tconv-BN-ReLU(128,64,2,2,0)\\
					Tconv-BN-ReLU(64,32,2,2,0)\\
					Tconv-BN-ReLU(32,16,2,2,0)\\
					Conv(16,3,3,1,1)\\
					Sigmoid()\\
					Cropping(28)
				} & \makecell[c]{
					UnFlatten()\\
					Upsample(16)\\
					Tconv-BN-ReLU($\mbox{dim}_{\mbox{input}}$,256,2,2,0)\\
					Conv-BN-ReLU(256,256,3,1,1)\\
					Tconv-BN-ReLU(256,128,2,2,0)\\
					Conv-BN-ReLU(128,128,3,1,1)\\
					Tconv-BN-ReLU(128,64,2,2,0)\\
					Conv-BN-ReLU(64,64,3,1,1)\\
					Tconv-BN-ReLU(64,32,2,2,0)\\
					Conv-BN-ReLU(32,32,3,1,1)\\
					Conv(32,3,3,1,1)\\
					Sigmoid()
				} & \makecell[c]{
					UnFlatten()\\
					Upsample(6)\\
					Tconv3d-BN-ReLU($\mbox{dim}_{\mbox{input}}$,256,2,2,0)\\
					Conv3d-BN-ReLU(256,256,3,1,1)\\
					Tconv3d-BN-ReLU(256,128,2,2,0)\\
					Conv3d-BN-ReLU(128,128,3,1,1)\\
					Tconv3d-BN-ReLU(128,64,2,2,0)\\
					Conv3d-BN-ReLU(64,64,3,1,1)\\
					Tconv3d-BN-ReLU(64,64,2,2,0)\\
					Conv3d-BN-ReLU(64,64,3,1,1)\\
					Conv3d(64,1,3,1,1)\\
					Sigmoid()
				} \\ \hline
				$\mbox{Enc}_d^{\mathrm{Data}}$ & \makecell[c]{
					FC-BN-ReLU($d$, 128)\\
					FC-BN-ReLU(128, 256)
				} & \makecell[c]{
					FC-BN-ReLU($d$, 256)\\
					FC-BN-ReLU(256, 512)\\
					FC-BN-ReLU(512, 512)
				} & \makecell[c]{
					FC-BN-ReLU($d$, 256)\\
					FC-BN-ReLU(256, 512)\\
					FC-BN-ReLU(512, 512)
				} \\ \hline
				$\mbox{Dec}_y^{\mathrm{Data}}$ & \makecell[c]{
					FC-BN-ReLU($\mbox{dim}_{z,s}$, 512)\\
					FC-BN-ReLU(512, 256)\\
					FC(256,2)
				} & \makecell[c]{
					FC-BN-ReLU($\mbox{dim}_{z,s}$, 512)\\
					FC-BN-ReLU(512, 256)\\
					FC(256,2)
				} &\makecell[c]{
					FC-BN-ReLU($\mbox{dim}_{z,s}$, 512)\\
					FC-BN-ReLU(512, 256)\\
					FC(256,2)
				}  \\ \hline
				$\mbox{Dec-CE}_y^{\mathrm{Data}}$ & \makecell[c]{
					FC-BN-ReLU($\mbox{dim}_{z,s}$, 512)\\
					FC-BN-ReLU(512, 256)\\
					FC(256,2)
				} & \makecell[c]{
					FC-BN-ReLU($\mbox{dim}_{z,s}$, 1024)\\
					FC-BN-ReLU(1024, 2048)\\
					FC(2048,2)
				} &
				\makecell[c]{
					FC-BN-ReLU($\mbox{dim}_{z,s}$, 512)\\
					FC-BN-ReLU(512, 256)\\
					FC(256,2)
				}\\ \hline
				$\mbox{DANN-CLS}_y^{\mathrm{Data}}$ & \makecell[c]{
					FC-BN-ReLU(256, 32)\\
					FC-BN-ReLU(32, 2)\\
				} & \makecell[c]{
					FC-BN-ReLU(1024, 2048)\\
					FC-BN-ReLU(2048, 2)\\
				} & \makecell[c]{
					FC-BN-ReLU(1024, 1024)\\
					FC-BN-ReLU(1024, 2)\\
				} \\ \hline
				${\Phi}_{z,s}^{\mathrm{Data}}$ & \makecell[c]{
					FC-ReLU($\mbox{dim}_{z,s}$, 256)\\
					FC-ReLU(256, $\mbox{dim}_{z,s}$)\\
				} &
				\makecell[c]{
					FC-ReLU($\mbox{dim}_{z,s}$, 1024)\\
					FC-ReLU(1024, $\mbox{dim}_{z,s}$)\\
				} &
				\makecell[c]{
					FC-ReLU($\mbox{dim}_{z,s}$, 1024)\\
					FC-ReLU(1024, $\mbox{dim}_{z,s}$)\\
				}
				\\ \hline
				$\mbox{Enc}_{z,s}^{\mathrm{Data}}$ & \makecell[c]{
					FC-ReLU(256, 256)\\
					FC-ReLU(256, $\mbox{dim}_{z,s}$)\\
				} &
				\makecell[c]{
					FC-ReLU(1024, 1024)\\
					FC-ReLU(1024, $\mbox{dim}_{z,s}$)\\
				} &
				\makecell[c]{
					FC-ReLU(1024, 1024)\\
					FC-ReLU(1024, $\mbox{dim}_{z,s}$)\\
				}
				\\ \hline
				$p^{\mathrm{Data}}_\theta(x|z_d,z_x,z_y)$ & \makecell[c]{
					FC-BN-ReLU(1024)\\
					UnFlatten()\\
					Upsample(8)\\
					TConv-BN-ReLU(64,128,5,1,0)\\
					Upsample(24)\\
					TConv-BN-ReLU(128,256,5,1,0)\\
					Conv(256, 256*3,1,1,0)
				} &
				\makecell[c]{
					FC-BN-ReLU(1024)\\
					UnFlatten()\\
					Upsample(16)\\
					TConv-BN-ReLU(64,128,5,1,0)\\
					Upsample(64)\\
					TConv-BN-ReLU(128,256,5,1,0)\\
					Upsample(256)\\
					Conv(256, 3,1,1,0)
				} &
				\makecell[c]{
					FC-BN-ReLU(1024)\\
					UnFlatten()\\
					Upsample(8)\\
					TConv3d-BN-ReLU(16,64,5,1,0)\\
					Conv3d-BN-ReLU(64,128,3,1,1)\\
					Upsample(24)\\
					TConv3d-BN-ReLU(128,128,5,1,0)\\
					Conv3d-BN-ReLU(128,128,3,1,1)\\
					Upsample(48)\\
					Conv3d-BN-ReLU(128,32,3,1,1)\\
					Conv3d(32, 1,1,1,0)\\
				}
				\\ \hline
				\makecell[c]{$p^{\mathrm{Data}}_{\theta_d}(z_d|d)$ \\ $p^{\mathrm{Data}}_{\theta_y}(z_y|y)$} & \makecell[c]{
					FC-BN-ReLU($\dim_{d,y}$, 64)\\
					FC(64,64); FC(64,64)\\} 
				&
				\makecell[c]{
					FC-BN-ReLU($\dim_{d,y}$, 64)\\
					FC(64,64); FC(64,64)\\} & 
				\makecell[c]{
					FC-BN-ReLU($\dim_{d,y}$, 64)\\
					FC(64,64); FC(64,64)}
				\\ \hline
				\makecell[c]{$q^{\mathrm{Data}}_{\phi_d}(z_d|x)$\\
					$q^{\mathrm{Data}}_{\phi_x}(z_x|x)$\\
					$q^{\mathrm{Data}}_{\phi_y}(z_y|x)$} & \makecell[c]{
					Conv-BN-ReLU(3,32,5,1,0)\\
					MaxPool(2)\\
					Conv-BN-ReLU(32,64,5,1,0)\\
					MaxPool(2)\\
					Flatten()\\
					FC(1024, 64); FC(1024, 64) $\mathrm{Data}$
				} &
				\makecell[c]{
					Conv-BN-ReLU(3,32,3,2,1)\\
					MaxPool(2)\\
					Conv-BN-ReLU(32,64,3,2,1)\\
					MaxPool(2)\\
					Conv-BN-ReLU(64,64,3,2,1)\\
					MaxPool(2)\\
					Flatten()\\
					FC(1024, 64); FC(1024, 64) $\mathrm{Data}$
				} &
				\makecell[c]{
					Conv3d-BN-ReLU(1,64,3,2,1)\\
					Conv3d-BN-ReLU(64,128,3,1,1)\\
					MaxPool(3)\\
					Conv3d-BN-ReLU(128,256,3,1,1)\\
					Conv3d-BN-ReLU(256,256,3,1,1)\\
					MaxPool(2)\\
					Conv3d-BN-ReLU(256,256,3,1,1)\\
					Conv3d-BN-ReLU(256,128,3,1,1)\\
					MaxPool(2)\\
					Flatten()\\
					FC(1024, 64); FC(1024, 64) $\mathrm{Data}$
				}
				\\ \hline
			\end{tabular}
		\end{table}

		
	\end{document}